\documentclass[twoside,11pt]{article}

\usepackage{amsmath,amsfonts,amssymb}
\usepackage{ntheorem}
\usepackage[abbrvbib]{jmlr2e}
\usepackage[capitalize, nameinlink]{cleveref} 

\usepackage[utf8]{inputenc}
\newtheorem{counterexample}[theorem]{Counterexample}

\usepackage{xcolor}
\usepackage{xspace} 
\usepackage{courier}
\usepackage{bbm}
\usepackage{dsfont}
\usepackage{graphicx}
\usepackage{verbatim}
\usepackage{booktabs}

\usepackage{subcaption}

\crefname{proposition}{Prop.}{Props.}
\crefname{definition}{Def.}{Defs.}
\crefname{lemma}{Lemma}{Lemmas}
\crefname{example}{Ex.}{Exs.}
\crefname{equation}{Eq.}{Eqs.}
\crefname{section}{Sec.}{Secs.}
\crefname{subsection}{Sec.}{Secs.}
\crefname{subsubsection}{Sec.}{Secs.}
\crefname{figure}{Fig.}{Figs.}
\crefname{wrapfigure}{Fig.}{Figs.}
\crefname{corollary}{Cor.}{Cors.}

\usepackage{enumitem} 

\usepackage{mathtools}
\newcommand{\defeq}{\vcentcolon=}

\newif\ifarxiv
\arxivtrue 

\def\gT{{\mathcal{T}}}
\def\gS{{\mathcal{S}}}
\def\gF{{\mathcal{F}}}
\def\gO{{\mathcal{O}}}
\def\gE{{\mathcal{E}}}
\def\gH{{\mathcal{H}}}
\def\gL{{\mathcal{L}}}
\def\sA{{\mathbb{A}}}
\def\R{{\mathbb{R}}}

\usepackage{lastpage}
\jmlrheading{24}{2023}{1-\pageref{LastPage}}{4/22; Revised
5/23}{6/23}{22-0364}{Yoshua Bengio, Salem Lahlou, Tristan Deleu, Edward J. Hu, Mo Tiwari, Emmanuel Bengio}
\ShortHeadings{GFlowNet Foundations}{Bengio, Lahlou, Deleu, Hu, Tiwari, Bengio}
\firstpageno{1}

\begin{document}

\title{GFlowNet Foundations}

\author{\name Yoshua Bengio\thanks{Equal Contribution} \email yoshua.bengio@mila.quebec \\
       \addr Mila, Universié de Montréal, CIFAR, IVADO
       \AND
       \name Salem Lahlou\footnotemark[1] \email lahlosal@mila.quebec \\
       \addr Mila, Universié de Montréal
       \AND
       \name Tristan Deleu\footnotemark[1] \email deleutri@mila.quebec \\
       \addr Mila, Universié de Montréal
       \AND
       \name Edward J. Hu \email edward@edwardjhu.com \\
       \addr Mila, Universié de Montréal, Microsoft
       \AND
       \name Mo Tiwari \email motiwari@stanford.edu \\
       \addr Stanford University
       \AND
       \name Emmanuel Bengio \email bengioem@mila.quebec \\
       \addr Mila, McGill University
}
\editor{David Sontag}

\maketitle

\begin{abstract}
    Generative Flow Networks (GFlowNets) have been introduced as a method to sample a diverse set of candidates in an active learning context, with a training objective that makes them approximately sample in proportion to a given reward function. In this paper, we show a number of additional theoretical properties of GFlowNets, including a new local and efficient training objective called detailed balance for the analogy with MCMC. GFlowNets can be used to estimate joint probability distributions and the corresponding marginal distributions where some variables are unspecified and, of particular interest, can represent distributions over composite objects like sets and graphs. GFlowNets amortize the work typically done by computationally expensive MCMC methods in a single but trained generative pass. They could also be used to estimate partition functions and free energies, conditional probabilities of supersets (supergraphs) given a subset (subgraph), as well as marginal distributions over all supersets (supergraphs) of a given set (graph). We introduce variations enabling the estimation of entropy and mutual information, 
\ifarxiv
    sampling from a Pareto frontier, connections to reward-maximizing policies, and extensions to stochastic environments, 
\fi 
continuous actions and modular energy functions.
\end{abstract}

\section{Introduction}
\label{sec:introduction}
Building upon the introduction of Generative Flow Networks (GFlowNets) by~\citet{bengio2021flow}, we provide here an in-depth formal foundation and expansion of the set of theoretical results in ways that may be of interest for the active learning scenario of~\citet{bengio2021flow} but also much more broadly. 

\subsection{What is a GFlowNet ?}
GFlowNets have properties which make them well-suited to perform amortized probabilistic inference in general, whether for sampling or for marginalizing. Sampling takes place at training time while run-time sampling or computations of marginalized quantities can be done in a single pass through a sequence of constructive stochastic steps. This makes GFlowNets an interesting alternative to Monte-Carlo Markov chains (MCMC) and related to amortized variational inference~\citep{malkin2022gfnhvi}. 

Because sampling of a compositional object $s$ can be achieved through a sequence of stochastic steps, very rich multimodal distributions $P_T(s)$ over such objects can be represented, and the offline training objectives make it possible to explore and discover modes of the distribution of interest. The key property of GFlowNets is that their sampling policy is trained to make the probability $P_T(s)$ of sampling an object $s$ approximately proportional to the value $R(s)$ of a given reward function applied to that object. We also talk of an energy function ${\cal E}(s)=-\log R(s)$, i.e., the reward function is non-negative and corresponds to an unnormalized probability. Whereas one typically trains a generative model from a dataset of positive examples, a GFlowNet is trained to match the given energy or reward function and convert it into a sampler. We view that sampler as a generative policy because the composite object $s$ is constructed through a sequence of smaller stochastic steps (see \cref{fig:gfn_sampling}), often corresponding to constructively composing different elements of $s$, like the edges of a graph. 

This conversion of an energy function or unnormalized probability function to a sampler  is similar to what MCMC methods achieve but once trained, GFlowNets will generate a sample in one shot instead of generating a long sequence of samples whose distribution would gradually approach the desired one. GFlowNets thus avoid the lengthy stochastic search in the space of such objects and  the associated mode-mixing intractability challenge of MCMC methods~\citep{jasra2005markov,bengio2013better,pompe2020framework}. Multiple iid samples can be obtained from the GFlowNet by calling the sampler multiple times. GFlowNets exchange that intractability of sampling with MCMC for the challenge of amortized training of the generative policy. The latter problem would be equally intractable if the modes of the reward function did not have a inherent (but not necessarily known) structure over which the learner could generalize, i.e., the learner had almost no chance to correctly guess where to find new modes based on (i.e., training on) those it had already visited. 


\begin{figure}
    \centering
    \hspace*{-0.8cm}\includegraphics[scale=0.6]{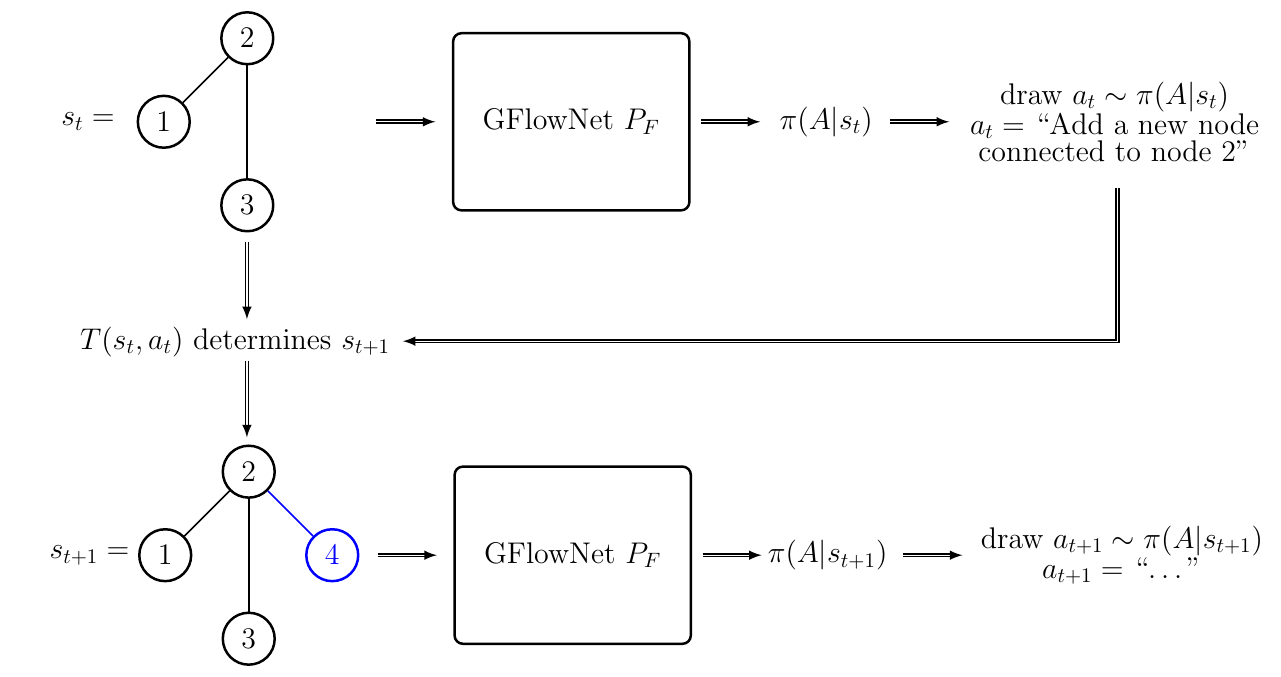}
    \caption{A diagram of how a GFlowNet iteratively constructs an object. We adopt notation that is common in the reinforcement learning literature: $s_t$ represents the state of the partially constructed object (in this case, a graph) at time $t$, $a_t$ represents the action taken by the GFlowNet at time $t$ to transition to state $s_{t+1}= T(s_t, a_t)$. In this diagram, the GFlowNet takes a 3-node graph as input and determines an action to take. The action, combined with the environment transition function $T(s_t, a_t)$, determines $s_{t+1}$: a four-node graph. This process repeats until an exit action is sampled and the sample is complete.}
    \label{fig:gfn_sampling}
\end{figure}

The energy function or reward function (exponential of minus energy) is evaluated only at the end of the sequential construction process for objects $s$, in what we call a terminating state. Every such constructive sequence starts in the single initial state $s_0$ and ends in a terminal state. As illustrated in Figure~\ref{fig:flows}, we can visualize the set of all trajectories starting from $s_0$ and ending in a terminal state $s$. The term "flow" in "generative flow networks" refers to unnormalized probabilities that can be learned by GFlowNet learning procedures. 
The flow in an intermediate state $s$ is a weighted sum of the non-negative rewards of the terminating states reachable from $s$.
Those weights are such as to avoid double-counting: if we were to inject a fixed flow of liquid in $s_0$ and dispatch that liquid in each child of any state $s$ proportionally to the GFlowNet policy for choosing a child of $s$, we would obtain the flow at each state and the flow at terminating states would match the reward function at those states. As shown in greater detail here and for the first time in the first GFlowNet paper~\citep{bengio2021flow}, this can be achieved with a flow constraint at each state: the sum of incoming flows must match the sum of outgoing flows.

\begin{figure}[t]
    \centering
    \includegraphics[width=\linewidth]{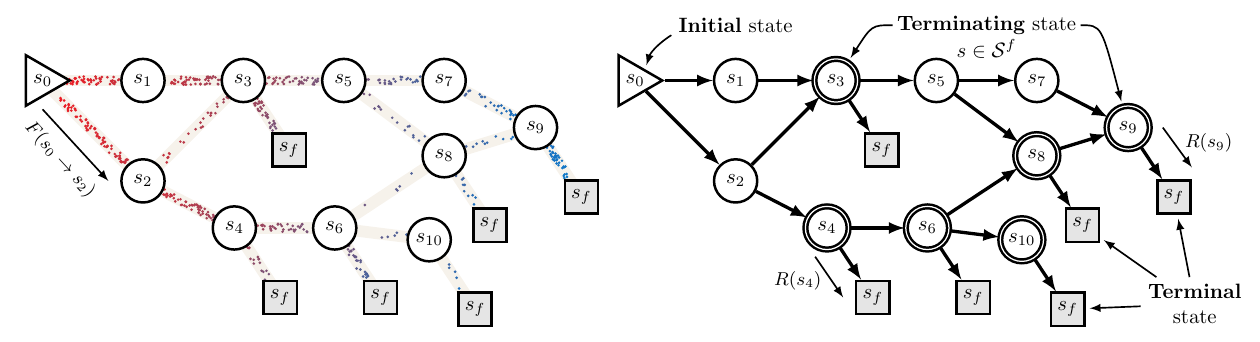}
    \caption{Illustration of the structure of a Generative Flow Network (GFlowNet), as a pointed DAG over states $s$, with particles flowing along edges to represent the flow function. Any object sampled by the GFlowNet policy can be obtained by starting from initial state $s_0$ and then at each step choosing a child with probability proportional to the GFlowNet policy's transition probability. This process stops when a terminating action is chosen from a terminating state $s$ (yielding a terminal state $s_f$), at which point a reward $R(s)$ is obtained. The figure shows a tiny GFlowNet and the possible trajectories from $s_0$ to any of the terminal states. It illustrates that in general a state can be reached through several trajectories. GFlowNet algorithms learn a policy such that the probability of sampling terminating state $s$ is proportional to $R(s)$. It tries to learn a flow function $F(s)$ and $F(s\rightarrow s')$ over all states (including intermediate states) $s$ and transitions $s\rightarrow s'$ with $F(s)=R(s)$ at terminal states and $F(s_0)$ being the sum of rewards over all terminal states. A sufficient property to achieve this is that at each state  the sum of incoming flows equals the sum of outgoing flows.}
    \label{fig:flows}
\end{figure}

\subsection{Contributions of this paper}
In this paper, an important contribution is the notion of \textit{conditional} GFlowNet, which enables estimation of intractable sums corresponding to marginalization over many steps of object construction, and can thus be used to compute free energies\footnote{In machine learning, a free energy is the logarithm of an unnormalized marginal probability, a generally intractable sum of exponentiated negative energies.} over different types of joint distributions, perhaps most interestingly over sets and graphs. This marginalization also enables estimation of entropies, conditional entropies\, and mutual information. GFlowNets can
thus 
be generalized to estimate multiple flows corresponding to modeling a rich outcome (rather than a scalar reward function)
.

We refer the reader to~\citet{bengio2021flow} and~\cref{sec:related-work} for a discussion of related approaches and differences with common generative models and reinforcement learning (RL) methods. In an RL context, two interesting properties of GFlowNets already noted in that paper are that they (1) can be trained in an offline manner with trajectories sampled from a distribution different from the one represented by the GFlowNet and (2) they match the reward function in probability rather than try to find a configuration which maximizes rewards or returns. The latter property is particularly interesting in the context of exploration, to ensure the configurations sampled from the generative policy are both interesting and diverse.  It is also interesting to transform GFlowNets into amortized probabilistic inference machines: if we choose the reward function to be a prior (over some random variable) times a likelihood (how well some data is fit given that choice of random variable value), then the GFlowNet policy learns to sample from the corresponding Bayesian posterior (which is proportional to prior times likelihood). The ability of GFlowNets to generate a diverse set of samples then corresponds to the ability to sample from the modes of the target distribution.

An important source of inspiration for GFlowNets is the way information propagates in temporal-difference RL methods~\citep{sutton2018reinforcement}.  Both rely on a principle of coherence for credit assignment which may only be achieved asymptotically when training converges. While exact gradient calculation may be intractable, because the number of paths in state space to consider is exponentially large, both methods rely on local coherence between different components and a training objective that states that if all the learned components are coherent with each other locally, then we obtain a system that estimates the quantities of interest globally. Examples include estimation of expected discounted returns in temporal-difference methods and probability measures with GFlowNets.

This paper extends the theory of the original GFlowNet construction~\citep{bengio2021flow} in several directions, including a new local training objective called detailed balance (for the analogy with the detailed balance condition of Monte-Carlo Markov chains) which avoids forming explicit sums required by the previously proposed flow matching loss, as well as formulations enabling the calculation of marginal probabilities (or free energies) for subsets of variables, more generally for subsets of larger sets, or subgraphs, 
\ifarxiv
their application to estimating entropy and mutual information, and the introduction of an unsupervised form of GFlowNets (the reward function is not needed while training, only observations of outcomes) enabling sampling from a Pareto frontier, for example. Although basic GFlowNets are more similar to bandits (in that a reward is only provided at the end of a sequence of actions), they can be extended to take into account intermediate rewards and thus a notion of return, and sample according to these returns. The original formulation of GFlowNets is also limited to discrete and deterministic environments, while this paper suggests how these two limitations could be lifted.
\else 
and their application to estimating entropy and mutual information.
\fi 
Finally, whereas the basic formulation of GFlowNets assumes a given reward or energy function, this paper considers how the energy function could be jointly learned with the GFlowNet, opening the door to novel energy-based modeling methodologies and a modular structure for both the energy function and the GFlowNet.

\subsection{GFlowNets in other works}
In addition to the theory presented in this paper, \citet{malkin2022gfnhvi} and \citet{zimmermann2022variational} prove some partial equivalences between GFlowNets and hierarchical variational methods, providing yet more theoretical evidence for the efficacy of GFlowNets in learning to sample proportionally to a given reward function. These works also provide evidence for the superiority of GFlowNets in off-policy settings.

GFlowNets
have found a wide array of applications due to the associated diversity of generated samples. 
In contexts where a cheap proxy for the true reward function exists, GFlowNets have been used to surface samples under which to query the proxy before more expensive evaluation under the true reward function.
In these settings, the diversity of samples generated by GFlowNets can be used for robustness to proxy misspecification and to incorporate epistemic uncertainty. 
For example, \citet{zhang2022scheduling} use GFlowNets to produce sample schedules for operations in a computation graph, where evaluating the runtimes of sample schedules via a proxy is fast but evaluating the same schedules on target hardware is expensive.
In active learning problems, \citet{jain2022biologicalseqdesign,jain2023multiobjective} use GFlowNet sampling as a subroutine inside an active learning loop as a substitute for Bayesian Optimization or RL-based methods. 
\citet{jain2022biologicalseqdesign} apply GFlowNets to search for novel anti-microbial peptides, discover DNA sequences that have high binding activity with human transcription factors, and to find proteins with high fluorescence.
Additionally, \citet{jain2023multiobjective} develops preference-conditional GFlowNets, where a preference weight vector is used to scalarize multiple objective functions into a single reward. 
The authors apply their techniques to various molecule and DNA sequence generation tasks and find that their methods are able to find different Pareto-optimal samples along the Pareto frontier.

GFlowNets have found applications in several other machine learning problems.
For example, \citet{zhang2022discretemodeling} simultaneously train an energy-based model and a GFlowNet; the energy function is trained with samples from a GFlowNet, which, in turn, uses the energy function to form its reward. Their method results in a generative model for binary vectors in high dimensions, e.g., binarized digits.
\citet{deleu2022bayesian} use a GFlowNet for structure learning; the GFlowNet produces samples that approximates the true posterior over causal graphs given a dataset. Their method works on both observational and interventional data, and  compares favorably to MCMC- and variational inference-based methods.
\citet{hu2023gfnem} find maximum-likelihood estimates of latent variable models with discrete compositional latents by jointly training a GFlowNet to approximately sample from the generally intractable posterior in the E-step of the expectation-maximization (EM) algorithm.

\section{Flow Networks and Markovian Flows}
\label{sec:measures-over-markovian-flows}
\subsection{Some elements of graph theory}
\label{sec:elements-of-graph-theory}
In this section, we recall some basic definitions and properties of graphs, which are the basis of flow networks and GFlowNets.
\begin{definition}
A directed graph is a tuple $G = (\gS, \sA)$, where $\gS$ is a finite set of states, and $\sA$ a subset of $\gS \times \gS$ representing directed edges. Elements of $\sA$ are denoted $s{\rightarrow}s'$ and called {\bf edges} or {\bf transitions}. 
\\
A \textbf{trajectory} in such a graph is a sequence $\tau=(s_1, \dots, s_n)$ of elements of $\gS$ such that every transition $s_t{\rightarrow}s_{t+1} \in \sA$ and $n > 1$. We denote $s \in \tau$ to mean that $s$ is in the trajectory $\tau$, i.e., $\exists t \in \{1,\dots, n\} \ s_t=s$, and similarly $s{\rightarrow}s' \in \tau$ to mean that $\exists t \in \{1, \dots, n-1\} \ s_t=s,s_{t+1}=s'$. For convenience, we also use the notation $\tau = s_1 \rightarrow \dots \rightarrow s_n$. The \textbf{length} of a trajectory is the number of edges in it (the length of $\tau = (s_1, \dots, s_n)$ is thus $n-1$).
\\
A \textbf{directed acyclic graph} (DAG) is a directed graph in which there is no trajectory $\tau=(s_1, \dots, s_n)$ satisfying $s_n = s_1$.
\end{definition}

Given a DAG $G=(\gS, \sA)$, and two states $s, s' \in \gS$, if there exists a trajectory in $G$ starting in $s$ and ending in $s'$, then we write $s < s'$. The binary relationship ``$<$'' defines a \textbf{strict partial order} (i.e. it is irreflexive, asymmetric and transitive).
We write $s \leq s'$ if $s < s'$ or $s = s'$. The binary relation ``$\leq$'' is a (non-strict) \textbf{partial order} (i.e. it is reflexive, antisymmetric and transitive).

If there is no order relation between $s$ and $s'$, we write $s \lessgtr s'$.


\begin{definition}
\label{eq:def-parents-children-consistent}
Given a DAG $G=(\gS, \sA)$, the {\bf parent set} of a state $s \in \gS$, which we denote $Par(s)$, contains all of the direct parents of $s$ in $G$, i.e., $Par(s) = \{s' \in \gS \ : \ \mbox{$s'{\rightarrow}s$} \in \sA\}$; similarly, the {\bf child set} $Child(s)$ contains all of the direct children of $s$ in $G$, i.e., $Child(s) = \{s' \in \gS \ : \ s{\rightarrow}s' \in \sA\}$.
\end{definition}

\begin{definition}
Given a DAG $G=(\gS, \sA)$. $G$ is called a \textbf{pointed DAG} if there exist two states $s_0, s_f \in \gS$ that satisfy:
\begin{equation*}
    \forall s \in \gS \setminus \{s_0\} \ \ s_0 < s \text{ and } \forall s \in \gS \setminus \{s_f\} \ \ s < s_f.
\end{equation*}
$s_0$ is called the \textbf{source state} or \textbf{initial state}. $s_f$ is called the \textbf{sink state} or \textbf{final state}. Because ``$<$'' is a strict partial order, these two states are unique.

A \textbf{complete trajectory} in such a DAG is any trajectory starting in $s_0$ and ending in $s_f$. We denote such a trajectory as $\tau = (s_0, s_1, \dots, s_n, s_{n+1}=s_f)$.

We denote by $\gT$ the set of all complete trajectories in $G$, and by $\gT^{partial}$ the set of (possibly incomplete) trajectories in $G$.

A state $s \in \gS$ is called a \textbf{terminating state} if it is a parent of the sink state, i.e. $s \rightarrow s_f \in \sA$. The transition $s \rightarrow s_f$ is called a \textbf{terminating edge}. We denote by:
\begin{itemize}
    \item $\sA^{-f} = \{ s\rightarrow s' \in \sA, \ s' \neq s_f \}$, the set of non-terminating edges in $G$,
    \item $\sA^{f} = \{ s\rightarrow s' \in \sA, \ s' = s_f \} = \sA \setminus \sA^{-f}$, the set of terminating edges in $G$,
    \item $\gS^{f} = \{ s \in  \gS, \ s \rightarrow s_f \in \sA^f\} = Par(s_f)$, the set of terminating states in $G$.
\end{itemize}
\end{definition}

In \cref{fig:illustration-dag-terminating}, we visualize the concepts introduced in the previous definitions.

\begin{figure}[ht]
    \centering
    \includegraphics{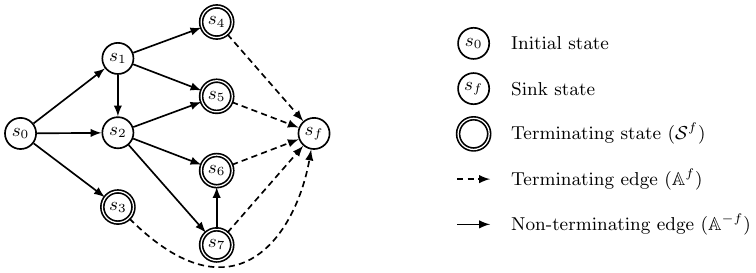}
    \caption{Example of a pointed DAG $G$ illustrating the notions of initial state ($s_0$), final or sink state ($s_f$), terminating states in $\gS^{f}$, with a transition to $s_f$ called a terminating edge, in $\sA^{f}$. A terminating state may have other children different from the sink state (e.g., the terminating state $s_{7}$).}
    \label{fig:illustration-dag-terminating}
\end{figure}

Note that the constraint of a single source state and single sink state is only a mathematical convenience since a bijection exists between general DAGs and those with this constraint (by the addition of a unique source/sink state connected to all the other source/sink states).

\begin{definition}
\label{def:P_F-P_B-consistent}
Let $G$ be a pointed DAG with source state $s_0$ and sink state $s_f$. A \textbf{forward} (resp. \textbf{backward}) \textbf{probability function consistent with} $G$ is any non-negative function $\hat{P}_F$ (resp. $\hat{P}_B$) defined on $\sA$ that satisfies $\forall s \in \gS \setminus\{s_f\}, \ \sum_{s' \in Child(s)} \hat{P}_F(s' \mid s) = 1$ (resp. $\forall s \in \gS\setminus\{s_0\}, \sum_{s' \in Par(s)} \hat{P}_B(s' \mid s) = 1$).
\end{definition}

With pointed DAGs, consistent forward and backward probability functions, that are probabilities over states, can be used to define probabilities over trajectories, i.e. probability measures on some subsets of $\gT^{partial}$. The following lemma shows how to construct such factorized probability measures:
\begin{lemma}
\label{lemma:PF-PB-extension}
Let $G = (\gS, \sA)$ be a pointed DAG, and consider a forward probability function $\hat{P_F}$, and a backward probability function $\hat{P_B}$ both consistent with $G$. For any state $s \in \gS \setminus \{s_f\}$, we denote by $\gT_{s, f} \subseteq \gT^{partial}$ the set of trajectories in $G$ starting in $s$ and ending in $s_f$; and for any state $s \in \gS \setminus \{s_0\}$, we denote by $\gT_{0, s} \subseteq \gT^{partial}$ the set of trajectories in $G$ starting in $s_0$ and ending in $s$.

Consider the extensions of $\hat{P}_F$ and $\hat{P}_B$ on $\gT^{partial}$ defined by:
\begin{align}
    \label{eq:PF-extension}
    \forall \tau = (s_1, \dots, s_n) \in \gT^{partial} \ \ \hat{P}_F(\tau) \defeq \prod_{t=1}^{n-1} \hat{P}_F(s_{t+1} \mid s_t) \\
    \label{eq:PB-extension}
    \forall \tau = (s_1, \dots, s_n) \in \gT^{partial} \ \ \hat{P}_B(\tau) \defeq \prod_{t=1}^{n-1} \hat{P}_B(s_t \mid s_{t+1})
\end{align}
We have the following:
\begin{align}
    \forall s \in \gS \setminus \{s_f\} \ \ \sum_{\tau \in \gT_{s, f}} \hat{P}_F(\tau) = 1 \label{eq:PF-extension-valid}  \\
    \forall s' \in \gS \setminus \{s_0\} \ \ \sum_{\tau \in \gT_{0, s'}} \hat{P}_B(\tau) = 1  \label{eq:PB-extension-valid}
\end{align}
\end{lemma}

\begin{proof}
For convenience, we will use $\gT_{s\rightarrow s', s_f}$ to denote the set of trajectories starting with $s \rightarrow s'$ and ending in $s_f$, and $\gT_{0, s\rightarrow s'}$ to denote the set of trajectories starting in $s_0$ and ending with $s \rightarrow s'$. This allows to write:
\begin{align*}
    &\forall s\neq s_f \ \ \gT_{s, f} = \bigcup_{s' \in Child(s)} \gT_{s\rightarrow s', s_f}, \ \ \{\gT_{s\rightarrow s', s_f}, \ s' \in Child(s)\} \text{ pairwise disjoint}, \\
    &\forall s' \neq s_0 \ \ \gT_{0, s'} = \bigcup_{s \in Par(s')} \gT_{0, s\rightarrow s'}, \ \ \{\gT_{0, s \rightarrow s'}, \ s \in Par(s')\} \text{ pairwise disjoint}.
\end{align*}
Additionally, for any $s \neq s_f$, we denote by $d_{s, f}$ the maximum trajectory length in $\gT_{s, f}$; and for any $s' \neq s_0$, we denote by $d_{0, s'}$ the maximum trajectory length in $\gT_{0, s}$.

We will prove \cref{eq:PF-extension-valid} by strong induction on $d_{s, f}$ and \cref{eq:PB-extension-valid} by strong induction on $d_{0, s'}$.

\textbf{Base cases: } If $d_{s, f} =1$ and $d_{0, s'} = 1$, then $\gT_{s, f} = \{ (s\rightarrow s_f) \}$ and $\gT_{0, s'} = \{ (s_0 \rightarrow s') \}$. Hence, $\sum_{\tau \in \gT_{s, f}} \hat{P}_F(\tau) = \hat{P}_F(s \rightarrow s_f) = \hat{P}_F(s_f \mid s) = 1$ given that $s_f$ is the only child of $s$ (otherwise $d_{s, f}$ cannot be $1$), and $\sum_{\tau \in \gT_{0, s'}} \hat{P}_B(\tau) = \hat{P}_B(s_0 \mid s') = 1$ given that $s_0$ is the only parent of $s'$ (otherwise $d_{0, s'}$ cannot be $1$).

\textbf{Induction steps: } Consider $s \neq s_f$ such that $d_{s, f} > 1$ and $s' \neq s_0$ such that $d_{0, s'} > 1$. Because of the disjoint unions written above, we have:
\begin{align*}
    \sum_{\tau \in \gT_{s, f}} \hat{P}_F(\tau) = \sum_{\tilde{s} \in Child(s)} \sum_{\tau \in \gT_{s \rightarrow \tilde{s}, f}} \hat{P}_F(\tau) = \sum_{\tilde{s} \in Child(s)} \hat{P}_F(\tilde{s} \mid s) \sum_{\tau \in \gT_{\tilde{s}, f}} \hat{P}_F(\tau) =  1, \\ 
    \sum_{\tau \in \gT_{0, s'}} \hat{P}_B(\tau) = \sum_{\tilde{s}' \in Par(s')} \sum_{\tau \in \gT_{0, \tilde{s}' \rightarrow s}} \hat{P}_B(\tau) = \sum_{\tilde{s}' \in Par(s')} \hat{P}_B(\tilde{s}' \mid s') \sum_{\tau \in \gT_{0, \tilde{s}'}} \hat{P}_B(\tau) = 1,
\end{align*}
where we used the induction hypotheses in the third equality of each line.
\end{proof}

\subsection{Trajectories and Flows}
\label{sec:trajectories-and-flows}
We augment pointed DAGs it with a function $F$ called a \emph{flow}. An analogy which helps to picture flows is a stream of particles flowing through a network where each particle starts at $s_0$ and flowing through some trajectory terminating in $s_f$. The flow $F(\tau)$ associated with each complete trajectory $\tau$ contains the number of particles sharing the same path $\tau$.

\begin{definition}
\label{def:flowmeasure}
Given a pointed DAG, a {\bf trajectory flow} (or ``\textbf{flow}'') is any non-negative function $F:{\cal T} \mapsto \R^{+}$ defined on the set of {\em complete trajectories} $\cal T$. $F$ induces a \emph{measure} over the $\sigma$-algebra $\Sigma = 2^{\gT}$, the power set on the set of complete trajectories $\gT$. In particular, for every subset $A \subseteq \gT$, we have
\begin{equation}
    F(A) = \sum_{\tau \in A} F(\tau). 
    \label{eq:F-event}
\end{equation}
The pair $(G, F)$ is called a \textbf{flow network}.
\end{definition}
This definition ensures that $(\gT, 2^{\gT}, F)$ is a measure space. We abuse the notation here, using $F$ to denote both a function of complete trajectories, and its corresponding measure over $(\gT, 2^{\gT})$. A special case is when the event $A$ is the singleton trajectory $\{\tau\}$, where we just write its measure as $F(\tau)$. We also abuse the notation to define the flow through either a particular state $s$, or through a particular edge $s{\rightarrow}s'$ in the following way.
\begin{definition}
The {\bf flow through a state} (or state flow) $F:{\cal S} \mapsto \R^{+}$ corresponds to the measure of the set of complete trajectories going through a particular state:
\begin{equation}
 F(s) \defeq F(\{\tau \in \gT \,:\, s \in \tau\}) = \sum_{\tau \in \gT \,:\, s\in\tau}F(\tau).
 \label{eq:Fs}
\end{equation}
Similarly, the {\bf flow through an edge} (or edge flow) $F: \sA \mapsto \R^{+}$ corresponds to the measure of the set of complete trajectories going through a particular edge:
\begin{equation}
\label{eq:Fss'}
    F(s{\rightarrow}s') \defeq F(\{\tau \in \gT \,:\, s{\rightarrow}s' \in \tau\}) = \sum_{\tau \in \gT \,:\, s{\rightarrow}s'\in\tau}F(\tau).
\end{equation}
\end{definition}
Note that with this definition, we have $F(s{\rightarrow}s') = 0$ if $s{\rightarrow}s'\notin \sA$ is not an edge in the pointed DAG (since $F(\emptyset) = 0$). We call the flow of a terminating transition $F(s{\rightarrow}s_{f})$ a \emph{terminating flow}. The following proposition relates the state flows and the edge flows:
\begin{proposition}
\label{prop:state-edge-flow}
Given a flow network $(G, F)$. The state flows and edge flows satisfy:
\begin{align}
\forall s \in \gS \setminus \{s_f\} \ \ F(s) = \sum_{s' \in Child(s)} F(s \rightarrow s') \label{eq:state-flow-outgoing-edge-flows} \\
\forall s' \in \gS \setminus \{s_0\} \ \ F(s') = \sum_{s \in Par(s')} F(s \rightarrow s') \label{eq:state-flow-incoming-edge-flows}
\end{align}
\end{proposition}
\begin{proof}
Given $s \neq s_{f}$, the set of complete trajectories going through $s$ is the (disjoint) union of the sets of trajectories going through $s \rightarrow s'$, for all $s'\in Child(s)$:
\begin{equation*}
    \{\tau \in \gT\,:\,s\in \tau\} = \bigcup_{s'\in Child(s)}\{\tau \in \gT\,:\,s \rightarrow s'\in \tau\}.
\end{equation*}
Therefore, it follows that:
\begin{equation*}
    F(s) = \sum_{\tau\,:\,s\in\tau}F(\tau) = \sum_{s'\in Child(s)}\sum_{\tau\,:\,s \rightarrow s' \in \tau}F(\tau) = \sum_{s'\in Child(s)}F(s \rightarrow s')
\end{equation*}
Similarly, \cref{eq:state-flow-incoming-edge-flows} follows by writing the set of complete trajectories going though $s'\neq s_{0}$ as the (disjoint) union of the sets of trajectories going through $s \rightarrow s'$ for all $s \in Par(s')$.
\end{proof}

\subsection{Flow Induced Probability Measures}
\label{sec:flow-probability}
\begin{definition}
\label{def:total_flow}
Given a flow network $(G, F)$, the {\bf total flow} $Z$ is the measure of the whole set $\gT$, corresponding to the sum of the flows of all the complete trajectories:
\begin{equation}
 Z \defeq F(\gT) = \sum_{\tau \in \cal T} F(\tau).
\label{eq:Z}
\end{equation}
\end{definition}
\begin{proposition}
\label{prop:flow-initial-state}
The flow through the initial state equals the flow through the final state equals the total flow $Z$.
\end{proposition}
\begin{proof}
 Since $\forall \tau \in {\cal T},\; s_0, s_f \in \tau$, applying~\cref{eq:Fs} to $s_0$ and $s_f$ yields
 \begin{align}
     F(s_0) &= \sum_{\tau\in \cal T} F(\tau) = Z \label{eq:F0=Z}, \\
      F(s_f) &= \sum_{\tau \in \cal T} F(\tau) = Z. \label{eq:Ff=Z}
 \end{align}
\end{proof}
Intuitively, \cref{prop:flow-initial-state} justifies the use of the term ``flow", introduced by~\citet{bengio2021flow}, by analogy with a stream of particles flowing from the initial state to the final states.

We use the letter $Z$ in \cref{def:total_flow}, often used to denote the partition function in probabilistic models and statistical mechanics, because it is a normalizing constant which can turn the measure space $(\gT, 2^\gT, F)$ defined above into the probability space $(\gT, 2^\gT, P)$:
\begin{definition}
\label{def:p}
Given a flow network $(G, F)$, the \textbf{flow probability} is the probability measure $P$ over the measurable space $(\gT, 2^\gT)$ associated with $F$:
\begin{equation}
 \forall A \subseteq \gT \ \ P(A) \defeq \frac{F(A)}{F(\gT)} = \frac{F(A)}{Z}.
 \label{eq:p-tau}
\end{equation}
For two events $A, B \subseteq \gT$, the conditional probability $P(A \mid B)$ thus satisfies:
\begin{align}
 P(A \mid B) &\defeq \frac{F(A{\cap}B)}{F(B)}.
\label{eq:cond-p}
\end{align}
\end{definition}
Similar to the flow $F$, we abuse the notation $P$ to define the probability of going through a state:
\begin{equation}
\label{eq:P-A}
 \forall s \in \gS \qquad  P(s) \defeq \frac{F(s)}{Z},
\end{equation}
and similarly for the probability of going through an edge. Note that $P(s)$ \emph{does not} correspond to a distribution over states, in the sense that $\sum_{s\in\gS} P(s) \neq 1$; in particular, it is easy to see that $P(s_{0}) = 1$ (in other words, the probability of a trajectory passing through the initial state $s_{0}$ is $1$). Additionally, for a trajectory $\tau \in \gT$, we also use the abuse of notation $P(\tau)$ instead of $P(\{\tau\})$ to denote the probability of going through a specific trajectory $\tau$.


\begin{definition}
\label{def:transitionprob}
Given a flow network $(G, F)$, the {\bf forward transition probability} operator $P_F$ is a function on $\gS \times \gS$, that is
a special case of the conditional probabilities induced by $F$ (\cref{eq:cond-p}):
\begin{equation}
\forall s \rightarrow s' \in \sA \ \ P_F(s'\mid s) \defeq P(s{\rightarrow}s'\mid s)=\frac{F(s{\rightarrow}s')}{F(s)}.
\label{eq:trans-p}
\end{equation}
Similarly, the {\bf backwards transition probability} is the operator defined by:
\begin{equation}
\forall s \rightarrow s' \in \sA \ \    P_B(s\mid s') \defeq P(s{\rightarrow}s'\mid s')=\frac{F(s{\rightarrow}s')}{F(s')}.
    \label{eq:trans-pb}
\end{equation}
Note how $P_F$ and $P_B$ are consistent with $G$ (in the sense of \cref{def:P_F-P_B-consistent}), as a consequence of \cref{prop:state-edge-flow}.

\end{definition}

Because flows define probabilities over states and edges, they can be used to define probability distributions over the terminating states of a graph (denoted by $\gS^f = Par(s_f)$) as follows:
\begin{definition}
\label{def:flow-P_T}
Given a flow network $(G, F)$, the \textbf{terminating state probability} $P_T$ is the probability over terminating states $\gS^f$ under the flow probability $P$:
\begin{equation}
    \label{eq:flow-P_T}
    \forall s \in \gS^f \ \ P_T(s) \defeq P(s \rightarrow s_f) = \frac{F(s \rightarrow s_f)}{Z}
\end{equation}
\end{definition}
Contrary to the probability $P(s)$ of going through a state $s$, the terminating state probability $P_{T}$ is a well-defined distribution over the terminating states $s\in \gS^{f}$, in the following sense:
\begin{proposition}
\label{prop:P_T-probability}
The terminating state probability $P_{T}$ is a well-defined distribution over the terminating states $s \in \gS^{f}$, in that $P_{T}(s) \geq 0$ for all $s \in \gS^{f}$, and
\begin{equation*}
    \sum_{s\in\gS^{f}}P_{T}(s) = 1.
\end{equation*}
\end{proposition}
\begin{proof}
Since the flow $F(s\rightarrow s_{f})$ is non-negative, it is easy to see that $P_{T}(s) \geq 0$. Moreover, using the definition of $\gS^{f} = Par(s_{f})$, \cref{prop:state-edge-flow} (relating the edge flows and the state flows), and \cref{prop:flow-initial-state} ($F(s_{f}) = Z$), we have
\begin{equation*}
    \sum_{s\in\gS^{f}}P_{T}(s) = \frac{1}{Z}\sum_{s\in\gS^{f}}F(s\rightarrow s_{f}) = \frac{1}{Z}\sum_{s\in Par(s_{f})}F(s\rightarrow s_{f}) = \frac{F(s_{f})}{Z} = 1.
\end{equation*}
\end{proof}
The terminating state probability is particularly important in the context of estimating flow networks (see \cref{sec:GFlowNets-Learning-a-Flow}), as it shows that a flow network $(G, F)$ induces a probability distribution over terminating states which is proportional to the terminating flows ${F(s\rightarrow s_{f})}$, the normalization constant $Z$ being given by initial flow $F(s_{0})$.

\subsection{Markovian Flows}
\label{sec:markovian-flows}
Defining a flow requires the specification of $|\gT|$ non-negative values (one for every trajectory $\tau \in \gT$), which is generally exponential in the number of graph edges. Markovian flows however have the remarkable property that they can be defined with much fewer ``numbers'', given that trajectory flows factorize according to $G$. 
\begin{definition}
\label{def:markov}
Let $(G, F)$ be a flow network, with flow probability measure $P$. $F$ is called a {\bf Markovian flow} (or equivalently $(G, F)$ a {\bf Markovian flow network})
if, for any state $s \neq s_0$, outgoing edge $s{\rightarrow}s' $, and for any trajectory $\tau = (s_0, s_1, \dots, s_n=s) \in \gT^{partial}$ starting in $s_0$ and ending in $s$:
\begin{equation}
 P(s{\rightarrow}s' \mid \tau) = P(s{\rightarrow}s' \mid s) = P_F(s' \mid s).
\label{eq:markov}
\end{equation}
\end{definition}

Note that the Markovian property does not hold for all of the flows as defined in the previous sections (e.g. \cref{fig:equivalentflows}). Intuitively, a flow can be considered non-Markovian if a particle in the ``flow stream" can remember its past history; if not, its future behavior can only depend on its current state and the flow must be Markovian. In this work, we will primarily be concerned with Markovian flows, though later we will re-introduce a form of memory via state-conditional flows that allow each flow ``particle" to remember parts of its history. The following proposition shows that Markovian flows have the property that the flows at (or the probabilities of) complete trajectories factorize according the the graph, and that it is a sufficient condition for defining Markovian flows.

\begin{proposition}
\label{prop:markovian-equivalences}
Let $(G, F)$ be a flow network, and $P$ the corresponding flow probability. The following three statements are equivalent:
\begin{enumerate}
    \item $F$ is a \emph{Markovian} flow
    \item There exists a unique probability function $\hat{P}_{F}$ consistent with $G$ such that for all complete trajectories $\tau = (s_{0}, \ldots, s_{n+1} = s_{f})$:
    \begin{equation}
        P(\tau) = \prod_{t=1}^{n+1} \hat{P}_{F}(s_{t}\mid s_{t-1}).
        \label{eq:markov-forward-decomposition}
    \end{equation}
    Moreover, the probability function $\hat{P}_{F}$ is exactly the forward transition probability associated with the flow probability $P$: $\hat{P}_{F} = P_{F}$.
    \item There exists a unique probability function $\hat{P}_{B}$ consistent with $G$ such that for all complete trajectories $\tau = (s_{0}, \ldots, s_{n+1} = s_{f})$:
    \begin{equation}
        P(\tau) = \prod_{t=1}^{n+1} \hat{P}_{B}(s_{t-1}\mid s_{t}).
        \label{eq:markov-backwards-decomposition}
    \end{equation}
    Moreover, the probability function $\hat{P}_{B}$ is exactly the backwards transition probability associated with the flow probability $P$: $\hat{P}_{B} = P_{B}$.
\end{enumerate}
\end{proposition}
\begin{proof}
Recall from \cref{lemma:PF-PB-extension} the notations $\gT_{0,s}$ to denote the set of partial trajectories from $s_{0}$ to $s$, and $\gT_{s',f}$ to denote the set of partial trajectories from $s'$ to $s_{f}$. We will prove the equivalences $1 \Leftrightarrow 2$ and $1 \Leftrightarrow 3$.

\begin{itemize}
    \item $1 \Rightarrow 2$: Suppose that $F$ is a Markovian flow. Then using the laws of probability, the Markov property in \cref{eq:markov}, and $P(s_{0}) = 1$, for some complete trajectory $\tau = (s_{0}, \ldots, s_{n+1} = s_{f})$:
    \begin{align*}
        P(\tau) &= P(s_{0} \rightarrow s_{1} \rightarrow \ldots \rightarrow s_{n+1}) = P(s_{0}\rightarrow s_{1})\prod_{t=1}^{n}P(s_{t} \rightarrow s_{t+1} \mid s_{0} \rightarrow \ldots \rightarrow s_{t})\\
        &= P(s_{0} \rightarrow s_{1}) \prod_{t=1}^{n}P(s_{t}\rightarrow s_{t+1} \mid s_{t})\\
        &= P(s_{0})P_{F}(s_{1}\mid s_{0})\prod_{t=1}^{n}P_{F}(s_{t+1}\mid s_{t})\\
        &= \prod_{t=1}^{n+1}P_{F}(s_{t}\mid s_{t-1}),
    \end{align*}
    where the second line uses to Markov property, and the third line uses the definition of the forward transition probability $P_{F}$. $P_F$ thus satisfies \cref{eq:markov-forward-decomposition} for all complete trajectories.
    
    To show uniqueness of $P_F$, assume \cref{eq:markov-forward-decomposition} is satisfied by some $\hat{P}_F$ for all complete trajectories. By definition of the forward transition probability:
    \begin{equation*}
        P_{F}(s'\mid s) \defeq P(s \rightarrow s'\mid s) = \frac{P(s\rightarrow s')}{P(s)}.
    \end{equation*}
    Any complete trajectory $\tau$ going through a state $s$ can be (uniquely) decomposed into a partial trajectory $\tau' \in \gT_{0,s}$ from $s_{0}$ to $s$, and a partial trajectory $\tau'' \in \gT_{s,f}$ from $s$ to $s_{f}$. Using the definition of $P(s)$, we have:
    \begin{align*}
        P(s) &= \sum_{\tau\,:\,s\in\tau}P(\tau) = \sum_{\tau\,:\,s\in\tau}\prod_{(s_{t}\rightarrow s_{t+1}) \in \tau}\hat{P}_{F}(s_{t+1}\mid s_{t})\\
        &= \left[\sum_{\tau'\in\gT_{0,s}}\prod_{(s_{t}\rightarrow s_{t+1})\in \tau'}\hat{P}_{F}(s_{t+1}\mid s_{t})\right]\underbrace{\left[\sum_{\tau''\in\gT_{s,f}}\prod_{(s_{t}\rightarrow s_{t+1})\in \tau''}\hat{P}_{F}(s_{t+1}\mid s_{t})\right]}_{=\,1\quad \textrm{(\cref{lemma:PF-PB-extension})}}\\
        &= \sum_{\tau'\in\gT_{0,s}}\prod_{(s_{t}\rightarrow s_{t+1})\in \tau'}\hat{P}_{F}(s_{t+1}\mid s_{t}).
    \end{align*}
    Similarly, any complete trajectory going through $s \rightarrow s'$ can be (uniquely) decomposed into a partial trajectory $\tau' \in \gT_{0,s}$ from $s_{0}$ to $s$, and a partial trajectory $\tau'' \in \gT_{s',f}$ from $s'$ to $s_{f}$. Again, using the definition of $P(s\rightarrow s')$:
    \begin{align*}
        P(s\rightarrow s') &= \sum_{\tau\,:\,(s\rightarrow s')\in \tau}P(\tau) = \sum_{\tau\,:\,(s\rightarrow s')\in\tau} \prod_{(s_{t}\rightarrow s_{t+1})\in\tau}\hat{P}_{F}(s_{t+1}\mid s_{t})\\
        &= \underbrace{\left[\sum_{\tau'\in\gT_{0,s}}\prod_{(s_{t}\rightarrow s_{t+1})\in\tau'}\hat{P}_{F}(s_{t+1}\mid s_{t})\right]}_{=\,P(s)}\hat{P}_{F}(s'\mid s)\underbrace{\left[\sum_{\tau''\in\gT_{s',f}}\prod_{(s_{t}\rightarrow s_{t+1})\in\tau''}\hat{P}_{F}(s_{t+1}\mid s_{t})\right]}_{=\,1\quad \textrm{(\cref{lemma:PF-PB-extension})}}\\
        &= P(s)\hat{P}_{F}(s'\mid s).
    \end{align*}
    Combining the two results above, we get:
    \begin{equation*}
        P_{F}(s'\mid s) = \frac{P(s\rightarrow s')}{P(s)} = \hat{P}_{F}(s'\mid s).
    \end{equation*}
    
    \item $2 \Rightarrow 1$: Suppose that there exists a probability function $\hat{P}_{F}$ consistent with $G$ such that for some complete trajectory $\tau = (s_{0}, \ldots, s_{n+1} = s_{f})$
    \begin{equation*}
        P(\tau) = \prod_{t=1}^{n+1}\hat{P}_{F}(s_{t} \mid s_{t-1}).
    \end{equation*}
For the same reasons as those used to justify the uniqueness in the $1 \Rightarrow 2$ proof,  $\hat{P}_{F}$ is necessarily equal to the forward transition probability $P_F$, associated with $P$.
     
    We now want to show that the flow $F$ associated with $P$ is Markovian, by showing the Markov property from \cref{eq:markov}. Let $\tau' \in \gT_{0, s}$ be any partial trajectory from $s_{0}$ to $s$; using the definition of conditional probability:
    \begin{equation*}
        P(s\rightarrow s' \mid \tau') = \frac{P(s_{0}\rightarrow \ldots \rightarrow s \rightarrow s')}{P(s_{0}\rightarrow \ldots \rightarrow s)}.
    \end{equation*}
    Following the same idea as above, we will now rewrite $P(s_{0} \rightarrow \ldots \rightarrow s)$, as a sum over complete trajectories that share the same prefix trajectory $\tau'$. Any such complete trajectory $\tau$ can be (uniquely) decomposed into this common prefix $\tau'$, and a partial trajectory $\tau'' \in \gT_{s,f}$ from $s$ to $s_{f}$. 
    \begin{align*}
        P(s_{0}\rightarrow \ldots \rightarrow s) &= \sum_{\tau\,:\,\tau'\subseteq \tau}P(\tau) = \sum_{\tau\,:\,\tau'\subseteq \tau}\prod_{(s_{t}\rightarrow s_{t+1})\in\tau}P_{F}(s_{t+1}\mid s_{t})\\
        &= \left[\vphantom{\sum_{\tau''\in\gT_{s',f}}}\prod_{s_{t-1} \rightarrow s_t \in \tau'}P_{F}(s_{t}\mid s_{t-1})\right]\underbrace{\left[\sum_{\tau''\in\gT_{s,f}}\prod_{(s_{t}\rightarrow s_{t+1})\in\tau''}P_{F}(s_{t+1}\mid s_{t})\right]}_{=\,1\quad \textrm{(\cref{lemma:PF-PB-extension})}}\\[-1em]
        &= \prod_{s_{t-1} \rightarrow s_t \in \tau'}P_{F}(s_{t}\mid s_{t-1}).
    \end{align*}
    Similarly, any complete trajectory $\tau$ that share the same prefix trajectory $(s_{0}, \ldots, s, s')$ can be (uniquely) decomposed into this common prefix, and a partial trajectory $\tau''\in\gT_{s',f}$ from $s'$ to $s_{f}$, leading to:
    \begin{align*}
        P(s_{0}\rightarrow \ldots \rightarrow s \rightarrow s') = P(s_{0}\rightarrow \ldots \rightarrow s)P_{F}(s'\mid s)
    \end{align*}
    Combining the two results above, we can conclude that $P$ satisfies the Markov property, and therefore that the flow $F$ is Markovian:
    \begin{equation*}
        P(s'\rightarrow s\mid \tau') = \frac{P(s_{0}\rightarrow \ldots \rightarrow s \rightarrow s')}{P(s_{0}\rightarrow \ldots \rightarrow s)} = P_{F}(s'\mid s) = P(s'\rightarrow s\mid s)
    \end{equation*}
    
    \item $\{1,2\}\Rightarrow 3$: Suppose that $F$ is a Markovian flow. We have shown above that this is equivalent to $P$ being decomposed into a product of forward transition probabilities $P_{F}$. For some complete trajectory $\tau = (s_{0}, \ldots, s_{n+1}=s_{f})$:
    \begin{equation*}
        P(\tau) = \prod_{t=1}^{n+1}P_{F}(s_{t}\mid s_{t-1}) = \prod_{t=1}^{n+1}\frac{P(s_{t-1} \rightarrow s_{t})}{P(s_{t-1})} = \prod_{t=1}^{n+1}\frac{P(s_{t-1}\rightarrow s_{t})}{P(s_{t})} = \prod_{t=1}^{n+1}P_{B}(s_{t-1}\mid s_{t}),
    \end{equation*}
    where the third equality uses the fact that $P(s_{0}) = P(s_{f}) = 1$, and using the definition of the backwards transition probability $P_{B}$. The proof of uniqueness of $P_B$ is similar to that of $P_F$ in $1 \Rightarrow 2$, and uses:
    \begin{align*}
        P(s\rightarrow s') &= \sum_{\tau\,:\,(s\rightarrow s')\in\tau}P(\tau) = \sum_{\tau\,:\,(s\rightarrow s')\in\tau}\prod_{(s_{t}\rightarrow s_{t+1})\in\tau}\hat{P}_{B}(s_{t}\mid s_{t+1})\\
        &= \underbrace{\left[\sum_{\tau'\in\gT_{0,s}}\prod_{(s_{t}\rightarrow s_{t+1})\in\tau'}\hat{P}_{B}(s_{t}\mid s_{t+1})\right]}_{=\,1\quad \textrm{(\cref{lemma:PF-PB-extension})}}\hat{P}_{B}(s\mid s')\underbrace{\left[\sum_{\tau''\in\gT_{s',f}}\prod_{(s_{t}\rightarrow s_{t+1})\in\tau''}\hat{P}_{B}(s_{t}\mid s_{t+1})\right]}_{=\,P(s')}\\
        &= P(s')\hat{P}_{B}(s\mid s'),
    \end{align*}
    
    \item $3 \Rightarrow 1$: Similar to the proof of $2 \Rightarrow 1$, $\hat{P}_B$ is necessarily equal to the backwards transition probability $P_B$ associated with $P$. Additionally, $P_{B}$ is related to the forward transition probability $P_{F}$:
    \begin{equation*}
        P(s\rightarrow s') = P_{B}(s\mid s')P(s') = P_{F}(s'\mid s)P(s).
    \end{equation*}
    We can therefore write the decomposition of $P$ in terms of $P_{F}$, instead of $P_{B}$. For some complete trajectory $\tau = (s_{0}, \ldots, s_{n+1}=s_{f})$:
    \begin{align*}
        P(\tau) &= \prod_{t=1}^{n+1}P_{B}(s_{t-1}\mid s_{t}) = \prod_{t=1}^{n+1}\frac{P(s_{t-1})}{P(s_{t})}P_{F}(s_{t+1}\mid s_{t}) = \frac{P(s_{0})}{P(s_{f})}\prod_{t=1}^{n+1}P_{F}(s_{t+1}\mid s_{t})\\
        &= \prod_{t=1}^{n+1}P_{F}(s_{t+1}\mid s_{t}),
    \end{align*}
    where we used the fact that $P(s_{0}) = P(s_{f}) = 1$. Using ``$2\Rightarrow 1$'', we can conclude that $F$ is a Markovian flow.
\end{itemize}
\end{proof}

The decomposition of \cref{eq:markov-forward-decomposition} shows how Markovian flows can be used to draw terminating states from the terminating state probability $P_T$ (\cref{eq:flow-P_T}). Namely, we have the following result:
\begin{corollary}
\label{cor:sampling-from-P_T}
Let $(G, F)$ be a Markovian flow network, and $P_F$ the corresponding forward transition probability. Consider the procedure starting from $s = s_0$, and iteratively drawing one sample from $P_F(. \mid s)$ until reaching $s_f$. Then the probability of the procedure terminating in a state $s$ is $P_T(s)$.
\end{corollary}

\begin{proof}
First, note that the procedure terminates with probability 1, given that $G$ is acyclic.

For the procedure to terminate in a state $s$, it means that the trajectory $\tau \in \gT$ implicitly constructed during the procedure contains the edge $s \rightarrow s_f$. The probability of the procedure terminating in $s$ is thus:
\begin{align*}
    \sum_{\tau \in \gT: s \rightarrow s_f \in \tau} \underbrace{\prod_{s' \rightarrow s'' \in \tau} P_F(s'' \mid s')}_{P(\tau) \text{, according to \cref{eq:markov-forward-decomposition}}} = P(s \rightarrow s_f) = P_T(s) 
\end{align*}
\end{proof}
The following proposition shows that, as a consequence of the \cref{prop:markovian-equivalences}, we obtain three different parametrizations of Markovian flows.
\begin{proposition}
\label{prop:flow-parametrizations}
Given a pointed DAG $G = (\gS, \sA)$, a Markovian flow on $G$ is \emph{completely} and \emph{uniquely} specified by one of the following:
\begin{enumerate}
    \item the combination of the total flow $\hat{Z}$ and the forward transition probabilities $\hat{P}_F(s' \mid s)$ for all edges $s \rightarrow s' \in \sA$,
    \item the combination of the total flow $\hat{Z}$ and the backward transition probabilities $\hat{P}_B(s \mid s')$ for all edges $s \rightarrow s' \in \sA$.
    \item the combination of the terminating flows $ \hat{F}(s{\rightarrow}s_f)$ for all terminating edges $s \rightarrow s_f \in \sA^f$
    and the backwards transition probabilities $\hat{P}_B(s \mid s')$ for all non-terminating edges $s \rightarrow s' \in \sA^{-f}$,
\end{enumerate}
\end{proposition}
\begin{proof}
In the first two settings, we define a flow function $F:\gT \rightarrow \R^+$, at a trajectory $\tau =(s_0, s_1, \ldots, s_n, s_{n+1}=s_f)$ as:
\begin{enumerate}
    \item $F(\tau) \defeq \hat{Z}  \prod_{t=1}^{n+1} \hat{P}_F(s_t \mid s_{t-1})$,
    \item $F(\tau) \defeq \hat{Z} \prod_{t=1}^{n+1} \hat{P}_B(s_{t-1} \mid s_t)$
\end{enumerate}
We need to prove that it is the only Markovian flow that can be defined for both settings. The proof for the third setting will follow from that of the second setting.

\textbf{First setting:} 

First, we need to show that the total flow $Z$ associated with the flow function $F$ (\cref{eq:Z}) matches $\hat{Z}$. This is a consequence of \cref{lemma:PF-PB-extension}:
\begin{equation*}
    Z = \sum_{\tau \in \gT} F(\tau) = \hat{Z} \underbrace{\sum_{\tau = (s_0, s_1, \dots, s_{n+1}=s_f) \in \gT} \prod_{t=1}^{n+1} \hat{P}_F(s_t \mid s_{t-1})}_{=1 \ \ \text{, according to \cref{lemma:PF-PB-extension}}} = \hat{Z}
\end{equation*}

Then, we need to show that the forward transition probability function $P_F$ associated with $F$ (\cref{eq:trans-p}) matches $\hat{P}_F$, and that the flow $F$ is Markovian. To this end, note that the corresponding flow probability $P$ satisfies \cref{eq:markov-forward-decomposition}. Thus, as a consequence of \cref{prop:markovian-equivalences}, $F$ is a Markovian flow, and its forward transition probability function is $\hat{P}_F$.

As a last requirement, we need to show that if a Markovian flow $F'$ has a partition function $Z' = \hat{Z}$ and a forward transition probability function $P'_{F} = \hat{P}_F$, then it is necessarily equal to $F$. This is a direct consequence of \cref{prop:markovian-equivalences}, given that for any $\tau = (s_0, \dots, s_{n+1}=s_f) \in \gT$:
\begin{equation*}
    F'(\tau) = Z' \prod_{t=1}^{n+1} P'_{F}(s_t \mid s_{t-1}) = \hat{Z} \prod_{t=1}^{n+1} \hat{P}_{F}(s_t \mid s_{t-1}) = F(\tau)
\end{equation*}

\textbf{Second setting:}

First, we show that as a consequence of \cref{lemma:PF-PB-extension}, the total flow $Z$ associated with $F$ matches $\hat{Z}$:
\begin{equation*}
    Z = \sum_{\tau \in \gT} F(\tau) = \hat{Z} \underbrace{\sum_{\tau = (s_0, s_1, \dots, s_{n+1}=s_f) \in \gT} \prod_{t=1}^{n+1} \hat{P}_B(s_{-1} \mid s_{t})}_{=1 \ \ \text{, according to \cref{lemma:PF-PB-extension}}} = \hat{Z}
\end{equation*}

Second, we note that the flow probability $P$ associated with $F$ satisfies \cref{eq:markov-backwards-decomposition}. Thus, as a consequence of \cref{prop:markovian-equivalences}, $F$ is a Markovian flow, and its backward transition probability function is $\hat{P}_B$.

Finally, if a Markovian flow $F'$ has a partition function $Z' = \hat{Z}$ and a backward transition probability function $P_B' = \hat{P}_B$, then following \cref{prop:markovian-equivalences}, $\forall \tau \in \gT, \ F'(\tau) = F(\tau)$.

\textbf{Third setting}:

From the terminating flows $\hat{F}(s \rightarrow s_f)$ and the backwards transition probabilities $\hat{P}_B(s \mid s')$ for non-terminating edges, we can uniquely define a total flow $\hat{Z}$, and extend $\hat{P}_B$ to all edges as follows:
\begin{align*}
    \hat{Z} &\defeq \sum_{s \in Par(s_f)} \hat{F}(s \rightarrow s_f) \\
    \hat{P}_B(s \mid s') &\defeq \begin{cases}
    \hat{P}_B(s \mid s') \quad \text{if } s' \neq s_ f \\
    \frac{\hat{F}(s \rightarrow s_f)}{\hat{Z}} \quad \text{otherwise.}
    \end{cases}
\end{align*}
This takes us back to the second setting, for which we have already proven that with $\hat{Z}$ and $\hat{P}_B$ defined for all edges, a Markovian flow is uniquely defined.

\end{proof}



\subsection{Flow Matching Conditions}
\label{sec:flow-matching-conditions}
In \cref{prop:flow-parametrizations}, we saw how forward and backward probability functions can be used to uniquely define a Markovian flow. We will show in the next proposition how non-negative functions of states and edges can be used to define a Markovian flow. Such functions cannot be unconstrained (as $\hat{P}_F$ and $\hat{Z}$ in \cref{prop:flow-parametrizations} e.g.), as we have seen in \cref{prop:state-edge-flow}.
\begin{proposition}
\label{prop:flow-matching}
Let $G=(\gS, \sA)$ be a pointed DAG. Consider a non-negative function $\hat{F}$ taking as input either a state $s \in \gS$ or a transition $s{\rightarrow} s' \in \sA$. 
Then $\hat{F}$ corresponds to a flow if and only if the \textbf{flow matching conditions}: 
\begin{align}
    \forall s'>s_0,\;\; \hat{F}(s') & = \sum_{s\in Par(s')} \hat{F}(s{\rightarrow}s') \nonumber \\
    \forall s'<s_f,\;\; \hat{F}(s') &= \sum_{s''\in Child(s')} \hat{F}(s'{\rightarrow} s'')
\label{eq:flow-match}
\end{align}
are satisfied. More specifically, $\hat{F}$ uniquely defines a Markovian flow $F$ matching $\hat{F}$ on states and transitions: 
\begin{equation}
\label{eq:corresponding-flow}    
\forall \tau = (s_0, \dots, s_{n+1}=s_f) \in \gT \ \ F(\tau)=\frac{\prod_{t=1}^{n+1} \hat{F}(s_{t-1}{\rightarrow}s_t)}{\prod_{t=1}^{n}\hat{F}(s_t)}.
\end{equation}
\end{proposition}

\begin{proof}
Necessity is a direct consequence of \cref{prop:state-edge-flow}. Let's show sufficiency. Let $\hat{P}_F$  be the forward probability function defined by:
\begin{align*}
    \forall s \rightarrow s' \in \sA \ \ \hat{P}_F(s' \mid s) \defeq \frac{\hat{F}(s \rightarrow s')}{\hat{F}(s)}.
\end{align*}
$\hat{P}_F$ is consistent with $G$ given that $\hat{F}$ satisfies the flow matching conditions (\cref{eq:flow-match}). Let $\hat{Z} = \hat{F}(s_0)$. According to \cref{prop:flow-parametrizations}, there exists a unique Markovian flow $F$ with forward transition probability function $P_F = \hat{P}_F$ and partition function $Z = \hat{Z}$, and such that for a trajectory $\tau = (s_0, \dots, s_{n+1}=s_f) \in \gT$:
\begin{equation}
\label{eq:corresponding-flow-proof}    
\forall \tau = (s_0, \dots, s_{n+1}=s_f) \in \gT \ \ F(\tau)= \hat{Z} \prod_{t=1}^{n+1} \hat{P}_F(s_t \mid s_{t-1}) = \frac{\prod_{t=1}^{n+1} \hat{F}(s_{t-1}{\rightarrow}s_t)}{\prod_{t=1}^{n}\hat{F}(s_t)}.
\end{equation}
Additionally, similar to the proof of \cref{prop:markovian-equivalences}, we can write for any state $s' \neq s_0$:
\begin{align*}
    F(s') &= \hat{Z} \sum_{\tau \in \gT_{0, s'}} \prod_{(s_t \rightarrow s_{t+1}) \in \tau} \hat{P}_F(s_{t+1} \mid s_t) \\
    &= \hat{Z} \frac{\hat{F}(s')}{\hat{F}(s_0)} \underbrace{\sum_{\tau \in \gT_{0, s'}} \prod_{(s_t \rightarrow s_{t+1}) \in \tau} \hat{P}_B(s_{t} \mid s_{t+1})}_{=1 \ \text{, according to \cref{lemma:PF-PB-extension}}} \\
    &= \hat{F}(s'),
\end{align*}
where $\hat{P}_B(s' \mid s) \defeq \frac{\hat{F}(s \rightarrow s')}{\hat{F}(s')}$ defines a backward probability function consistent with $G$. 
And because $\forall s \rightarrow s' \in \sA \ \ P_F(s' \mid s) = \hat{P}_F(s' \mid s)$, it follows that $\forall s \rightarrow s' \in \sA \ \ F(s \rightarrow s') = \hat{F}(s \rightarrow s')$.

To show uniqueness, let's consider a Markovian flow $F'$ that matches $\hat{F}$ on states and edges. Following \cref{prop:markovian-equivalences}, for any trajectory $\tau = (s_0, \dots, s_{n+1}=s_f) \in \gT$
\begin{align*}
    F'(\tau)= \hat{Z} \prod_{t=1}^{n+1} \hat{P}_F(s_t \mid s_{t-1}) = \frac{\prod_{t=1}^{n+1} \hat{F}(s_{t-1}{\rightarrow}s_t)}{\prod_{t=1}^{n}\hat{F}(s_t)} = F(\tau).
\end{align*}
 
\end{proof}
Note how~\cref{eq:flow-match} can be used to recursively define the flow in all the states if $Z$ is given and either the forward or the backwards transition probabilities are given. Either way, we would start from the flow at one of the extreme states $s_0$ or $s_f$ and then distribute it recursively through the directed acyclic graph of the flow network, either going forward or going backward. A setting of particular interest, that will be central in \cref{sec:GFlowNets-Learning-a-Flow}, is when we are given all the terminal flows $F(s{\rightarrow}s_f)$, and we would like to deduce a state flow function $F(s)$ and a forward transition probability function $P_F(s' \mid s)$ for the rest of the flow network. 

Next, we will see how to parametrize Markovian flows using forward and backward probability functions consistent with the DAG.
Unlike the condition in~\cref{prop:flow-matching}, the new condition does not involve a sum over transitions, which could be problematic if each state can have a large number of successors or if the state-space is continuous. Interestingly, the resulting condition is analogous to the {\em detailed balance} condition of Monte-Carlo Markov chains.

\begin{definition}
\label{def:compatible}
Given a pointed DAG $G=(\gS, \sA)$, a forward transition probability function $\hat{P}_F$ and a backward transition probability function $\hat{P}_B$ consistent with $G$,  $\hat{P}_F$ and $\hat{P}_B$ are {\bf compatible} if there exists an edge flow function $\hat{F}:\sA \rightarrow \R^+$ such that 
\begin{align}
\label{eq:comp-cond}
    \forall s\rightarrow s' \in \sA \ \ \hat{P}_F(s' \mid s) &=\frac{\hat{F}(s{\rightarrow}s')}{\sum_{s' \in Child(s)}\hat{F}(s{\rightarrow}s')}, \ \ \hat{P}_B(s \mid s') &=\frac{\hat{F}(s{\rightarrow}s')}{\sum_{s'' \in Par(s')}\hat{F}(s''{\rightarrow}s')}
\end{align}
\end{definition}

\begin{proposition}
\label{prop:detailed-balance}
Let $G = (\gS, \sA)$ be a pointed DAG. Consider a non-negative function $\hat{F}$ over states, a forward transition probability function $\hat{P}_F$ and a backwards transition probability function $\hat{P}_B$ consistent with $G$. Then, $\hat{F}, \hat{P}_B,$ and $\hat{P_F}$ jointly correspond to a flow if and only if the \textbf{detailed balance conditions} holds:
\begin{equation}
    \forall s{\rightarrow}s' \in \sA  \quad \hat{F}(s)\hat{P}_F(s' \mid s)=\hat{F}(s')\hat{P}_B(s \mid s').
\label{eq:detailed-balance}
\end{equation}
More specifically, $\hat{F}, \hat{P}_F, $ and $\hat{P}_B$ uniquely define a Markovian flow $F$ matching $\hat{F}$ on states, and with transition probabilities matching $\hat{P}_F$ and $\hat{P}_B$. 
Furthermore, when this condition is satisfied, the forward and backward transition probability functions $\hat{P}_F$ and $\hat{P}_B$ are compatible.
\end{proposition}
\begin{proof}
For necessity, consider a flow $F$, with state flow function denoted $F$, and forward and backward transitions $P_F$ and $P_B$. It is clear from the definition of $P_F$ and $P_B$ (\cref{def:transitionprob}) that \cref{eq:detailed-balance} holds. We prove the sufficiency of the condition by first defining the edge flow
\begin{equation}
    \forall s\rightarrow s' \in \sA \ \ \hat{F}(s{\rightarrow}s') \defeq \hat{F}(s) \hat{P}_F(s' \mid s).
\label{eq:def-edge-flow}
\end{equation}
We then sum both sides of~\cref{eq:detailed-balance} over $s$, yielding
\begin{equation}
    \forall s' > s0 \ \ \sum_{s \in Par(s')} \hat{F}(s) \hat{P}_F(s' \mid s) = \hat{F}(s')\sum_{s \in Par(s')} \hat{P}_B(s \mid s') = \hat{F}(s')
\label{eq:sum-pred}
\end{equation}
where we used the fact that $\hat{P}_B$ is a normalized probability distribution. Combining this with~\cref{eq:def-edge-flow}, we get
\begin{equation}
    \forall s' >s_0 \ \ \hat{F}(s') = \sum_{s \in Par(s')} \hat{F}(s{\rightarrow}s')
\label{eq:flow-match-1}
\end{equation}
which is the first equality of the flow-matching condition (\cref{eq:flow-match}) of~\cref{prop:flow-matching}.
We can obtain the second equality by first using the normalization of $\hat{P}$, and then using our definition of the edge flow (\cref{eq:def-edge-flow}):
\begin{align}
    \forall s' > s_0 \ \ \hat{F}(s') &= \hat{F}(s')\sum_{s'' \in Child(s')}\hat{P}_F(s'' \mid s') \nonumber\\
     &= \sum_{s'' \in Child(s')} \hat{F}(s')\hat{P}_F(s'' \mid s')\nonumber\\
      &= \sum_{s'' \in Child(s')} \hat{F}(s'{\rightarrow}s'').
\label{eq:flow-match-2}
\end{align}
Following \cref{prop:flow-matching}, there exists a unique Markovian flow $F$ with state and edge flows given by $\hat{F}$. Using \cref{eq:def-edge-flow} and \cref{eq:detailed-balance}, it follows that $F$ has transition probabilities $\hat{P}_F$ and $\hat{P}_B$ as required. The uniqueness is also a consequence of \cref{eq:def-edge-flow}. This proves sufficiency.

To show that $\hat{P}_F$ and $\hat{P}_B$ are compatible (\cref{def:compatible}), we first combine~\cref{eq:def-edge-flow} and~\cref{eq:flow-match-2} (with relabeling of variables) to obtain
$$
  \forall s \rightarrow s' \in \sA \ \ \hat{P}_F(s' \mid s) = \frac{\hat{F}(s{\rightarrow}s')}{\sum_{s' \in Child(s)}\hat{F}(s{\rightarrow}s')},
$$
we then isolate $\hat{P}_B$ in~\cref{eq:detailed-balance}, yielding
$$
 \forall s \rightarrow s' \in \sA \ \  \hat{P}_B(s \mid s') = \frac{\hat{F}(s)}{\hat{F}(s')} \hat{P}_F(s' \mid s) = \frac{\hat{F}(s{\rightarrow}s')}{\hat{F}(s')} = \frac{\hat{F}(s{\rightarrow}s')}{\sum_{s'' \in Par(s')}\hat{F}(s'' \rightarrow s')},
$$
We thus get ~\cref{eq:comp-cond} of~\cref{def:compatible}, as desired.
\end{proof}

At first glance, it may seem that when $\hat{P}_B$ is unconstrained, the detailed balance condition can trivially be achieved by setting
\begin{equation}
\label{eq:DB-solution}
 \forall s \rightarrow s' \in \sA \ \ \hat{P}_B(s \mid s') = \frac{\hat{P}_F(s' \mid s)\hat{F}(s)}{\hat{F}(s')}
\end{equation}
However, because we also have the constraint $\sum_{s \in Par(s')} \hat{P}_B(s \mid s')=1$, then
\cref{eq:DB-solution} can only be satisfied if the flows are consistent with the forward transition:
$$
 \sum_{s \in Par(s')} \hat{P}_F(s' \mid s )\hat{F}(s) = \hat{F}(s').
$$

\subsection{Backwards Transitions can be Chosen Freely}
\label{sec:free-P_B}

Consider the setting in which we are given terminating flows to be matched, i.e. where the goal is to find a flow function with the right terminating flows. This is the setting introduced in~\citet{bengio2021flow}, and that will be studied in \cref{sec:GFlowNets-Learning-a-Flow}. In this case,~\cref{prop:flow-parametrizations} tells us that in order to fully determine the forward transition probabilities and the state or state-action flows, it is not sufficient in general to specify only the terminating flows; it is also necessary to specify the backwards transition probabilities on the edges other than the terminal ones (the latter being given by the terminating flows).

What this means is that the terminating flows do not specify the flow completely, e.g., because many different paths can land in the same terminating state. The preference over such different ways to achieve the same final outcome is specified by the backwards transition probability $P_B$ (except for $P_B(s \mid s_f)$ which is a function of the terminating flows and $Z$). For example, we may want to give equal weight to all parents of a node $s$, or we may prefer shorter paths, which can be achieved if we keep track in the state $s$ of the length of the shortest path to the node $s$, or we may let a learner discover a $P_B$ that makes learning $P_F$ or $F$ easier.

\subsection{Equivalence Between Flows}
\label{sec:equivalence-flows}
In the previous sections, we have seen that Markovian flows have the property that trajectory flows or probabilities factorize according to the DAG, and we have seen different ways of characterizing Markovian flows. In \cref{sec:GFlowNets-Learning-a-Flow}, we show how to approximate Markovian flows in order to define probability measures over terminating states. In this section, through an equivalence relation between trajectory flows, we justify the focus on Markovian flows. Given a pointed DAG $G = (\gS, \sA)$, we denote by:

\begin{itemize}
    \item $\gF(G)$: the set of flows on $G$, i.e. the set of functions from $\gT$, the set of complete trajectories in $G$, to $\R^+$,
    \item $\gF_{Markov}(G)$: the set of flows in $\gF(G)$ that are Markovian.
\end{itemize}

\begin{definition}
Let $G = (\gS, \sA)$ be a pointed DAG, and $F_1, F_2 \in \gF(G)$ two trajectory flow functions. We say that $F_1$ and $F_2$ are equivalent if they coincide on edge-flows, i.e.:
\begin{equation*}
   \forall s\rightarrow s' \in \sA \quad F_1(s \rightarrow s') = F_2(s \rightarrow s') 
\end{equation*}

\end{definition}
\cref{fig:equivalentflows} shows four flow functions in a simple pointed DAG that are pairwise equivalent.

\begin{figure}[ht]
\centering
\begin{minipage}{0.25\linewidth}
\includegraphics[width=\linewidth]{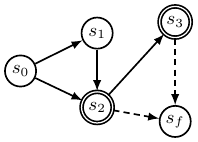}
\end{minipage}\hfill
\begin{minipage}{0.6\linewidth}
\begin{tabular}{ccccc}
    \toprule
    $\tau$ & $F_1(\tau)$ & $F_2(\tau)$ & $F_3(\tau)$ & $F_4(\tau)$  \\
    \midrule
    $s_{0}, s_{2}, s_{f}$ & $1$ & $4/5$ & $1$ & $6/5$ \\
    $s_{0}, s_{1}, s_{2}, s_{f}$ & $1$ & $6/5$ & $1$ & $4/5$ \\
    $s_{0}, s_{2}, s_{3}, s_{f}$ & $1$ & $6/5$ & 2 & $9/5$\\
    $s_{0}, s_1, s_{2}, s_{3}, s_{f}$ & $2$ & $9/5$ & $1$ & $6/5$ \\
    \bottomrule
\end{tabular}
\end{minipage}
    \caption{Equivalent flows and Markovian flows. Flows $F_1$ and $F_2$ are equivalent. $F_3$ and $F_4$ are equivalent, but not equivalent to $F_1$ and $F_2$. $F_2$ and $F_4$ are Markovian. $F_1$ and $F_3$ are not Markovian. $F_1, F_2, F_3$ and $F_4$ coincide on the terminating flows i.e. at $s_2 \rightarrow s_f$ and $s_3 \rightarrow s_f$.}
\label{fig:equivalentflows}
\end{figure}

This defines an equivalence relation (i.e., a relation that is reflexive, symmetric, and transitive). Hence, each flow $F$ belongs to an equivalence class, and the set of flows $\gF(G)$ can be partitioned into equivalence classes. Note that if two flows are equivalent, then the corresponding state flow functions also coincide (as a direct consequence of \cref{prop:state-edge-flow}).

\begin{proposition}
\label{prop:Markoveq}
Given a pointed DAG $G$. If two flow function $F_1, F_2 \in \gF_{Markov}(G)$ are equivalent, then they are equal. Additionally, for any flow function $F' \in \gF(G)$, there exists a unique Markovian flow function $F \in \gF_{Markov}(G)$ such that $F$ and $F'$ are equivalent.
\end{proposition}
\begin{proof}
Because $F_1$ and $F_2$ are Markovian, then for any trajectory $\tau = (s_0, \dots, s_{n+1} = s_f)$:
\begin{align*}
    F_1(\tau) &= \frac{\prod_{t=1}^{n+1} F_1(s_{t-1} \rightarrow s_t)}{\prod_{t=1}^{n}F_1(s_t)} \\
    &= \frac{\prod_{t=1}^{n+1} F_2(s_{t-1} \rightarrow s_t)}{\prod_{t=1}^{n}F_2(s_t)} \\
    &= F_2(\tau),
\end{align*}
where we combined the definition of equivalent flows and \cref{prop:markovian-equivalences}.

Given a flow function $F'$, because its state and edge flow functions satisfy the flow matching conditions (as a consequence of \cref{prop:state-edge-flow}), then according to \cref{prop:flow-matching}, the flow $F$ defined by: 
\begin{equation*}
    \forall \tau \defeq (s_0, \dots, s_{n+1} = s_f) \in \gT \quad F(\tau) = \frac{\prod_{t=1}^{n+1} F'(s_{t-1} \rightarrow s_t)}{\prod_{t=1}^{n} F'(s_t)}
\end{equation*}
is Markovian, and coincides with $F'$ on state and edge flows. Combining this with the statement above, we conclude that $F$ is the unique Markovian flow that is equivalent to $F'$.
\end{proof}
The previous proposition shows that in each equivalence class stands out a particular flow function, that has a property the other flows in the same equivalence class don't have: it is Markovian.

A consequence of this is that, if we care essentially about state and edge flows, instead of dealing with the full set of flows $\gF(G)$, it suffices to restrict any flow learning problem to the set of Markovian flows $\gF_{Markov}(G)$. The advantage of this restriction is that defining a flow requires the specification of $F(\tau)$ for all trajectories $\tau \in \gT$, whereas defining a Markovian flow requires the specification of $F(s \rightarrow s')$ for all edges $s \rightarrow s' \in \sA$, which is generally exponentially smaller than $\gT$ (note that the edge flows still need to satisfy the flow-matching conditions in \cref{prop:flow-matching}). Thus, in order to approximate or learn a flow function that satisfies some conditions on its edge or state values, it suffices to approximate or learn a Markovian flow, by learning the edge flow function, which is a much smaller object than the actual flow function.

\section{GFlowNets: Learning a Flow}
\label{sec:GFlowNets-Learning-a-Flow}

With the theoretical preliminaries established in \cref{sec:introduction} and \cref{sec:measures-over-markovian-flows}, we now consider the general class of problems introduced by~\citet{bengio2021flow} where some constraints or preferences over flows are given. Our goal is to find functions such as the state flow function $F(s)$ or the transition probability function $P(s{\rightarrow}s' \mid s)$ that best match these desiderata using corresponding estimators $\hat{F}(s)$ and $\hat{P}(s{\rightarrow}s' \mid s)$ which may not correspond to a proper flow. Such learning machines are called Generative Flow Networks (or GFlowNets for short). We focus on scenarios where we are given a target {\em reward function } $R:\gS^{f} \rightarrow \R^+$, and aim at estimating flows $F$ that satisfy:
\begin{equation}
\label{eq:gfn-condition-terminating}
    \forall s \in \gS^f \quad F(s \rightarrow s_f) = R(s)
\end{equation}

 Because of the equivalences that exist in the set of flows, then without loss of generality, we choose GFlowNets to approximate Markovian flows only. We are thus interested in the following set of flows:

 \begin{equation}
    \gF_{Markov}(G, R) = \{F \in \gF_{Markov}(G), \ \forall s \in \gS^f \ \ F(s\rightarrow s_f) = R(s) \}
\end{equation}
For now, we informally define a \textbf{GFlowNet} as an \textbf{estimator of a Markovian flow} function $F \in \gF_{Markov}(G, R)$. We provide a more formal definition later-on.

With an estimator $\hat{F}$ of such a Markovian flow $F$, we can define an approximate forward transition probability function $\hat{P}_F$, as in \cref{prop:markovian-equivalences}, in order to draw trajectories $\tau \in \gT$ (the set of complete trajectories in $G$) by iteratively sampling each state given the previous one, starting at $s_0$ and then with $s_{t+1} \sim \hat{P}_F(. \mid s_t)$ until we reach the sink state $s_{n+1}=s_f$ for some $n$.

Next, we will clarify how such an estimator can be obtained.
\subsection{GFlowNets as an Alternative to MCMC Sampling}
\label{sec:MCMC}

The main established methods to approximately sample from the distribution associated with an energy function $\cal E$ are Monte-Carlo Markov chain (MCMC) methods, which require significant computation (running a potentially very long Markov chain) to obtain samples. Instead, the GFlowNet approach amortizes upfront computation to train a generator that yields very efficient computation (a single configuration is constructed, no chain needed) for each new sample. For example, ~\citet{bengio2021flow} build a GFlowNet that constructs a molecule via a small sequence of actions, each of which adds an atom or a molecular substructure to an existing molecule represented by a graph, starting from an empty graph. Only one such configuration needs to be considered, in contrast with MCMC methods, which require potentially very long chains of such configurations, and suffer from the challenge of mode-mixing~\citep{jasra2005markov,bengio2013better,pompe2020framework}, which can take time exponentially long in the distance between modes. In GFlowNets, this computational challenge is avoided but the computational demand is converted to that of training the GFlowNet. To see how this can be extremely beneficial, consider having already constructed some configurations $x$ and obtained their unnormalized probability or reward $R(x)$. With these pairs $(x,R(x))$, a machine learning system could potentially generalize about the value of $R$ elsewhere, and if it is a generative model, sample new $x$'s in places of large $R(x)$. Hence, if there is an underlying statistical structure in how the modes of $R$ are related to each other, a generative learner that generalizes could guess the presence of modes it has not visited yet, taking advantage of the patterns it has already uncovered from the $(x,R(x))$ pairs it has seen. On the other hand, if there is no structure (the modes are randomly placed), then we should not expect GFlowNets to do significantly better than MCMC because training becomes intractable in high-dimensional spaces (since it requires visiting every area of the configuration space to ascertain its reward).

\subsection{GFlowNets and flow-matching losses}
\label{sec:flow-matching-losses}
We have seen in \cref{sec:markovian-flows} and \cref{sec:flow-matching-conditions} different ways of parametrizing a flow. For example, with a partition function and forward transition probabilities, or with edge flows that satisfy the flow matching conditions. Because there are many ways to parametrize GFlowNets, we start with an abstract formulation for them, where $o \in \gO$ represents a parameter configuration (e.g., resulting from or while training of a GFlowNet), $\Pi(o)$ gives the corresponding probability measure over trajectories $\tau \in \gT$, and $\gH$ maps a Markovian flow $F$ to its parametrization $o$. In the following definition, we show what conditions should be satisfied in order for such a parametrization to be valid.
\begin{definition}
Given a pointed DAG $G=(\gS, \sA)$, with an initial and sink states $s_0$ and $s_f$ respectively, and a target reward function $R:\gS^{f} \rightarrow \R^+$, we say that the triplet $(\gO, \Pi, \gH)$ is a \textbf{flow parametrization} of $(G, R)$ if:
\begin{enumerate}
    \item $\gO$ is a non-empty set,
    \item $\Pi$ is a function mapping each object $o \in \gO$ to an element $\Pi(o) \in \Delta(\gT)$, the set of probability distributions on $\gT$,
    \item $\gH$ is an injective functional from $\gF_{Markov}(G, R)$ to $\gO$,
    \item For any $F \in \gF_{Markov}(G, R)$, $\Pi(\gH(F))$ is the probability measure associated with the flow $F$ (\cref{def:p}).
\end{enumerate}
To each object $o \in \gO$, the distribution $\Pi(o)$ implicitly defines a \textbf{terminating state probability} measure:
\begin{equation}
\label{eq:P_T_o}
    \forall s \in \gS^f \quad P_T(s) \defeq \sum_{\tau \in \gT: s \rightarrow s_f \in \tau} \Pi(o)(\tau),
\end{equation}
where the dependence on $o$ in $P_T$ is omitted for clarity.
\end{definition}

The intuition behind the introduction of $(\gO, \Pi, \gH)$ is that we can define a probability measure over $\gT$ for each object $o \in \gO$, but only some of these objects correspond to a Markovian flow with the right terminating flows. For such objects $o$ (i.e. those that can be written as $o = \gH(F)$ for some flow $F \in \gF_{Markov}(G, R)$), the probability measure $P_T$ corresponds to the distribution of interest, according to \cref{def:flow-P_T}, i.e.:
\begin{equation*}
    \forall s \in \gS^f \quad P_T(s) \propto R(s)
\end{equation*}
GFlowNets thus provide a solution to the generally intractable problem of sampling from a target reward function $R$, or its associated \textbf{energy function}:
\begin{equation}
    \forall s \in \gS^f \quad \gE(s) \defeq - \log R(s)
\end{equation}
Directly approximating flows $F\in \gF_{Markov}(G, R)$ is a hard problem, whereas with some sets $\gO$, searching for an object $o \in \gH(\gF_{Markov}(G, R)) \subseteq \gO$ is a simpler problem that can be tackled with function approximation techniques.

Note that not the set $\gO$ cannot be arbitrary, as there needs to be a way to define an injective function from $\gF_{Markov}(G, R)$ to $\gO$. Below, for a given DAG $G$, we show three examples clarifying the abstract concept of parametrization:

\begin{example}
\label{ex:EDGEparam}
 \emph{Edge-flow parametrization: } Consider $\gO_{edge} = \gF(\sA^{-f}, \R^+)$, the set of functions from $\sA^{-f}$ to $\R^+$,
and the functionals $\gH_{edge}: \gF_{Markov}(G, R) \rightarrow \gO_{edge}$ and $\Pi_{edge}: \gO_{edge} \rightarrow \Delta(\gT)$ defined by:
 \begin{align*}
     \gH_{edge}(F) : (s\rightarrow s') \in \sA^{-f} \mapsto F(s\rightarrow s'),
 \end{align*}
\begin{align*}
    \forall \tau=(s_0, \dots, s_n=s_f) \in \gT \quad \Pi_{edge}(\hat{F})(\tau) \propto \prod_{t=1}^n P_{\hat{F}}(s_t \mid s_{t-1}),
\end{align*}
where 
\begin{equation}
    P_{\hat{F}}(s' \mid s) = \begin{cases}
    \frac{\hat{F}(s\rightarrow s')}{\sum_{s'' \neq s_f} \hat{F}(s\rightarrow s'') + R(s)} \quad \text{if } s' \neq s_f \\
    \frac{R(s)}{\sum_{s'' \neq s_f} \hat{F}(s\rightarrow s'') + R(s)} \quad \text{if } s' = s_f
    \end{cases}
\label{eq:FM-P-hat-F}
\end{equation}

The injectivity of $\gH_{edge}$ follows directly from \cref{prop:Markoveq} (two Markovian flows that coincide on both their terminating and non-terminating edge flow values are equal). And for any Markovian flow $F \in \gF_{Markov}(G, R)$, $\Pi_{edge}(\gH_{edge}(F))$ equals the probability measure associated with $F$, as is shown in \cref{prop:markovian-equivalences}.

$(\gO_{edge}, \Pi_{edge}, \gH_{edge})$ is thus a valid flow parametrization of $(G, R)$.
\end{example}

\begin{example}
\label{ex:PFparam}
 \emph{Forward transition probability parametrization: } Consider the set $\gO_{PF} = \gO_1 \times \gO_2$, where $\gO_1=\gF(\gS\setminus \{s_f\}, \R^+)$ is the set of function from $\gS\setminus \{s_f\}$ to $\R^+$ and $\gO_2$ is the set of forward probability functions $\hat{P}_F$ consistent with $G$
, and the functionals $\gH_{PF}: \gF_{Markov}(G, R) \rightarrow \gO_{PF}$ and $\Pi_{PF}: \gO_{PF}  \rightarrow \Delta(\gT)$ defined by:
 \begin{align*}
     \gH_{PF}(F) = \left(s \in \gS \setminus \{s_f\} \mapsto F(s), (s\rightarrow s') \in \sA \mapsto P_F(s' \mid s) \right),
 \end{align*}
\begin{align*}
    \forall \tau=(s_0, \dots, s_n=s_f) \in \gT \quad \Pi_{PF}(\hat{F}, \hat{P}_F)(\tau) \propto \prod_{t=1}^n \hat{P}_F(s_t \mid s_{t-1}),
\end{align*}
 where $P_F$ is the forward transition probability function associated with $F$ (\cref{eq:trans-p}). 
To verify that $\gH_{PF}$ is injective, consider $F_1, F_2 \in \gF_{Markov}(G, R)$ such that $\gH_{PF}(F_1) = \gH_{PF}(F_2)$. It means that $\forall s\in \gS^f$, $F_1(s) = F_2(s)$, and $\forall s\rightarrow s' \in \sA$, $\frac{F_1(s\rightarrow s')}{F_1(s)} = \frac{F_2(s\rightarrow s')}{F_2(s)}$. It follows that $\forall  s\rightarrow s' \in \sA$, $F_1(s\rightarrow s') = F_2(s\rightarrow s')$. Which, according to \cref{prop:Markoveq}, means that $F_1 = F_2$. And for any Markovian flow $F \in \gF_{Markov}(G, R)$, $\Pi_{PF}(\gH_{PF}(F))$ equals the probability measure associated with $F$, as is shown in \cref{prop:markovian-equivalences}.

$(\gO_{PF}, \Pi_{PF}, \gH_{PF})$ is thus a valid flow parametrization of $(G, R)$.

\end{example}

\begin{example}
\label{ex:PFBparam}
 \emph{Transition probabilities parametrization: }Similar to \cref{ex:PFparam}, we can parametrize a Markovian flow using the state-flow function and both its forward and backward transition probabilities, i.e. with $\gO_{PFB}=\gO_{PF} \times \gO_3$, $\gH_{PFB}$, and $\Pi_{PFB}$ defined as:
 \begin{align*}
    \gH_{PFB}(F) = \left(\gH_{PF}(F), (s\rightarrow s') \in \sA^{-f} \mapsto P_B(s \mid s'),  \right),
 \end{align*}
 \begin{align*}
    \forall \tau=(s_0, \dots, s_n=s_f) \in \gT \quad \Pi_{PFB}(\hat{F}, \hat{P}_F, \hat{P}_B)(\tau) \propto \prod_{t=1}^n \hat{P}_F(s_t \mid s_{t-1}),
 \end{align*}
 where $P_B$ is the function defined by \cref{eq:trans-pb}. and $\gO_3$ is the set of backward probability functions $\hat{P}_B$ consistent with $G$.
 The injectivity of $\gH_{PFB}$ is a direct consequence of that of $\gH_{PF}$. And for any Markovian flow $F$, $\Pi_{PFB}(\gH_{PFB}(F))$ equals the probability measure associated with $F$, as is shown in Prop.3.
 
$(\gO_{PFB}, \Pi_{PFB}, \gH_{PFB})$ is thus a valid flow parametrization of $(G, R)$.
\end{example}

We now have all the ingredients to formally define a GFlowNet:
\begin{definition}
\label{def:formal-gfn}
A \textbf{GFlowNet} is a tuple $(G, R, \gO, \Pi,  \gH)$, where:
\begin{itemize}
    \item $G=(\gS, \sA)$ is a pointed DAG with initial state $s_0$ and sink state $s_f$, 
    \item $R:\gS^f \rightarrow \R^+$ a target reward function,
    \item $(\gO, \Pi, \gH)$ a flow parametrization of $(G, R)$.
\end{itemize}

Each object $o \in \gO$ is called a GFlowNet configuration. When it is clear from context, we will use the term GFlowNet to refer to both $(G, R, \gO, \Pi, \gH)$ and a particular configuration $o$; similar to how the term ``Neural Network" refers to both the class of functions that can be represented with a particular architecture, and to a particular element of that class / weight configuration.

If $o \in \gH(\gF_{Markov}(G, R))$, then the corresponding terminating state probability measure (\cref{eq:P_T_o}) is proportional to the target reward $R$.
\end{definition}

Once we have a GFlowNet $(G, R, \gO, \Pi,  \gH)$, we still need a way to find objects $o \in \gH(\gF_{Markov}(G, R)) \subseteq \gO$. To this end, it suffices to design a \textbf{loss function} $\gL$ on $\gO$ that equals zero on objects  $o \in \gH(\gF_{Markov}(G, R))$ and only on those objects.
If our loss function $\gL$ is chosen to be non-negative, then an approximation of the target distribution (on $\gS^f$) is obtained by approximating the minimum of the function $\gL$. This provides a recipe for casting the search problem of interest to a minimization problem, as we typically do in machine learning. Such loss functions can be easily designed for the natural parametrizations we considered in \cref{ex:EDGEparam}, \cref{ex:PFparam}, and \cref{ex:PFBparam}, as we will illustrate below.

\begin{definition}
\label{def:flow-matching-loss}
Let $(G, R, \gO, \Pi,  \gH)$ be a GFlowNet. A \textbf{flow-matching loss} is any function $\gL:\gO \rightarrow \R^+$ such that:
\begin{align}
 \label{eq:flow--matching-loss}
\forall o \in \gO \quad \gL(o) = 0 \ \Leftrightarrow \ \exists F \in \gF_{Markov}(G, R) \ \ o = \gH(F)
\end{align}
We say that $\gL$ is \textbf{edge-decomposable}, if there exists a function $L:\gO \times \sA \rightarrow \R^+$ such that:
\begin{equation*}
    \forall o \in \gO \quad \gL(o) = \sum_{s\rightarrow s' \in \sA} L(o, s \rightarrow s'),
\end{equation*}
We say that $\gL$ is \textbf{state-decomposable}, f there exists a function $L:\gO \times \gS \rightarrow \R^+$ such that:
\begin{equation*}
    \forall o \in \gO \quad \gL(o) = \sum_{s \in \gS} L(o, s),
\end{equation*}
We say that $\gL$ is \textbf{trajectory-decomposable} if there exists a function $L:\gO \times \gT \rightarrow \R^+$ such that:
\begin{equation*}
    \forall o \in \gO \quad \gL(o) = \sum_{\tau \in \gT} L(o, \tau)
\end{equation*}
\end{definition}

As mentioned above, with a such a loss function, our search problems can be written as minimization problems of the form
\begin{equation}
\label{eq:minimizeLoss}
    \min_{o \in \gO} \gL(o),
\end{equation}
which can be tackled with gradient-based learning if the function $\gL$ is differentiable. 
Note that with an edge-decomposable flow-matching loss, the minimization problem in \cref{eq:minimizeLoss} is equivalent to:
\begin{equation}
    \label{eq:minimizeDecomposableLoss}
    \min_{o \in \gO} \mathbb{E}_{(s\rightarrow s') \sim \pi_T} [L(o, s\rightarrow s')], 
\end{equation}
where $\pi_T$ is any full support probability distribution on $\sA$, i.e. a probability distribution such that $\forall s\rightarrow s' \in \sA \ \ \pi_T(s \rightarrow s') > 0$. A similar statement can be made for state-decomposable or trajectory-decomposable flow-matching losses.

\begin{example}
\label{ex:fm-loss}
 Consider the edge-flow parametrization $(\gO_{edge}, \Pi_{edge}, \gH_{edge})$, and the function $L_{FM}:\gO_{edge} \times \gS \rightarrow \R^+$ defined for each $\hat{F} \in \gO_{edge}$ and $s' \in \gS$ as
 \begin{align*}
 L_{FM}(\hat{F}, s') = \begin{cases} \left( \log \left( \frac{\delta + \sum_{s \in Par(s')} \hat{F}(s \rightarrow s')}{\delta + R(s') + \sum_{s'' \in Child(s') \setminus \{s_f\}}\hat{F}(s' \rightarrow s'')} \right) \right)^2 \quad \text{if } s' \neq s_f, \\
 0 \quad \text{otherwise}
 \end{cases}
 \end{align*}
where $\delta \geq 0$ is a hyper-parameter. 
The function $\gL_{FM}$ mapping each $\hat{F} \in \gO_{edge}$ to 
\begin{equation}
    \label{eq:flow-matching-log-loss}
    \gL_{FM}(\hat{F}) = \sum_{s \in \gS} L_{FM}(\hat{F}, s)
\end{equation}
is a flow-matching loss, that is (by definition) state-decomposable. 

To see this, let $\hat{F} \in \gO_{edge}$ such that $\gL_{FM}(\hat{F}) = 0$, and extend it to terminating edge:
\begin{equation*}
    \forall s \in \gS^f \ \ \hat{F}(s \rightarrow s_f) \defeq R(s)
\end{equation*}
Now that $\hat{F}$ is defined for all edges in $G$, we can write that 
\begin{equation*}
\forall s' \in \gS \ \sum_{s \in Par(s')} \hat{F}(s \rightarrow s') = \sum_{s'' \in Child(s)} \hat{F}(s' \rightarrow s'').
\end{equation*} 
Which, according to \cref{prop:flow-matching}, means that there exists a Markovian flow $F \in \gF_{Markov}(G, R)$ such that $\gH_{edge}(F) = \hat{F}$. The converse
\begin{equation*}
\forall F \in \gF_{Markov}(G, R) \ \ \gL_{FM}(\gH_{edge}(F)) = 0
\end{equation*} 
 is a trivial consequence of \cref{prop:flow-matching}.

This is the loss function proposed in~\cite{bengio2021flow}. $\delta$ allows to reduce the importance given to small flows (those smaller than $\delta$), and the usage of the square of the log-ratio is justified as a way to ensure that states with large flows do not contribute to the gradients of $\gL_{FM}$ much more than states with small flows.

\end{example}

\begin{example}
\label{ex:db-loss}
 {\em Detailed-balance loss:} Consider the transition probabilities parametrization $(\gO_{PFB}, \Pi_{PFB}, \gH_{PFB})$, and the function $L_{DB}: \gO_{PFB} \times \sA \rightarrow \R^+$ defined for each $(\hat{F}, \hat{P}_F, \hat{P}_B) \in \gO_{PFB}$ and $s \rightarrow s' \in \sA$ as
 \begin{align*}
     L_{DB}(\hat{F}, \hat{P}_F, \hat{P}_B, s\rightarrow s') = \begin{cases} \left( \log \left( \frac{\delta + \hat{F}(s) \hat{P}_F(s' \mid s) }{\delta + \hat{F}(s') \hat{P}_B(s \mid s')} \right)\right)^2 \quad \text{ if } s' \neq s_f, \\
 \left( \log \left( \frac{\delta + \hat{F}(s) \hat{P}_F(s' \mid s) }{\delta + R(s)} \right)\right)^2 \quad \text{otherwise},
     \end{cases}
 \end{align*}
 where $\delta \geq 0$ is a hyper-parameter. The function $\gL_{DB}$ mapping each $(\hat{F}, \hat{P}_F, \hat{P}_B) \in \gO_{PFB}$ to 
 \begin{equation*}
     \gL_{DB}(\hat{F}, \hat{P}, \hat{P}_B)) = \sum_{s\rightarrow s' \in \sA} L_{DB}(\hat{F}, \hat{P}, \hat{P}_B, s\rightarrow s')
 \end{equation*}
 is a flow-matching loss  that is (by definition) edge-decomposable. 
 The proof of this statement is similar to the one of the example above, using \cref{prop:detailed-balance}.
\end{example}

According to \cref{sec:free-P_B}, the reward function does not completely specify the flow. Thus, the detailed-balance loss of \cref{ex:db-loss} can be used with the $(\gO_{PF}, \Pi_{PF}, \gH_{PF})$ parametrization, using any function $\hat{P}_B \in \gO_3$ as input to the detailed-balance loss.

\begin{example}
\label{ex:tb-loss}
\emph{Trajectory-balance loss:} This loss has been introduced in~\citet{malkin2022trajectory} for the parametrization $(\gO_{TB}, \Pi_{TB}, \gH_{TB})$, where $\gO_{TB} = \gO_1 \times \gO_2 \times \gO_3$, with $\gO_1 = \R^+$ parametrizes the partition function $\hat{Z}$, and $\gO_2$ and $\gO_3$ introduced in \cref{ex:PFparam} and \cref{ex:PFBparam} (the set of forward and backward probabilities consistent with $G$). $\gH_{TB}$ maps a Markovian flow in $\gF_{Markov}(G, R)$ to the corresponding triplet $(Z, P_F, P_B)$, and $\Pi_{TB}$ maps a parametrization $(\hat{Z}, \hat{P}_F, \hat{P}_B)$ to a probability over trajectories defined by $\hat{P}_F$ as in \cref{ex:PFparam}. \cref{prop:flow-parametrizations} justifies the validity of this parametrization. The loss $\mathcal{L}_{TB}$ maps each $(\hat{Z}, \hat{P}_F, \hat{P}_B) \in \gO_{TB}$ to:
\begin{equation*}
    \gL_{TB}(\hat{Z}, \hat{P}_F, \hat{P}_B) = \sum_{\tau \in \gT} L_{TB}(\hat{Z}, \hat{P}_F, \hat{P}_B, \tau),
\end{equation*}
where
\begin{equation}
    \forall \tau=(s_0, \dots, s_{n+1}=s_f) \in \gT \ \ L_{TB}(\hat{Z}, \hat{P}_F, \hat{P}_B, \tau) = \left( \log \frac{\hat{Z} \prod_{t=1}^{n+1} \hat{P}_F(s_t \mid s_{t-1})}{R(s_n) \prod_{t=1}^n \hat{P}_B(s_{t-1} \mid s_{t})} \right)^2.
\label{eq:L_TB}
\end{equation}
\citet{malkin2022trajectory} prove that $\gL_{TB}$ is a flow-matching loss and call it {\em trajectory balance}. It is trajectory-decomposable by definition.
\end{example}

\paragraph{Training by stochastic gradient descent:}
\label{sec:practical}
In the examples of the previous section, given a GFlowNet $(G, R, \gO, \Pi, \gH)$ and a flow-matching loss $\gL$, objects $o \in \gO$ are themselves functions or combinations of functions, and we can thus parametrize $\gO$ with function approximators such as Neural Networks. However, most of the times, the evaluation (let alone the minimization) of $\gL(o)$ is intractable, given that even with a full support distribution, only a subset of edges (or states or trajectories) can be visited in finite time. In practice, with an edge-decomposable loss e.g., we resort to a stochastic gradient, such as
\begin{equation}
    \label{eq:StochGradDecomposableLoss}
    \nabla_o L(o, s\rightarrow s'), \ \ s \rightarrow s' \sim \pi_o
\end{equation}
for edge-decomposable losses, or
\begin{equation}
    \label{eq:StochGradDecomposableLoss-trajectories}
    \nabla_o L(o, \tau), \ \ \tau \sim \pi_o
\end{equation}
for trajectory-decomposable losses, where $\pi_o$, called the \textbf{training distribution}, is a distribution over edges or trajectories that can be associated with $\Pi(o)$, corresponding to the online setting in RL, or defined in other ways, corresponding to the behavior policy in offline RL, see \cref{sec:offline} below.

\subsection{Extensions}
In this section, we discuss possible relaxations to the GFlowNet training paradigm introduced thus far. 
\subsubsection{Introducing Time Stamps to Allow Cycles}
Note that the state-space of a GFlowNet can easily be modified to accommodate an underlying state space for which the transitions do not form a DAG, e.g., to allow cycles. Let ${\cal S}$ be such an underlying state-space. Define the augmented state space ${\cal S}'={\cal S} \times \mathbb{N}$, where $\mathbb{N}=\{0,1,2,\ldots\}$ is the set of natural numbers, and $s'_t = (s_t, t)$ is the augmented state, where $t$ is the position of the state $s_t$ in the trajectory. With this augmented state space, we automatically avoid cycles. Furthermore, we may design or train the backwards transition probabilities $P_B(s'_t \mid s'_{t+1}=(s_{t+1}, t+1))$ to create a preference for shorter paths towards $s_{t+1}$, as discussed in~\cref{sec:free-P_B}. Note that we can further generalize this setup by replacing $\mathbb{N}$ with any totally ordered indexing set; the augmented state space will still have an associated DAG. The ordering ``$<$'' in the original state-space is lifted to the augmented state-space: $(s_t, t) < (s'_{t'}, t')$ if and only if $t < t'$ and $s_t < s'_t$.

\subsubsection{Stochastic Rewards}
\label{sec:stochastic-rewards}
We also consider the setting in which the given reward is stochastic rather than being a deterministic function of the state, yielding training procedures based on stochastic gradient descent. For example, with the trajectory balance loss of \cref{eq:L_TB}, if $R(s)$ is stochastic (even when given $s$), we can think of what is being really optimized is the squared loss with $\log R(s)$ replaced by its expectation (given $s$). This is a straightforward consequence of minimizing the expected value of a squared error loss (as for example in neural networks trained with a squared error loss and a stochastic target output, where the neural network effectively tries to estimate the expected value of that target).


\subsubsection{GFlowNets can be trained offline}
\label{sec:offline}
 
As discussed in \cref{sec:practical}, we do not need to train a GFlowNet using samples from its own trajectory distribution $\hat{P} = \Pi(o)$. Those training trajectories can be drawn from any training distribution $\pi_T$ with full support, as already shown by~\citet{bengio2021flow}. It means that a GFlowNet can be trained offline, as in offline reinforcement learning~\citep{ernst2005tree,riedmiller2005neural,lange2012batch}.

It should also be noted that with a proper adaptive choice of $\pi_T$, and assuming that computing $R$ is cheaper or comparable in cost to running the GFlowNet on a trajectory, it should be more efficient to continuously draw new training samples from $\pi_T$ than to rehearse the same trajectories multiple times. An exception would be rehearsing the trajectories leading to high rewards if these are rare.

How should one choose the training distribution $\pi_T$? It needs to cover the support of $R$ but if it were uniform it would be very wasteful and if it were equal to the current GFlowNet policy $\pi$ it might not have sufficient effective support and thus miss modes of $R$, i.e., regions where $R(x)$ is substantially greater than 0 but $R(x)\gg P_T(x)$. Hence the training distribution should be sampled from an exploratory policy that visits places that have not been visited yet and may have a high reward. High epistemic uncertainty around the current policy would make sense and the literature on acquisition functions for Bayesian optimization~\citep{srinivas2009gaussian} may be a good guide. More generally, this means the training distribution should be adaptive. For example, $\pi_T$ could be the policy of a second GFlowNet trained mostly to match a different reward function that is high when the losses observed by the main GFlowNet are large. It would also be good to regularly visit those trajectories corresponding to known large $R$, i.e., according to samples from $\pi$, to make sure those are not forgotten, even temporarily.

\subsection{Exploiting Data as Known Terminating States}
In some applications we may have access to a dataset of $(s,R(s))$ pairs and we would like to use them in a purely offline way to train a GFlowNet, or we may want to combine such data with queries of the reward function $R$ to train the GFlowNet. For example, the dataset may contain examples of some of the high-reward terminating states $s$ which would be difficult to obtain by sampling from a randomly initialized GFlowNet. How can we compute a gradient update for the GFlowNet parameters using such $(s,R(s))$ pairs?

If we choose to parametrize the backwards transition probabilities $P_B$ (which is necessary for implementing the detailed balance loss), then we can just sample a trajectory $\tau$ leading to $s$ using $P_B$ and use these trajectories to update the flows and forward transition probabilities along the traversed transitions. However, this alone is not guaranteed to produce the correct GFlowNet sampling distribution because the empirical distribution over training trajectories $\tau$ defined as above does not have full support. Suppose for example that the dataset only contains high-reward terminating states with $R(s)=1$. The GFlowNet could then just sample trajectories uniformly (which would be wrong, we would like the probability of most states not in the training set to be very small). On the other hand, if we combine the distribution of trajectories leading to terminal transitions in the dataset with a training distribution whose support covers all possible trajectories, then the offline property of GFlowNet guarantees that we can recover a flow-matching model.

\section{Conditional Flows and Free energies}
\label{sec:conditoinal-flows}
A remarkable property of flow networks is that we can recover the normalizing constant $Z$ from the initial state flow $F(s_0)$ (\cref{prop:flow-initial-state}). $Z$ also gives us the partition function associated with a given terminal reward function $R$ specifying the terminating flows.

What about internal states $s$ with $s_0<s<s_f$? If we had something like a normalizing constant for only the terminating flows achievable from $s$, we would be able to obtain a form of marginalization given state $s$, i.e., a conditional probability for terminating states $s'\geq s$, given $s$. Naturally, one could ask: does the flow $F(s)$ through state $s$ give us that kind of marginalization over only the downstream terminating flows? Unfortunately in general, the answer to this question is \emph{no}, as illustrated in \cref{fig:state-conditional-flow}: in this example $F(s_{2}) = 4$, whereas the sum of terminating flows achievable from $s_{2}$ is $6$ (the terminating states reachable from $s_{2}$ are $\{s_{5}, s_{6}, s_{7}\}$). The discrepancy is caused by the flow through $(s_{0}, s_{1}, s_{5})$ that contributes to the terminating flow $F(s_{5}{\rightarrow}s_{f})$, but not to $F(s_{2})$ since there is no order relation between $s_{1}$ and $s_{2}$.


\begin{figure}[ht]
     \centering
     \begin{subfigure}[b]{0.3\textwidth}
         \centering
         \includegraphics[width=\textwidth]{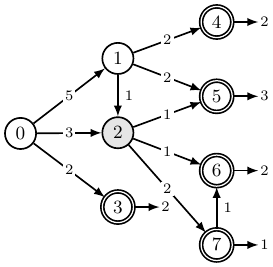}
         \caption{}
         \label{fig:state-conditional-flow-1}
     \end{subfigure}
     \hfill
     \begin{subfigure}[b]{0.3\textwidth}
         \centering
         \includegraphics[width=\textwidth]{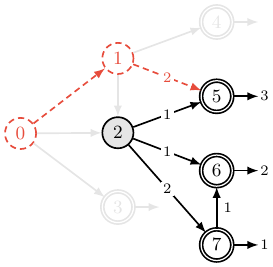}
         \caption{}
         \label{fig:state-conditional-flow-2}
     \end{subfigure}
     \hfill
     \begin{subfigure}[b]{0.3\textwidth}
         \centering
         \includegraphics[width=\textwidth]{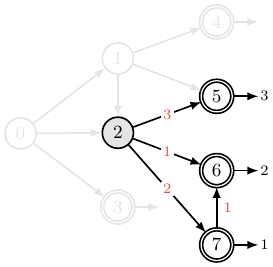}
         \caption{}
         \label{fig:state-conditional-flow-3}
     \end{subfigure}
        \caption{Example of a state-conditional flow network. (a) The original (Markovian) flow network. (b) The subgraph of states reachable from $s_{2}$; there is a flow through $(s_{0}, s_{1}, s_{5})$ that contributed to $F(s_{5}{\rightarrow}s_{f})$, but not to $F(s_{2})$, showing that $F(s_{2})$ does not marginalize the rewards of its descendant. (c) State-conditional flow network $F_{s_{2}}$, which differs from the original flow $F$ on the subgraph, but satisfies the desired marginalization property.}
        \label{fig:state-conditional-flow}
\end{figure}


In \cref{sec:marginalizing}, we show how GFlowNets applied to sampling sets of random variables can be used to estimate the marginal probability for the values given to a subset of the variables. It requires computing the kind of intractable sum discussed above (over the rewards associated with all the descendants of a state $s$, with $s$ corresponding to such a subset of variables and a descendant to a full specification of all the variables). That motivates the following definition:
\begin{definition}
\label{def:free-energy}
Given a pointed DAG $G = (\gS, \sA)$, the corresponding partial order denoted by $\geq$, and a function $\gE:\gS \rightarrow \R$, called the {\bf energy function}, we define the {\bf free energy} ${\cal F}(s)$ of a state $s$ as:
\begin{equation}
\label{eq:free-energy}
  e^{-{\cal F}(s)}\defeq \sum_{s':s'\geq s} e^{-{\cal E}(s')}.
\end{equation}
\end{definition}
Free energies are generic formulations for the marginalization operation (i.e. summing over a large number of terms) associated with energy functions, and we find their estimation to open the door to interesting applications where expensive MCMC methods would typically be the main approach otherwise. 

\subsection{Conditional flow networks}
In \cref{sec:trajectories-and-flows}, we defined a flow network as a DAG, augmented with some function $F$ over the set of complete trajectories $\gT$. We can extend this notion of flow networks by conditioning each component on some information $x$. In general, this conditioning variable can represent any conditioning information, either external to the flow network (but influencing the terminating flows), or internal (e.g., $x$ can be a property of complete trajectories over another flow network, like passing through a particular state).

\begin{definition}
\label{def:conditional-flow-net}
Let $\mathcal{X}$ be a set of conditioning variables. We consider a family of DAGs $G_x = (\gS_x, \mathcal{A}_x)$ indexed by $x \in \mathcal{X}$, along with a family of initial and terminal states denoted by $(s_{0}\mid x) \in \gS_x$ and $(s_{f}\mid x) \in \gS_x$ respectively. For each DAG $G_x$, we denote by $\gT_x$ the set of complete trajectories in $G_x$, and we denote by $\gT$ their union: 
\begin{equation*}
\gT = \bigcup_{x\in \mathcal{X}} \gT_x.
\end{equation*}
A \textbf{conditional flow network} is the specification of $\mathcal{X}$, the family $\{G_x, \ x \in \mathcal{X}\}$, along with a conditional flow function $F$, i.e. a function $F: \mathcal{X} \times \gT \rightarrow \R^+$ such that $F(x, \tau) = 0$ if $\tau \notin \gT_x$. For clarity, we will denote, for each $x \in \mathcal{X}$, by $F_x$ the function mapping each $\tau \in \gT_x$ to $F(x, \tau)$. Similar to \cref{sec:trajectories-and-flows}, $F_x$ induces a measure of the $\sigma$-algebra $2^{\gT_x}$ for each $x$.
\end{definition}

Conditional flow networks effectively represent a family of flow networks, indexed by the value of $x$. Since conditional flow networks are defined using the same components as an unconditional flow network, they inherit from all the properties of flow networks for all DAGs $G_{x}$ and flow functions $F_{x}$. In particular, we can directly extend the notion of probability distribution over flows, state and edge flows, forward and backward transition probabilities (\cref{sec:flow-probability}), of Markovian flows (\cref{sec:markovian-flows}), and any flow matching condition (\cref{sec:flow-matching-conditions}) to conditional flows; the only difference is that now every term explicitly depends of the conditioning variable $x$.

In \cref{subsec:rewardcond-flownet} and \cref{subsec:statedcond-flownet}, we will elaborate two important examples of conditional flow networks: flow networks conditioned on external information that changes the reward $R(s \mid x)$, and state-conditional flow networks that depend on internal information, i.e., previously visited states.

\subsection{Reward-conditional flow networks}
\label{subsec:rewardcond-flownet}
\begin{definition}
\label{def:reward-cond-flow-net}
Let $\mathcal{X}$ be a set of conditioning variables. Consider a flow network given by a pointed DAG $G = (\gS, \sA)$ and a flow function $F$. Consider a family $\mathcal{R}$ of non-negative functions of $\gS$: $\{R_x: \gS^f \rightarrow \R^+, \ x \in \mathcal{X} \}$. A \textbf{reward-conditional flow network} compatible with the family $\mathcal{R}$ is a conditional flow network (\cref{def:conditional-flow-net}), with $G_x=G$ for every $x\in\mathcal{X}$, such that the edge-flow functions induced by the conditional flow function $F$ satisfy:
\begin{align*}
    \forall x \in \mathcal{X}  \ \ \forall s \in \gS^f \ \ F_x(s \rightarrow s_f) = R_x(s).
\end{align*}
We will use the notations $R_x(s)$ and $R(s \mid x)$ interchangeably.
\end{definition}
Note that the definition above implies that all the DAGs of a reward-conditional flow network are identical, and only the terminating flows differ amongst the members of the family.

\begin{example}
We will see in \cref{sec:conditional-gflownets} that we can estimate a conditional flow network using a GFlowNet (\cref{sec:GFlowNets-Learning-a-Flow}), given a reward function $R(s \mid x)$. In an Energy-Based Model, the model $P_{\theta}(s)$ is associated with a given energy function $\gE_{\theta}(s)$, parametrized by $\theta$, with
\begin{equation*}
P_{\theta}(s) = \frac{\exp(-\gE_{\theta}(s))}{Z(\theta)}.
\end{equation*}
This model can be parametrized using a reward-conditional flow network, conditioned on $\theta$ with the reward function $R(s \mid \theta) = \exp(-\gE_{\theta}(s))$. We show in \cref{sec:energy-based-gfn} how to use such a conditional flow-network to learn an Energy-Based Model.
\end{example}

\subsection{State-conditional flow networks}
\label{subsec:statedcond-flownet}
\begin{definition}
\label{def:state-cond-flow-net}
Consider a flow network given by a DAG $G = (\gS, \sA)$ and a flow function $F$. For each state $s \in \gS$, let $G_s$ be the subgraph of $G$ containing all the states $s'$ such that $s' \geq s$. A \textbf{state-conditional flow network} is given by the family $\{G_s, \ s \in \gS\}$, along with a conditional flow function $F: \gS \times \gT \rightarrow \R^+$, where $\gT = \bigcup_{s \in \gS} \gT_s$, and $\gT_s$ the set of complete trajectories in $\mathcal{G}_s$, that satisfies:
\begin{equation}
\label{eq:terminating-state-conditional-flows}
    F_s( s' \rightarrow s_f) = F(s' \rightarrow s_f).
\end{equation}
\end{definition}

Note that in the definition above, we abused the notation $F$ to refer to both flow functions and edge flow functions, but also used $F_s$ to refer to the conditional flow function (or the corresponding edge flow function) $\tau \mapsto F(s, \tau)$. Unlike the reward-conditional flow networks defined in \cref{subsec:rewardcond-flownet}, the structure of the DAG in a state-conditional flow network depends on the anchor state $s$. In particular, this means that the initial state $(s_{0}\mid s) = s$ changes, but the final state $(s_{f}\mid s) = s_{f}$ remains unchanged, for any state $s$.

Since the definition of a state-conditional flow network depends on an original flow network, we must ensure that this definition is indeed correct, i.e. that such a state-conditional flow network that satisfies the conditions in \cref{eq:terminating-state-conditional-flows} exists.

\begin{proposition}
\label{prop:existence-state-conditional-flow}
For any flow network given by a DAG $G = (\gS, \sA)$ and a flow F, we can define a state-conditional flow network as per \cref{def:state-cond-flow-net}.
\end{proposition}

\begin{proof}
Let $s \in \gS$ be a state. Since the structure of the DAG $G_{s}$ is clearly well-defined, we just need to show that there exists a flow function $F_{s}: \gT_s \rightarrow \R^+$ that satisfies  \cref{eq:terminating-state-conditional-flows}. If such a function exists for every $s\in \gS$, then it would suffice to define the conditional flow function $F: \gS \times \gT \rightarrow \R^+$ as:
\begin{equation*}
    F(s, \tau) = \begin{cases}
    F_s(\tau) \quad \text{if } \tau \in \gT_s \\
    0 \quad \text{otherwise}.
    \end{cases}
\end{equation*}
Let $A_{s'\mid s}$ be the set of complete trajectories in $\gT_{s}$ terminating in $s' \geq s$; the condition in  \cref{eq:terminating-state-conditional-flows} then reads:
\begin{equation}
\label{eq:A_s'_s'}
F_{s}(s'{\rightarrow}s_{f}) = F_{s}(A_{s'\mid s}) = \sum_{\tau \in A_{s'\mid s}}F_{s}(\tau) = F(s'{\rightarrow}s_{f}).
\end{equation}
Note that in \cref{eq:A_s'_s'}, $F(s'{\rightarrow}s_{f})$ is a given quantity because the flow $F$ is known. Since the sets of trajectories $\{A_{s'\mid s} , \ , s' \geq s\}$ form a partition of all the complete trajectories $\gT_{s}$, \cref{eq:A_s'_s'} is a system of linear equations, whose unknowns are $F_{s}(\tau)$ for all $\tau \in \gT_{s}$, where each equation involves separate sets of unknowns. Therefore there exists at least a solution $F_{s}(\tau)$ of this system.

We can construct such a solution in the following way. For some $\tau \in \gT_{s}$, we can first start by selecting the complete trajectories $\bar{\tau} \in \gT$ that contain $\tau$:
\begin{equation*}
C_{\tau} = \{\bar{\tau} \in \gT\,:\, \tau \subseteq \bar{\tau}\}.
\end{equation*}
The key difference between the DAG $G$ and the subgraph $G_{s}$ though is that $G$ may contain trajectories that terminate in some $s' \geq s$ but do not pass through $s$, and those are therefore not covered by the trajectories of $G_{s}$. Let $U_{s'\mid s}$ be the set of complete trajectories of $G$ defined as

\begin{equation*}
U_{s'\mid s} = \{\bar{\tau} \in \gT \,:\, \exists s'' > s, s'' \in \bar{\tau}, s' \in \bar{\tau}, s \notin \bar{\tau}\}.
\end{equation*}
For all $\tau \in \gT_{s}$ such that $\tau$ terminates in some $s' \geq s$, we can therefore construct the flow $F_{s}(\tau)$ as
\begin{equation*}
F_{s}(\tau) := F(C_{\tau}) + \frac{1}{n}F(U_{s' \mid s}),
\end{equation*}
where $n = |A_{s'\mid s}|$ is the number of trajectories $\tau' \in \gT_{s}$ that terminate in $s'$. It is easy to verify that $F_{s}(\tau)$ is a solution of \cref{eq:A_s'_s'}.
\end{proof}

While we saw in \cref{prop:flow-initial-state} that the initial flow $F(s_{0})$ was equal to the partition function, the initial state-conditional flow also benefit from a marginalizing property, and is now related to the free energy at $s$.

\begin{proposition}
\label{prop:free-energy-state-cond-flow-net}
Given a state-conditional flow network $(G, F)$ as in \cref{def:state-cond-flow-net}, for any state $s$, the initial flow of the state-conditional flow network corresponds to marginalizing the terminating flows $F(s' \rightarrow s_f)$ for $s' \geq s$:
\begin{equation*}
    F_{s}(s_{0} \mid s) = F_{s}(s) = \sum_{s'\,:\,s'\geq s}F(s' \rightarrow s_f) = \exp(-\gF(s)),
\end{equation*}
where $\gF(s)$ is the free energy associated to the energy function $\gE(s') = -\log F(s'{\rightarrow}s_f)$.
\end{proposition}

\begin{proof}
This is a direct consequence of \cref{prop:flow-initial-state}, applied to the state-conditional flow function $\gF_s$, along with  \cref{def:free-energy}.

\begin{align*}
F_{s}(s_{0} \mid s) &= \sum_{\tau \in \gT_{s}}F_{s}(\tau)= \sum_{s'\,:\,s'\geq s}F_{s}(s'{\rightarrow}s_{f})\\
&= \sum_{s'\,:\,s'\geq s}F(s'{\rightarrow}s_{f}) = \sum_{s'\,:\,s'\geq s}e^{-\gE(s')}
\end{align*}
\end{proof}

Note that the definition of state-conditional flow networks is consistent with our original definition of (unconditional) flow networks in Section 2.1, in the sense that the original flow network is a valid state-conditional flow networks anchored at the initial state.

Another quantity of interest that state-conditional flow networks allow us to evaluate, is the probability of terminating a trajectory in a state $s'$ if all terminating edge flows were diverted towards an earlier state $s < s'$:

\begin{corollary}
\label{cor:state-conditional-P-T}
Consider a flow network given by a DAG $G=(\gS, \sA)$ and a flow $F$, from which we define any state-conditional flow network, as per \cref{def:state-cond-flow-net}. Given a state $s$, the flow function $F_s$ induces a probability distribution over $\{s'' \in \gS^f: \  s'' \geq s\} \subseteq \gS^f$, that we denote by $P_T(. \mid s)$. 

Under this measure, the probability of terminating a trajectory in $G_s$ in a state $s'$ (i.e. the last edge of the trajectory is $s' \rightarrow s_f$) is:
\begin{equation}
    P_T(s' \mid s) = \mathbf{1}_{s'\geq s} e^{-\gE(s') + \gF(s)},
\end{equation}
where $\gE$ is the energy function mapping each state $s'$ that is parent of $s_f$ to $-\log F(s'\rightarrow s)$, and $\gF$ is the corresponding free energy function.
\end{corollary}
\begin{proof}
Because $F_s$ is a flow function, \cref{def:p} and \cref{prop:flow-initial-state} tell us that:
\begin{align*}
    P_T(s' \mid s) = \begin{cases}
    \displaystyle \frac{F_s(s' \rightarrow s_f)}{F_s(s)} & \text{if } s' \geq s \\
    0 & \text{otherwise}
    \end{cases}
\end{align*}
Combining this with \cref{prop:free-energy-state-cond-flow-net}, and \cref{eq:terminating-state-conditional-flows}, we obtain for $s' \geq s$:
\begin{align*}
    P_T(s' \mid s)
    &= \mathbf{1}_{s' \geq s} \frac{F(s'\rightarrow s_f)}{e^{-\gF(s)}}\\
    &= \mathbf{1}_{s' \geq s} \frac{e^{-\gE(s')}}{e^{-\gF(s)}}
\end{align*}
\end{proof}

\subsection{Conditional GFlowNets}
\label{sec:conditional-gflownets}
Similar to the way we used a GFlowNet to estimate the flow of a flow network, we can also use a (conditional) GFlowNet in order to estimate a conditional flow network, with given target reward functions. A conditional GFlowNet follows the construction presented in Section 3, with the exception that all quantities to be learned now depend on the conditioning variable $x \in \mathcal{X}$ (e.g., $x$ is an additional input of the neural network).

All parametrizations and losses presented in \cref{sec:flow-matching-losses} could in principle be used to train a conditional GFlowNet, regardless of the conditioning set. 
Below we discuss yet another loss, first presented in \citet{deleu2022bayesian}, that could be used to train both GFlowNets and conditional GFlowNets.
\begin{example}
\label{ex:conditionalPFBparam}
Given a family of DAGs $G_x$ and reward functions $R_x$ indexed by $x \in \mathcal{X}$, where each state $s \in G_x$ is terminating (i.e. is a parent of $s_f$), and following \cref{ex:PFparam,ex:PFBparam}, we consider a parametrization given by the forward and backward transition probabilities $\gO^x_P = \gO^x_2 \times \gO^x_3$, where $\gO^x_2$ (resp. $\gO^x_3$) is the set of forward (resp. backward) probability functions $\hat{P}_F$ (resp. $\hat{P}_B$) consistent with $G_x$ for every $x \in \mathcal{X}$, and $(\Pi^x_P, \gH^x_P)$ defined as in \cref{ex:PFBparam,ex:PFBparam}. Each $(\gO^x_P, \Pi^x_P, \gH^x_P)$ is a flow parametrization of $(G_x, R_x)$, which can be trained with an edge-decomposable flow-matching loss, as proved in \citet{deleu2022bayesian}, and defined for every $s \rightarrow s' \in \sA^{-f}$:
\begin{align*}
    \gL_{DB}(\hat{P}_F, \hat{P}_B, s \rightarrow s', x) = \left(\log \frac{R_x(s') P_B(s \mid s', x) P_F(s_f \mid s, x)}{R_x(s) P_F(s' \mid s, x) P_F(s_f \mid s', x)} \right)^2
\end{align*}
\end{example}

\subsection{Training Energy-Based Models with a GFlowNet}
\label{sec:energy-based-gfn}
A GFlowNet can be trained to convert an energy function into an approximate corresponding sampler. Thus, it can be used as an alternative to MCMC sampling (\cref{sec:MCMC}). Consider the model $P_\theta(s)$ associated with a given parametrized energy function ${\cal E}_\theta(s)$ with parameters $\theta$: $P_\theta(s)=\frac{e^{-{\cal E}_\theta(s)}}{Z}$. Sampling from $P_\theta(s)$ could be approximated by sampling from the terminating probability distribution $P_T(s)$ of a GFlowNet trained with target terminal reward $R(s)=e^{-{\cal E}_\theta(s)}$ (see~\cref{eq:P_T_o}). In practice, $\hat{P}_T$ would be an estimator for the true $P_\theta$ because the GFlowNet training objective is not zeroed (insufficient capacity or finite training time). The GFlowNet samples drawn according to $\hat{P}_T$ could then be used to obtain a stochastic gradient estimator for the negative log-likelihood of observed data $x$ with respect to parameters $\theta$ of an energy function ${\cal E}_\theta$:
\begin{equation}
    \frac{\partial -\log P_\theta(x)}{\partial \theta} = \frac{\partial {\cal E}_\theta(x)}{\partial \theta} - \sum_{s} P_\theta(s) \frac{\partial {\cal E}_\theta(s)}{\partial \theta}.
\end{equation}
An approximate stochastic estimator of the second term could thus be obtained by sampling one or more terminating states $s \sim \hat{P}_T(s)$, i.e., from the trained GFlowNet's sampler. Furthermore, if the GFlowNet's loss is $0$, i.e. $\hat{P}_T = P_\theta$, the gradient estimator would be unbiased.

One could thus potentially jointly train an energy function ${\cal E}_\theta$ and a corresponding GFlowNet by alternating updates of $\theta$ using the above equation (with sampling from $P_\theta$ replaced by sampling from $\hat{P}_T$) and updates of the GFlowNet using the updated energy function for the target terminal reward. 

If we fix $\hat{F}(s{\rightarrow}s_f)=R(s)$ by construction (which we can do if the reward function is deterministic), then we can parametrize the energy function with the same neural network that computes the flow, since ${\cal E}(s)=-\log R(s)=-\log \hat{F}(s{\rightarrow}s_f)$. Hence {\em the same parameters are used for the energy function and for the GFlowNet}, which is appealing.

The above strategy for learning jointly an energy function and how to sample from it could be generalized to {\bf learning conditional distributions by using a conditional GFlowNet} instead. Let $x$ be an observed random variable and $h$ be a hidden variable, with the GFlowNet generating the pair $(x,h)$ in two sub-trajectories: either first generate $x$ and then generate $h$ given $x$, or first generate $h$ and then generate $x$ given $h$. This can be achieved by introducing a 6-valued component $u$ in the state to make sure that both $h$ and $x$ are generated before exiting into $s_f$, with the following values and constraints:
\begin{align}
    s=s_0 &\Rightarrow u= 0\\
    s=s_f &\Rightarrow u=5\\
    (u_t{\rightarrow}u_{t+1})& \in \{0{\rightarrow}1, 0{\rightarrow}2,1{\rightarrow}1,2{\rightarrow}2,1{\rightarrow}3,2{\rightarrow}4,3{\rightarrow}3,4{\rightarrow}4,3{\rightarrow}5,4{\rightarrow}5\}
\end{align}
where $u=1$ indicates that $x$ is being generated (before $h$), $u=2$ indicates that $h$ is being generated (before $x$), $u=3$ indicates that $h$ is being generated (conditioned on $x$), and $u=4$ indicates that $x$ is being generated (given $h$). The GFlowNet cannot reach the final state $s_f$ until both $x$ and $h$ have been generated. The conditional GFlowNet can thus approximately sample $P_T(x)$, $P_T(h \mid x)$, $P_T(h)$, $P_T(x \mid h)$ as well as $P_T(x,h)$. If we only want to sample $x$ (or only $h$), we allow exiting as soon as it is generated (resp. $h$ is generated). See~\cref{sec:marginalizing} for a more general discussion on how to represent, estimate and sample marginal distributions.

Let us denote by $P_\theta(x,h)\propto e^{-{\cal E}_\theta(x,h)}$ the joint distribution over $(x,h)$ associated with the energy function, i.e., with $F(s\rightarrow s_f)$. When $x$ is observed but $h$ is not, $\theta$ could thus be updated by approximating the marginal log-likelihood gradient
\begin{equation}
        \frac{\partial -\log P_\theta(x)}{\partial \theta} = \sum_h P_T(h \mid x) \frac{\partial {\cal E}_\theta(x,h)}{\partial \theta} - \sum_{x',h} P_T(x',h) \frac{\partial {\cal E}_\theta(x',h)}{\partial \theta}
\end{equation}
using samples from the estimated terminal sampling probabilities $\hat{P}_T$ of a trained GFlowNet to approximate in a stochastic gradient way the above sums (using one or a batch of samples).

Note how we now have outer loop updates (of the energy function, i.e., the reward function) from actual data, and an inner loop updates (of the GFlowNet) using the energy function as a driving target for the GFlowNet. How many inner loop updates are necessary for such a scheme to work is an interesting open question but most likely depends on the form of the underlying data generating distribution. If the work on GANs~\citep{goodfellow2014generative} is a good analogy, a good strategy may be to interleave updates of the energy function (as minus the log-terminal flow of a GFlowNet) based on a batch of data, and updates of the GFlowNet as a sampler based on both these samples (trajectories can be sampled backwards from a terminating state $s$ using $P_B$) and forward samples from the tempered training policy $\pi_T$ defined by the forward transition probabilities of the GFlowNet.

\subsection{Active Learning with a GFlowNet}
An interesting variant on the above scheme is one where the GFlowNet sampler is used not just to produce negative examples for the energy function but also to actively explore the environment. \citet{jain2022biologicalseqdesign} use an active learning scheme where the GFlowNet is used to sample candidates $x$ for which we expect the reward $R(x)$ to be generally large (since the GFlowNet approximately samples proportionally to $R(x)$). The challenge is that evaluating the true reward $R$ for any $x$ is computationally expensive and can potentially be noisy (for example, a biological assay to measure the binding energy of a drug to a given target protein). Thus, instead of using the true reward directly, the authors introduce a proxy $\hat{f}$ (which approximates the true reward function $f$), which is used to train the GFlowNet. This would lead to a setup similar to~\cref{sec:energy-based-gfn}, with an inner loop where a GFlowNet is trained to match the proxy $\hat{f}$, and an outer-loop where the proxy $\hat{f}$ is learned in a supervised fashion using $(x, y)$ pairs, where $x$ is proposed by the GFlowNet, and $y$ is the corresponding \textit{true} reward from the environment (for example, outcome of a biological of chemical assay). It is important to note here that the GFlowNet and the proxy are intricately linked since the coverage of proxy $\hat{f}$ over the domain of $x$ relies on diverse candidates from the GFlowNet. And similarly, since the GFlowNet matches a reward distribution defined by the proxy reward function $\hat{f}$, it also depends on the quality of the true reward function $f$. 

This setup can be further extended by incorporating information about how \textit{novel} a given candidate is, or how much epistemic uncertainty, $u(x, f)$, there is in the prediction of $\hat{f}$. We can use the acquisition function heuristics (like Upper Confidence Bound (UCB) or Expected Improvement (EI)) from Bayesian optimization~\citep{movckus1975bayesian,srinivas2009gaussian} to combine the predicted usefulness $\hat{f}(x)$ of configuration $x$ with an estimate of the epistemic uncertainty around that prediction. Using this as the reward can allow the GFlowNet to explore areas where the predicted usefulness is high ($\hat{f}(x)$ is large) and at the same time explore areas where there is more information to be gathered about useful configurations of $x$. The uncertainty over the predictions of $\hat{f}$ with the appropriate acquisition function can provide more control over the exploratory behaviour of GFlowNets.

As discussed by~\citet{bengio2021flow} when comparing GFlowNets with return-maximizing reinforcement learning methods, an interesting property of sufficiently trained GFlowNets is that they will sample from all the modes of the reward function, which is particularly desirable in a setting where exploration is required, as in active learning. The experiments in the paper also demonstrate this advantage experimentally in terms of the diversity of the solutions sampled by the GFlowNet compared with PPO, an RL method that had previously been used for generating molecular graphs and that tends to focus on a single mode of the reward function.

\subsection{Estimating Entropies, Conditional Entropies and Mutual Information}
\label{sec:mutual-information}

\begin{definition}
Given a reward function $R$ with $0\leq R(s)<1$ $\forall s$, we define the {\bf entropic reward function} $R'$ associated with $R$ as:
\begin{equation}
\label{eq:R'}
    R'(s) = - R(s) \log R(s). 
\end{equation}
\end{definition}
In brief, in this section, we show that we can estimate entropies by training two GFlowNets: one that estimates flows as usual for a target terminal reward function $R(s)$, and one that estimates flows for the corresponding entropic reward function. We show below that we obtain an estimator of entropy by looking up the flow in the initial state, and if we do this exercise with conditional flows, we get conditional entropy. Once we have the conditional entropy, we can also estimate the mutual information. 

\begin{proposition}
\label{prop:entropy}
Consider a flow network $(G, F)$ such that the terminating flows match a given reward function $R$, i.e. $\forall s \in \gS^f, \ F(s\rightarrow s_f) = R(s)$, with $R(s)<1$ for all $s$, and a second flow network $(G, F')$ with the same pointed DAG, but with a flow function for which the terminating flows match the entropic reward function $R'$ (\cref{eq:R'}), then the entropy $H[S]$ associated with the terminating state random variable $S \in \gS^f$ with distribution $P_T(S=s) = \frac{R(s)}{Z}$ (\cref{eq:flow-P_T}) is
\begin{equation}
\label{eq:H}
 H[S] \defeq - \sum_s P_T(s) \log P_T(s) = \frac{F'(s_0)}{F(s_0)} + \log F(s_0).
\end{equation}
\end{proposition}
\begin{proof}
First apply the definition of $P_T(s)$, then~\cref{eq:F0=Z} on both flows:
\begin{align*}
    - \sum_s P_T(s) \log P_T(s) &= - \sum_{s}\frac{R(s)}{F(s_0)} (\log R(s) - \log F(s_0)) \\
    &= \frac{\big(-\sum_s R(s) \log R(s)\big) + \big(\log F(s_0) \,\sum_s R(s)\big)}{F(s_0)} \\
    &= \frac{F'(s_0)}{F(s_0)} + \log F(s_0).
\end{align*}
Note that we need $R(s)<1$ to make sure that the rewards $R'(s)$ (and thus the flows) are positive.
\end{proof}

\begin{proposition}
Given a set $\mathcal{X}$ of conditioning variables, consider a conditional flow network defined by a conditional flow function $F$, for which the terminating flows match a target reward $R$ family (conditioned on $x \in \mathcal{X}$) that satisfies $R_x(s)<1$ for all $s$, and a second conditional flow network defined by a conditional flow function $F'$, for which the terminating flows match the entropic reward functions $R'_x$  (\cref{eq:R'}), then the conditional entropy $H[S \mid x]$ of random terminating states $S \in \gS^f$ consistent with condition $x$ is given by
\begin{equation}
\label{eq:cond-entropy}
    H[S \mid x] = \frac{F'(s_0 \mid x)}{F(s_0 \mid x)} + \log F(s_0 \mid x).
\end{equation}
In particular, for a state-conditional GFlowNet ($\mathcal{X}=\gS$ is the state space of the DAG), we obtain
\begin{equation}
\label{eq:state-cond-entropy}
    H[S \mid s] = \frac{F'(s \mid s)}{F(s \mid s)} + \log F(s \mid s).
\end{equation}
More generally, the mutual information MI$(S;X)$ between the random draw of a terminating state $S=s$ according to $P_T(s \mid x)$ and the conditioning random variable $X$ is
\begin{equation}
    \label{eq:MI}
    {\rm MI}(S;X) = H[S]-E_X[H[S \mid X]] = 
    \frac{F'(s_0)}{F(s_0)} + \log F(s_0) -
    E_X\bigg[ \frac{F'(s_0 \mid X)}{F(s_0 \mid X)} + \log F(s_0 \mid X)\bigg]
\end{equation}
where $F(s)$ and $F'(s)$ indicate the unconditional flows (trained with no condition $x$ given) while $F(s \mid x)$ and $F'(s \mid x)$ are their conditioned counterparts.
\end{proposition}
\begin{proof}
The proof of \cref{eq:cond-entropy} follows from the fact that each $(G_x, F_x)$ is a flow network, to which we can apply \cref{prop:entropy}. \cref{eq:MI} is a direct consequence of the definition of the Mutual Information, \cref{eq:H} and \cref{eq:cond-entropy}.
\end{proof}

If we have a sampling mechanism for $P(X)$, we can thus approximate the expectation in~\cref{eq:MI} by a Monte-Carlo average with draws from $P(X)$.

\section{GFlowNets on Sets, Graphs, and to Marginalize Joint Distributions}
\label{sec:GFlowNets-on-Sets}

\subsection{Set GFlowNets}

We first define an action space for constructing sets and we view the GFlowNet as a means to generate a random set $S$ and to estimate quantities like probabilities, conditional probabilities or marginal probabilities for realizations of this random variable. The elements of those sets are taken from a larger ``universe" set $\cal U$.
\begin{definition}
\label{def:set-gflownets}
Given a ``universe'' set $\mathcal{U}$, consider the pointed DAG $G=(\gS, \sA)$, where $\gS \defeq 2^{\mathcal{U}} \cup \{s_f\}$ is the set of all subsets of $\mathcal{U}$ with an additional state $s_f$, $s_0 = \{ \}$ is the empty set, and for any two subsets $s, s'$ of $\mathcal{U}$, $s \rightarrow s' \in \sA \Leftrightarrow \exists a \in \mathcal{U} \setminus s, \ \ s' = s \cup \{a\}$; meaning that each transition in the DAG corresponds to adding one element of $\mathcal{U}$ to the current subset. Additionally all subsets are connected to $s_f$, i.e. $\forall s \in \gS, \ s \rightarrow s_f \in \sA$. A \textbf{set flow network} is a flow network on this graph $G$, and a \textbf{set GFlowNet} is an estimator of such a flow network, as defined in \cref{sec:GFlowNets-Learning-a-Flow}. The target terminal reward function $R:s \mapsto F(s \rightarrow s_f)$ satisfies:
\begin{equation}
    Z = \sum_{s \in 2^{\cal U}} R(s) < \infty.
\end{equation}
\end{definition}
A set flow network defines a terminating probability distribution $P_T$ on states (see~\cref{def:flow-P_T} and~\cref{eq:free-energy}), with 
\begin{equation}
    P_T(s)=e^{-{\cal E}(s)+{\cal F}(s_0)}=\frac{F(s{\rightarrow}s_f)}{F(s_0)}
\end{equation}
where $\cal E$ represents the energy $-\log R$. Similarly,~\cref{cor:state-conditional-P-T} provides us with a formula for conditional probabilities of a given superset $s'$ of a given set $s$ under $P_T$,
\begin{equation}
\label{eq:P_T-superset}
P_T(s'  \mid  s' \supseteq s) \defeq e^{-{\cal E}(s')+{\cal F}(s)} = \frac{F(s'{\rightarrow}s_f)}{F(s \mid s)}
\end{equation}
where ${\cal F}$ indicates free energy (see~\cref{def:free-energy}).

Remember that with a GFlowNet with state and edge flow estimator $\hat{F}$, it is not guaranteed that $\hat{F}(s)=R(s)$ for all states $s \in \gS \setminus \{s_f\}$, so we could estimate probabilities with
\begin{equation}
\hat{P}_T(s) = \frac{\hat{F}(s{\rightarrow}s_f)}{\hat{F}(s_0)}
\end{equation}
or alternatively
\begin{equation}
\hat{P}_T(s) = \frac{R(s)}{\hat{F}(s_0)}.
\end{equation}
Similarly, we can estimate conditional superset probabilities with~\cref{eq:P_T-superset} or with
\begin{equation}
\hat{P}_T(s' \mid s' \supseteq s) = \frac{R(s')}{\hat{F}(s \mid s)},
\end{equation}
none of which are guaranteed to exactly sum to 1.

We can also compute the marginal probability over all supersets of a given set $s$, as shown below.
\begin{proposition}
Following the notations of \cref{def:set-gflownets}, let $\mathfrak{S}(s)=\{s'\supseteq s\}$ be the set of all supersets of a set $s$. The probability of drawing any element from $\mathfrak{S}(s)$ given a set flow network is
\begin{equation}
\label{eq:P_T-continuations}
    P_T(\mathfrak{S}(s))=\sum_{s'\supseteq s} P_T(s') = \frac{e^{-{\cal F}(s)}}{Z} = \frac{F(s \mid s)}{F(s_0)}.
\end{equation}
\end{proposition}
\begin{proof}
We can rewrite the sum as follows, first applying the definition of $P_T$ (\cref{eq:flow-P_T}), and \cref{prop:free-energy-state-cond-flow-net}:
\begin{align*}
    \sum_{s'\supseteq s} P_T(s') &= \sum_{s'\geq s} \frac{F(s'{\rightarrow}s_f)}{F(s_0)} \\
    &= \frac{F(s \mid s)}{F(s_0)} \\
    &= \frac{e^{-{\cal F}(s)}}{Z}
\end{align*}
where we notice that for states that are sets, the order relationship $s \leq s'$ is equivalent to the subset relationship $s \subseteq s'$. 
\end{proof}

To summarize, a GFlowNet which is trained to match a given energy function (i.e. derived rewards) over sets can be used to represent that distribution, sample from it, estimate the probability of a set under it, estimate the partition function, search for the lowest energy set, sample a conditional distribution over supersets of a given set, estimate that conditional distribution for a given pair of set and superset, compute the marginal probability of a subset (i.e., summing over the probabilities of the supersets), and compute the entropy of the set distribution or of the conditional distribution of supersets of a set. For example,
using a GFlowNet to model distributions over sets has been successfully used in \cite{malik2023batchgfn}, where the goal is to iteratively select an informative batch of points for efficient active learning. The authors additionally show that the distribution over sets defined by a mutual information-based reward (\cref{sec:mutual-information}) can be approximated to satisfying levels using a neural network as a function approximator for the GFlowNet policy.

\subsection{GFlowNet on Graphs}

A graph is a special kind of set in which there are two kinds of elements: nodes and edges, with edges being pairs of node indices. Graphs may also have content attached to nodes and/or edges. The set operations described in the previous section can thus be specialized accordingly. Some actions (i.e. edges in the GFlowNet DAG) could insert a node while other actions could insert an edge. The set of allowable actions can be limited, for example to make sure the graph has a single connected component, or to ensure acyclicity. Like for sets in general, one cannot have an action which adds a node or edge which is already in the set. \citet{deleu2022bayesian} used a GFlowNet over DAGs in order to learn an approximation of the posterior distribution over the graphical structure of a Bayesian Network. While the GFlowNet was learned by minimizing some loss derived from flow-matching conditions as in \cref{sec:flow-matching-losses}, they showed that the resulting distribution is an accurate approximation of the target posterior. Since graphs are sets, all the GFlowNet operations on sets can be applied on graphs. 

\subsection{Marginalizing over Missing Variables}
\label{sec:marginalizing}

The ability of GFlowNets to capture probability distributions over sets can be applied to modeling the joint distribution over random variables, to calculating marginal probabilities over given subsets of variable values, and to sampling or computing probabilities for any conditional (e.g., for a subset of variables given another subset of variables).

Let $X=(X_1,X_2,\ldots,X_n)$ be a composite random variable with $n$ element random variables $X_i$, $1\leq i \leq n$, each with possible values $x_i \in {\cal X}_i$ (not necessarily numbers). If we are given an energy function or a terminal reward function $R(x)$ to score any instance $X=x$, we can train a particular kind of set GFlowNet for which the set elements are pairs $(i,x_i)$ and at most one element in the set with index $i$, for all $i$. The only allowed terminating transitions are when the set has exactly size $n$ and every size-$n$ set $s$ terminates on the next transition. 

Note how that GFlowNet can sample an $X$ in any possible order, if $P_B$ allows that order. Given an existing set of $(i,x_i)$ pairs (represented by a  state $s=\{(i,x_i)\}_i$), we can estimate the marginal probability of that subset of variables (implicitly summing over all the missing ones, see~\cref{eq:P_T-continuations}). We can sample the other variables by setting the state at $s$ and continuing to sample from the GFlowNet's policy (the learned forward transition probability function). We can sample from a chosen subset $S'$ of the other variables by constraining that policy to only add elements which are in $S'$. In addition, we can do all the other things that are feasible on set GFlowNets, such as estimating the partition function, sampling in an order which prefers the early subsequences with the largest marginal probability, searching for the most probable configuration of variables, or estimating the entropy of the distribution.

\subsection{Modular Energy Function Decomposition}

Let us see how we can apply the graph GFlowNet framework to a special kind of graph: a factor graph~\citep{kschischang2001factor} with reusable factors. This will yield a distribution $P_T(g)$ over graphs $g$, each of which is associated with an energy function value ${\cal E}(g)$ (and associated reward $R(g)$). Energy-based models are convenient because they can decompose a joint probability into independent pieces ~\citep[possibly corresponding to independent mechanisms,][]{SchJanPetSgoetal12,goyal2019recurrent,goyal2020inductive}, each corresponding to a {\em factor} of a factor graph. In our case, we would like a shared set of factors $\mathbb{F}$ to be reusable across many factor graphs $g$. The factor graph will provide an energy and a probability over a set of random variables $\cal V$. Let the graph $g=\{(F^i,v^i)\}_i$ be written as a set of pieces $(F^i,v^i)$, where $F^i \in \mathbb{F}$ is the index of a factor with energy function term ${\cal E}_{F^i}$, selected from the pool $\mathbb{F}$ of possible factors, and where $v^i=(v_1,v_2,\ldots)$ is a list of realizations of the random variables $V_{j} \leftarrow v_{j}$, where $V_j \in \mathcal{V}$ is a node of the factor graph. That list defines the edges of the factor graph connecting variable $V_j$ with the $j$-th argument of ${\cal E}_{F^i}$. Let us denote ${\cal E}_{F^i}(v^i)$ the value of this energy function term ${\cal E}_{F^i}$ applied to those values $v^i$, i.e.,
\begin{equation}
    {\cal E}_{F^i}(v^i) = {\cal E}_{F^i}(v_1,v_2,\ldots).
\end{equation}
The total energy function of such a graph can then be decomposed as follows
\begin{equation}
    \label{eq:total_energy_decomposition}
    {\cal E}(g) = \sum_i {\cal E}_{F^i}(v^i).
\end{equation}
What is interesting with this construction is that the graph GFlowNet can now sample a graph $g$, possibly given some conditioning observations $x$: see~\cref{sec:energy-based-gfn} on how GFlowNets can be trained jointly with an energy function, including the case where only some random variables are observed. Hence, given some observed variables (not necessarily always the same), the graph GFlowNet can sample a latent factor graph containing and connecting (with energy function terms) both observed and latent random variables, and whose structure defines an energy function over values of the joint observed and latent variables. 

Not only can we use the compositional nature of the objects generated by a GFlowNet to decompose the total energy into reusable energy terms corresponding to ideally independent mechanisms, but we can also decompose the GFlowNet itself into modules associated with each mechanism. The action space of this graph GFlowNet is fairly complex, with each action corresponding with the addition of a latent variable $V_k$ or the addition of a graph piece $(F^i,v^i)$. Such an action is taken in the context of the state of the GFlowNet, which is a partially constructed graph (arising from the previous actions). The GFlowNet and its associated energy function parameters are thus decomposed into modules. Each module knows how to compute an energy function ${\cal E}_{F_i}$ and how to score and sample competitively (against the other modules) a new graph piece (to insert the corresponding factor in the graph).

Consider some observed variables (a subset of the $V_j$'s with their values $v_j$), collectively denoted $x$. Consider a graph $g$ among those compatible with $x$ (i.e. with some nodes corresponding to $V_j=v_j$ for the observed variables) and denote $h$ the specification of the part of $g$ not already provided by $x$. We can think about latent variable $h$ as the explanation for the observed $x$. Note how marginalizing over all the possible $h$, we can compute the free energy of $x$. The principles of \cref{sec:energy-based-gfn} can be applied to train such an energy-based GFlowNet. It also makes sense to represent a prior over graph structures in the energy function. For example, we may prefer sparse factors (with few arguments), and we may introduce soft or hard constraints having to do with a notion of type that is commonly used in computer programming and in natural language. Each random variable in the graph can have as one of its attributes a type, and each factor energy function argument can expect a type. Energy function terms can be added to construct this prior by favouring graph pieces $(F_i,v^i=(v_k)_k)$ in which the type of variable $V_k$ (of which $v_k$ is a realization) matches the type expected of the $k$-th argument of $F_i$. This is very similar to attention mechanisms~\citep{bahdanau2014neural,vaswani2017attention}, which can be seen to match a query (an expected type) with a key (the type associated with an element). For instance, \cite{pan2023better} propose the Forward-Looking GFlowNet framework that exploits the modularity of the energy function, as in \cref{eq:total_energy_decomposition} and obtain a better approximation of the target distribution than those resulting from regular losses.

\section{Continuous or Hybrid Actions and States}
\label{sec:continuous}
All of the mathematical developments above have used sums over states or actions, with the idea that these would be elements of a discrete space. However, for the most part one can replace these sums by integrals in case the states or actions are either continuous or hybrid (with some discrete components and some continuous components). Beyond this, we discuss below what the presence of continuous-valued actions and states changes to the GFlowNet framework.

Although there are explicit sums respectively over successors and predecessors which come up in \cref{eq:flow-matching-log-loss}, such sums are also hiding in the detailed balance constraint of \cref{eq:detailed-balance}. Indeed, these sums are implicit as part of the normalizing constant in the conditional density of the next state or previous state in $P_F(s_{t+1} \mid s_t)$ and $P_B(s_t \mid s_{t+1})$. We consider below ideas to deal with this challenge.

\subsection{Integrable Normalization Constants}

We first note that if we can handle a continuous state, we can also handle a hybrid state, as follows. Let the state be decomposed into
\begin{equation}
    s = (s^i, s^x)
\end{equation}
where $s^i$ is discrete and $s^x$ is continuous. Then we can decompose any of the transition conditionals as follows:
\begin{equation}
    P_F(s_{t+1} \mid s_t) = P(s^x_{t+1} \mid s^i_{t+1},s_t) P(s^i_{t+1} \mid s_t).
\end{equation}
We note that this is formally equivalent to decomposing the transition into two transitions, first to perform the discrete choice into the next state, and second to perform the continuous choice into the next state (given the discrete choice). Having continuous-valued inputs to a neural net is no problem. The challenge is to represent continuous densities on the output, with the need to both being able to compute the density of a particular value (say $P(s^x_{t+1} \mid s^i_{t+1},s_t)$) and to be able to sample from it. Computing categorical probabilities and sampling from a conditional categorical is standard fare, so we only discuss the continuous conditional. One possibility is to parametrize $s^x_{t+1} \mid s^i_{t+1},s_t$ with a density for which the normalization constant is a known tractable integral, like the Gaussian. However, that may limit capacity too much, and may prevent a good minimization of the detailed balance or flow-matching loss. One workaround is to augment the discrete part of the state $s_i$ with extra dimensions corresponding to ``cluster IDs", i.e., partition the continuous density into a mixture. We know that with enough mixture components, we can arbitrarily well approximate densities from a very large family. Other approaches include modeling the conditional density with an autogressive or normalizing flow model~\citep[with a different meaning of the word flow]{rezende2015variational}, or, like in denoising diffusion models~\citep{sohl2015deep}, decomposing sampling of $s^x$ into several resampling steps, transforming its ditribution from a simple one to complex one.

To guarantee that the detailed balance constraint can be exactly satisfied, we could go further and think about parametrizing the edge flow $F((s_i,s_x){\rightarrow}(s'_i,s'_x))$, and note that this is the natural parametrization if we use the node-based flow-matching loss. For example, keeping with the Gaussian example, we would now have a joint Gaussian energy in the vector $(s_x,s'_x)$ for each feasible discrete component indexed by $(s_i,s'_i)$. Note that in practical applications of GFlowNets, not all the transitions that satisfy the order relationship are generally allowed in the GFlowNet's underlying DAG. For example, with set GFlowNets, the only allowed actions add one element to the set (not an arbitrary number of elements). These constraints on the action space mean that the number of legal $(s_i,s'_i)$ pairs is manageable and correspond to the number of discrete actions. The overall action is therefore seen as having a discrete part (choosing $s'_i$ given $s_i$) and a continuous part (choosing $s'_x$ given $s'_i$, $s_i$ and $s_x$). With such a joint flow formulation, the forward and backward conditional densities can be computed exactly and be compatible with each other.

\subsection{GFlowNets in GFlowNets}

Another way to implement an edge flow involving continuous variables is to use a lower-level $\langle$GFlowNet, energy function$\rangle$ pair to represent its flow, conditional probabilities and sample from them. Remember that such a pair can be trained following the approach discussed in~\cref{sec:energy-based-gfn}. Instead of a joint Gaussian for $(s_x,s'_x)$ given $(s_i,s'_i)$ we could have a smaller-scale GFlowNet and energy function (representing an edge flow in the outer GFlowNet) to handle a whole family of transitions of a particular type in a larger-scale outer GFlowNet. Imagine that we have a fairly arbitrary energy function for such a transition, with parameters that we will learn. Then we can also train a GFlowNet to sample in either direction (either from $s_t$ to $s_{t+1}$ or from $s_{t+1}$ to $s_t$) and to evaluate the corresponding normalizing constants (and hence, the corresponding conditional probabilities). The discrete aspect of the state and of its transition may correspond to a family of transitions (e.g., insert a particular type of node in a graph GFlowNet), and a separate $\langle$GFlowNet, energy function$\rangle$ module may be specialized and trained to handle such transitions.

While this paper was under review, \citet{lahlou2023continuous} extended the theory of GFlowNets to more general state spaces, including continuous ones, and experimentally validated that the usual losses can be used to effectively perform inference in continuous domains.

\section{Related Work}
\label{sec:related-work}
There are several classes of related literature that concern the problem of generating a diversity of samples, given some energy or reward signal, in particular:
\begin{itemize}[itemsep=-3pt,topsep=5pt]
    \item generative models (in particular deep learning ones),
    \item RL methods that maximize reward with some form of exploratory behavior or smoothness prior,
    \item MCMC methods that solve the problem of sampling from $p(x) \propto f(x)$ \textit{in principle},
    \item evolutionary methods, that can leverage group diversity through iterations over a population of solutions.
\end{itemize}

In what follows we discuss these and offer insights into similarities and differences between GFlowNets and these approaches. Note that the literature related to this problem is much larger than we can reference here, and extends to many other subfields of ML, such as GANs~\citep{kumar2019maximum}, VAEs~\citep{kingma2013auto,kusner2017grammar}, and normalizing flows~\citep{dinh2014nice,dinh2016density,rezende2015variational}. Yet another related
type of approach are the Bayesian optimization methods~\citep{movckus1975bayesian,srinivas2009gaussian}, which have also been used for searching in the space of molecules~\citep{griffiths2017constrained}. The main
relation with Bayesian optimization methods is that 
GFlowNets are generative and can thus complement Bayesian
optimization methods which scan a tractable list of candidates.
When the search space is too large to be able to separately 
compute a Bayesian optimization acquisition function score
on every candidate, using a generative model is appealing.
In addition, GFlowNets are used to explore the modes of the distribution
rather than to search for the single most dominant mode.
This difference is similar to that with classical RL
methods, discussed further below.

\subsection{Contrast with Generative Models}

The main difference between GFlowNets and established deep generative
models like VAEs or GANs is that whereas the latter are trained by being provided a finite set
of examples sampled from the distribution of interest, 
a GFlowNet is normally trained by being provided an energy function or a reward function. 

This reward function tells us not just about the samples that are likely under the distribution of interest (which we can think of as positive examples) but also about those that are unlikely (which we can think of as negative examples) and also about those in-between (whose reward is not large but is not zero either). If we think of the maximum likelihood training objective in those terms, it is like a reward function that gives a high reward to every training example (seen as a positive example, where the probability should be high) and a zero reward everywhere else. However, other reward functions are possible, as seen in the application of GFlowNets to the discovery of new molecules~\citep{bengio2021flow}, where the reward is not binary and increases monotonically as a function of the value of a desirable property of the candidate molecule.

Note however that the difference with other generative modeling approaches blurs when we include the learning of the energy function along with the learning of the GFlowNet sampler, as outlined in~\cref{sec:energy-based-gfn}. In that case,
the pair comprising the trainable GFlowNet sampler and the trainable energy function achieves a similar objective as a trainable generative model. Note that GFlowNets have been designed for generating discrete variable-size compositional structures (like sets or graphs), for both latent and observed variables, whereas GANs, VAEs or normalizing flows start from the point of view of modeling real-valued fixed-size vectors using real-valued fixed-size latent variables. 

An interesting difference between GFlowNets and most generative model training frameworks (typically some variation on maximum likelihood) is in the very nature of the training objective for GFlowNets, which came about in the context of active learning scenarios. Whereas the GFlowNet training pairs $(s,R(s))$ can come from any distribution over $s$ (any full-support training policy $\pi_T$), which does not have to be stationary (and indeed will generally not be, in an active learning setting), the maximum likelihood framework is very sensitive to changes in the distribution of the data it sees. This is connected to the ``offline learning" property of the flow matching objective (\cref{sec:offline}, among others). 

\subsection{Contrast with Regularized Reinforcement Learning}

The flow-matching loss of GFlowNets~\citep{bengio2021flow} arose from the inspiration of the temporal-difference training~\citep{sutton2018reinforcement} objectives associated with the Bellman equation. The flow-matching equations are analogous to the Bellman equation in the sense that the training objective is local (in time and states), credit assignment propagates through a bootstrap process and tries to fix the parametrization so that these equations are satisfied, knowing that if they were (everywhere), we would obtain the desired properties. However, these desired properties are different, as elaborated in the next paragraph. The context in which GFlowNets were developed is also different from the typical way of thinking about agents learning in some environment: we can think of the deterministic environments of GFlowNets as involving internal actions typically needed by a cognitive agent that needs to perform some kind of inference through a sequence of steps (predict or sample some things given other things), i.e., through actions internal to the agent and controlling its computation. This is in contrast with the origins of RL, focused on the actions of an agent in an external and unknown stochastic environment. GFlowNets were introduced as a tool for learning an internal policy, similar to the use of attention in modern deep learning, where we know the effect of actions, and the composition of these actions defines an inference machinery for that agent.

Classical RL ~\citep{sutton2018reinforcement} control methods work by maximizing return in Markov Decision Processes (MDPs); their focus is on finding the policy $\pi^* \in \mbox{argmax}_\pi V^\pi(s) \; \forall s$ maximizing the expected return $V^\pi(s)$, which happens to provably be achieved with a \textit{deterministic} policy~\citep{sutton2018reinforcement}, even in stochastic MDPs. In a deterministic MDP, of interest here, this means that training an RL agent is a search for the most rewarding trajectory, or in the case of terminal-reward-only MDPs (again of interest here), the most rewarding terminating state.

Another perspective, that emerged out of both the probabilistic inference literature~\citep{toussaint2006probabilistic} and the bandits literature~\citep{auer2002nonstochastic}, is concerned with finding policies of the form $\pi(a|s) \propto f(s,a)$. It turns out that maximizing both return and \textit{entropy}~\citep{ziebart2008maximum} of policies in a control setting yield policies such that
\begin{equation}
    p(\tau) = \left[p(s_0) \prod_{t=0}^{T-1} P(s_{t+1}|s_t,a_t)\right] \exp \left( \eta \sum_{t=0}^{T-1} R(s_t,a_t)\right)
\end{equation}
where $\tau=(s_0,a_0,s_1,a_1,\ldots,s_T)$ and $\eta$ can be seen as a temperature parameter. 
This result can also be found under the control-as-inference framework~\citep{haarnoja2017reinforcement, levine2018reinforcement}. In deterministic MDPs with terminal rewards and no discounting of future rewards, this simplifies to $p(\tau)\propto \exp(\eta \rho(\tau))$, where $\rho$ is the return.

In recent literature, this entropy maximization (MaxEnt) is often interpreted as a regularization scheme~\citep{nachum2017bridging}, entropy being used either as an intrinsic reward signal or as an explicit regularization objective to be maximized. Another way to understand this scheme is to imagine ourselves in an adversarial bandit setting~\citep{auer2002nonstochastic} where each arm corresponds to a unique trajectory, drawn with probability $\propto \exp(\rho(\tau))$.

An important distinction to make between MaxEnt RL and GFlowNets is that, in the general case they do not find the same result. A GFlowNet learns a policy such that $P_T(s) \propto R(s)$, whereas MaxEnt RL (with appropriately chosen temperature and $R$) learns a policy such that $P_T(s) \propto n(s) R(s)$, where $n(s)$ is the number of paths in the DAG of all trajectories that lead to $s$~\citep[a proof is provided in][]{bengio2021flow}. An equivalence only exists if the DAG minus $s_f$ is a tree rooted at $s_0$, which has been found to be useful~\citep{buesing2019approximate}. What this overweighting by
a factor $n(s)$ means practically is that states corresponding to longer sequences (which typically will have exponentially more paths to them) will tend to be sampled much more often (typically exponentially more often) than states corresponding to shorter sequences. Clearly, this breaks the objective of sampling terminating states in proportion to their reward and provides a strong motivation for considering GFlowNets instead.

Another perspective on maximizing entropy in RL is that one can also maximize entropy on the states' \textit{stationary distribution} $d^\pi$~\citep{ziebart2008maximum}, rather than the policy. In fact, one can show that the objective of training a policy such that $P_T(s)\propto R(s)$ is equivalent to training a policy that maximizes $r(s,a) = \log R(s,a) - \log d^\pi(s,a)$. Unfortunately, computing stationary distributions, although possible~\citep{nachum2019dualdice,wen2020batch}, is not always tractable nor precise enough for purposes of reward regularization. 


\subsection{Contrast with Monte-Carlo Markov Chain methods}

MCMC has a long and rich history~\citep{metropolis1953equation,hastings1970monte,andrieu2003introduction}, and is particularly relevant to the present work, since it is also a principled class of methods towards sampling from $P_T(s) \propto R(s)$.
MCMC-based methods have already found some amount of success with learned deep neural networks used to drive sampling~\citep{grathwohl2021oops,dai2020learning,xie2021mars,nash2019autoregressive,seff2019discrete}.

An important drawback of MCMC is its reliance on iterative sampling (forming the Markov chain, one configuration at a time, each of which is like a terminating state of a GFlowNet): a new state configuration is obtained at each step of the chain by making a small stochastic change to the configuration in the previous step. Although these methods guarantee that asymptotically (in the length of the chain) we obtain samples drawn from the correct distribution, there is an important set of distributions for which finite chains are unlikely to provide enough diversity of the modes of the distribution.

This is known as the mode-mixing problem~\citep{jasra2005markov,bengio2013better,pompe2020framework}: the chances of going from one mode to a neighboring one may become exponentially small (and thus require exponentially long chains) if the modes are separated by a long sequence of low-probability configurations. This can be alleviated by burning more computation (sampling longer chains) but becomes exponentially unsustainable with increased mode separation. The issue can also be reduced by introducing random sampling (e.g., drawing multiple chains) and simulated annealing~\citep{andrieu2003introduction} to facilitate jumping between modes. However,
this becomes less effective in high dimensions and when the modes occupy a tiny volume (which can become an exponentially small fraction of the total space as its dimension increases) since random sampling is unlikely to land in the neighborhood of a mode. 

In contrast, GFlowNets belong to the family of {\bf amortized} sampling methods \citep[which includes VAEs,][]{kingma2013auto}, where we train a machine learning system to produce samples: we have exchanged the complexity of sampling through long chains for the complexity of training the sampler. The potential advantage of such amortized samplers is when the distribution of interest has generalizable structure: when it is possible to guess reasonably well where high-probability samples can be found, based on the knowledge of a set of known high-probability samples (the training set). This is what makes deep generative models work in the first place and thus suggests that in such high-dimensional settings where modes occupy tiny volumes ~\citep[as per the manifold hypothesis,][]{cayton2005algorithms,narayanan2010sample,rifai2011manifold}, one can capitalize on the already observed $(x,R(x))$ pairs (where $x$ is an already visited configuration and $R(x)$ its reward) to ``jump" from known modes to yet unvisited ones, even if these are far from the ones already visited. 

How well this will work then depends on the ability to generalize of the learner, i.e., on the strength and appropriateness of its inductive biases, as usual in machine learning. 
In the case where there is no structure at all (and thus no possibility to generalize when learning about the distribution), there is no reason to expect that amortized ML methods will fare better than MCMC. But if there is structure, then the exponential cost of mixing between modes could go away. There is plenty of evidence that ML methods can do a good job in such high-dimensional spaces (like the space of natural images) and this suggests that GFlowNets and other amortized sampling methods would be worth considering where ML generally works well. Molecular graph generation experiments~\citep{bengio2021flow} comparing GFlowNets and MCMC methods appear to confirm this. 

Another factor to consider (independent of the mode mixing issue) is the amortization of the computational costs: GFlowNets pay a large price upfront to train the network and then a small price (sampling once from $P_T$) to generate each new sample. Instead, MCMC has no upfront cost but pays a lot for each independent sample. Hence, if we want to only sample once, MCMC may be more efficient, whereas if we want to generate a lot of samples, amortized methods may be advantageous. One can imagine settings where GFlowNets and MCMC could be combined to achieve some of the advantages of both approaches.

\paragraph{Evolutionary Methods}

Evolutionary methods work similarly to MCMC methods, via an iterative process of stochastic local search, and populations of candidates are found that maximize one or many objectives ~\citep{brown2004graph,salimans2017evolution,jensen2019graph,swersky2020amortized}. From such a perspective, they have similar advantages and disadvantages. One practical advantage of these methods is that natural diversity is easily obtainable via group metrics and subpopulation selection~\citep{mouret2012encouraging}. This is not something that is explicitly tackled by GFlowNet, which instead relies on i.i.d. sampling and giving non-zero probability to suboptimal samples as a diversity mechanism.

\paragraph{Sequential Monte-Carlo}

Sequential Monte Carlo~\citep[SMC,][]{naesseth2019elements, aru2002smc} methods are a class of methods aimed at solving inference problems. Similar to GFlowNets, SMC samplers are trained to sample from a distribution given by its unnormalized probability density, and require forward and backward kernels, as in GFlowNets. Unlike GFlowNets though, they require specifying intermediate targets $\gamma_t(z_t)$. In addition,
with GFlowNets, the reward normally only comes at the end of the sequence, unlike the per-time-step likelihoods used to reweight particles in SMC. GFlowNets can be applied in settings which do not fit the typical particle filter setting, such as those whose intermediate states do not correspond to valid elements of the sample space. The trajectories do not necessarily represent a sequence of latent variables associated with corresponding observations, as in filtering tasks. The GFlowNet trajectory distributions are only defined by the terminal state reward function.

\section{Conclusions and Open Questions}

This paper extends and deepens the mathematical framework and mathematical properties of GFlowNets~\citep{bengio2021flow}. It connects the notion of flow in GFlowNets with that of measure over trajectories and introduces a novel training objective (the detailed balance loss) which makes it possible to choose a parametrization separating the backward policy $P_B$ which controls preferences over the order in which things are done from the constraints imposed by the target reward function. 
 
An important contribution of this paper is the mathematical framework for marginalization or free energy estimation using GFlowNets. It relies on the simple idea of conditioning the GFlowNet so as to push the ability to estimate a partition function already introduced by~\citet{bengio2021flow} to a much more general setting. This makes it possible in principle to estimate intractable sums of rewards over the terminating states reachable by an arbitrary state, opening the door to marginalization over supergraphs of graphs, supersets of sets, and supersets of (variable,value) pairs. In turn, this provides formulae for estimating entropies, conditional entropies and mutual information. 

\ifarxiv
\cref{sec:direct-credit-assignment} introduces alternatives to the flow-matching objective which may bypass the slow ``bootstrapping" propagation of credit information from the end of the action sequence to its beginnings, with so-called direct credit assignment (which has some similarity to policy gradient). This question also motivated the trajectory balance objective introduced by~\citet{malkin2022trajectory} and further work in this direction is warranted.

In an attempt to better discern the links between GFlowNets and the more common forms of RL,  \cref{sec:stochastic-environment} considers the cases where rewards are provided not just at the end but possibly after every action, and where the environment may be stochastic. It also shows how a greedy policy that maximises returns could be obtained from a trained GFlowNet.
Another link with RL is introduced in \cref{sec:expected-downstream-reward} by considering something similar to a value function in GFlowNets, i.e., the expected downstream reward from a given state $s$. In a similar spirit,
\cref{sec:intermediate-rewards-and-returns} generalizes GFlowNets to the case where rewards can be earned along the trajectory, thus introducing a notion of return, but keeping the objective of sampling in proportion to the return.

\cref{sec:multi-flows} generalizes GFlowNets in another direction, considering the possibility of having a family of flows being learned simultaneously over the same graph. This opens the door to
a distributional generalization of GFlowNets
(by analogy with distributional RL) as well a
notion of unsupervised GFlowNet (where the reward
function can be defined after the GFlowNet
has been trained in an unsupervised fashion)
and learning to sample from the Pareto front
defined by a set of reward functions.
\fi 

Many open questions obviously remain, from the extension to continuous actions and states to hierarchical versions of GFlowNets with abstract actions and integrating the energy function in the GFlowNet parametrization itself, enabling an interesting form of modularization and knowledge decomposition. Importantly, many of the mathematical formulations presented in this paper will require empirical validation to ascertain their usefulness, improve these ideas, turn them into impactful algorithms and explore a potentially very broad range of interesting applications, from replacing MCMC or being combined with MCMC in some settings, to probabilistic reasoning to further applications in active learning for scientific discovery.

\section*{Acknowledgements}
The authors want to acknowledge the useful suggestions and feedback on the paper and its ideas provided by Alexandra Volokhova, Marc Bellemare, Valentin Thomas, Modjtaba Shokrian Zini, Mohammad Pedramfar, Arthur Maffre, and Axel Nguyen Kerbel. They are also grateful for the financial support from CIFAR, Samsung, IBM, Google, Microsoft, JP Morgan Chase, and the Thomas C. Nelson Stanford Interdisciplinary Graduate Fellowship.\\

This research was funded in part by JPMorgan Chase \& Co. Any views or opinions expressed herein are solely those of the authors listed, and may differ from the views and opinions expressed by JPMorgan Chase \& Co. or its affiliates. This material is not a product of the Research Department of J.P. Morgan Securities LLC. This material should not be construed as an individual recommendation for any particular client and is not intended as a recommendation of particular securities, financial instruments or strategies for a particular client. This material does not constitute a solicitation or offer in any jurisdiction.

\bibliography{main.bib}

@article{hu2023gfnem,
  title={GFlowNet-EM for Learning Compositional Latent Variable Models},
  author={Hu, Edward and Malkin, Nikolay and Jain, Moksh and Everett, Katie and Graikos, Alexandros and Bengio, Yoshua},
  journal={arvix},
  year={2023},
  note={}
}

@article{bengio2021flow,
	title        = {Flow Network based Generative Models for Non-Iterative Diverse Candidate Generation},
	author       = {Bengio, Emmanuel and Jain, Moksh and Korablyov, Maksym and Precup, Doina and Bengio, Yoshua},
	year         = 2021,
	journal      = {NeurIPS'2021, arXiv:2106.04399}
}

@article{goyal2020inductive,
	title        = {Inductive Biases for Deep Learning of Higher-Level Cognition},
	author       = {Anirudh Goyal and Yoshua Bengio},
	year         = 2020,
	journal      = {arXiv},
	volume       = {abs/2011.15091},
	note         = {\url{https://arxiv.org/abs/2011.15091}}
}

@article{malkin2022gfnhvi,
  title={{GFlowNets} and variational inference},
  author={Malkin, Nikolay and Lahlou, Salem and Deleu, Tristan and Ji, Xu and Hu, Edward and Everett, Katie and Zhang, Dinghuai and Bengio, Yoshua},
  journal={International Conference on Learning Representations (ICLR)},
  year={2023},
}

@article{zimmermann2022variational,
  title={A variational perspective on generative flow networks},
  author={Zimmermann, Heiko and Lindsten, Fredrik and van de Meent, Jan-Willem and Naesseth, Christian A.},
  journal={arXiv preprint 2210.07992},
  year={2022}
}

@inproceedings{deleu2022bayesian,
  title={Bayesian structure learning with generative flow networks},
  author={Deleu, Tristan and G{\'o}is, Ant{\'o}nio and Emezue, Chris and Rankawat, Mansi and Lacoste-Julien, Simon and Bauer, Stefan and Bengio, Yoshua},
  booktitle={Uncertainty in Artificial Intelligence},
  pages={518--528},
  year={2022},
  organization={PMLR}
}

@article{naesseth2019elements,
  title     = {Elements of sequential monte carlo},
  author    = {Naesseth, Christian A and Lindsten, Fredrik and Sch{\"o}n, Thomas B and others},
  journal   = {Foundations and Trends{\textregistered} in Machine Learning},
  volume    = {12},
  number    = {3},
  pages     = {307-392},
  year      = {2019},
  publisher = {Now Publishers, Inc.}
}

@ARTICLE{aru2002smc,
  author={Arulampalam, M.S. and Maskell, S. and Gordon, N. and Clapp, T.},
  journal={IEEE Transactions on Signal Processing}, 
  title={A tutorial on particle filters for online nonlinear/non-Gaussian Bayesian tracking}, 
  year={2002},
  volume={50},
  number={2},
  pages={174-188},
  doi={10.1109/78.978374}}

@inproceedings{SchJanPetSgoetal12,
	title        = {On Causal and Anticausal Learning},
	author       = {Sch{\"o}lkopf, Bernhard and Janzing, Dominik and Peters, Jonas and Sgouritsa, Eleni and Zhang, Kun and Mooij, Joris},
	year         = 2012,
	booktitle    = {ICML'2012},
	pages        = {1255--1262}
}

@article{goyal2019recurrent,
	title        = {Recurrent independent mechanisms},
	author       = {Goyal, Anirudh and Lamb, Alex and Hoffmann, Jordan and Sodhani, Shagun and Levine, Sergey and Bengio, Yoshua and Sch{\"o}lkopf, Bernhard},
	year         = 2019,
	journal      = {ICLR'2021, arXiv:1909.10893}
}

@inproceedings{dosovitskiy2019you,
	title        = {You only train once: Loss-conditional training of deep networks},
	author       = {Dosovitskiy, Alexey and Djolonga, Josip},
	year         = 2019,
	booktitle    = {International Conference on Learning Representations}
}

@inproceedings{bellemare2017distributional,
	title        = {A Distributional Perspective on Reinforcement Learning},
	author       = {Bellemare, Marc G and Dabney, Will and Munos, R\'{e}mi},
	year         = 2017,
	booktitle    = {International Conference on Machine Learning}
}

@article{dinh2016density,
	title        = {Density estimation using real {NVP}},
	author       = {Dinh, Laurent and Sohl-Dickstein, Jascha and Bengio, Samy},
	year         = 2016,
	journal      = {ICLR'2017, arXiv:1605.08803}
}

@article{goodfellow2014generative,
	title        = {Generative adversarial nets},
	author       = {Goodfellow, Ian and Pouget-Abadie, Jean and Mirza, Mehdi and Xu, Bing and Warde-Farley, David and Ozair, Sherjil and Courville, Aaron and Bengio, Yoshua},
	year         = 2014,
	journal      = {Advances in neural information processing systems},
	volume       = 27
}

@article{kschischang2001factor,
	title        = {Factor graphs and the sum-product algorithm},
	author       = {Kschischang, Frank R and Frey, Brendan J and Loeliger, H-A},
	year         = 2001,
	journal      = {IEEE Transactions on information theory},
	publisher    = {IEEE},
	volume       = 47,
	number       = 2,
	pages        = {498--519}
}

@inproceedings{rezende2015variational,
	title        = {Variational inference with normalizing flows},
	author       = {Rezende, Danilo and Mohamed, Shakir},
	year         = 2015,
	booktitle    = {International conference on machine learning},
	pages        = {1530--1538},
	organization = {PMLR}
}

@inproceedings{movckus1975bayesian,
	title        = {On Bayesian methods for seeking the extremum},
	author       = {Mo{\v{c}}kus, Jonas},
	year         = 1975,
	booktitle    = {Optimization techniques IFIP technical conference},
	pages        = {400--404},
	organization = {Springer}
}

@article{andrieu2003introduction,
	title        = {An introduction to MCMC for machine learning},
	author       = {Andrieu, Christophe and De Freitas, Nando and Doucet, Arnaud and Jordan, Michael I},
	year         = 2003,
	journal      = {Machine learning},
	publisher    = {Springer},
	volume       = 50,
	number       = 1,
	pages        = {5--43}
}

@article{cayton2005algorithms,
	title        = {Algorithms for manifold learning},
	author       = {Cayton, Lawrence},
	year         = 2005,
	journal      = {Univ. of California at San Diego Tech. Rep},
	volume       = 12,
	number       = {1-17},
	pages        = 1
}

@inproceedings{bengio2013better,
	title        = {Better mixing via deep representations},
	author       = {Bengio, Yoshua and Mesnil, Gr{\'e}goire and Dauphin, Yann and Rifai, Salah},
	year         = 2013,
	booktitle    = {International conference on machine learning},
	pages        = {552--560},
	organization = {PMLR}
}

@article{jasra2005markov,
	title        = {Markov chain Monte Carlo methods and the label switching problem in Bayesian mixture modeling},
	author       = {Jasra, Ajay and Holmes, Chris C and Stephens, David A},
	year         = 2005,
	journal      = {Statistical Science},
	publisher    = {JSTOR},
	pages        = {50--67}
}

@article{pompe2020framework,
	title        = {A framework for adaptive MCMC targeting multimodal distributions},
	author       = {Pompe, Emilia and Holmes, Chris and {\L}atuszy{\'n}ski, Krzysztof},
	year         = 2020,
	journal      = {The Annals of Statistics},
	publisher    = {Institute of Mathematical Statistics},
	volume       = 48,
	number       = 5,
	pages        = {2930--2952}
}

@article{ghosh2018learning,
	title        = {Learning actionable representations with goal-conditioned policies},
	author       = {Ghosh, Dibya and Gupta, Abhishek and Levine, Sergey},
	year         = 2018,
	journal      = {arXiv preprint arXiv:1811.07819}
}

@article{schmidhuber2019reinforcement,
	title        = {Reinforcement Learning Upside Down: Don't Predict Rewards--Just Map Them to Actions},
	author       = {Schmidhuber, Juergen},
	year         = 2019,
	journal      = {arXiv preprint arXiv:1912.02875}
}

@inproceedings{narayanan2010sample,
	title        = {Sample complexity of testing the manifold hypothesis},
	author       = {Narayanan, Hariharan and Mitter, Sanjoy},
	year         = 2010,
	booktitle    = {NIPS'2010},
	pages        = {1786--1794}
}

@article{rifai2011manifold,
	title        = {The manifold tangent classifier},
	author       = {Rifai, Salah and Dauphin, Yann N and Vincent, Pascal and Bengio, Yoshua and Muller, Xavier},
	year         = 2011,
	journal      = {Advances in neural information processing systems},
	volume       = 24,
	pages        = {2294--2302}
}

@inproceedings{srinivas2009gaussian,
	title        = {Gaussian process optimization in the bandit setting: No regret and experimental design},
	author       = {Srinivas, Niranjan and Krause, Andreas and Kakade, Sham M and Seeger, Matthias},
	year         = 2010,
	booktitle    = {International Conference on Machine Learning (ICML)}
}

@inproceedings{sohl2015deep,
	title        = {Deep unsupervised learning using nonequilibrium thermodynamics},
	author       = {Sohl-Dickstein, Jascha and Weiss, Eric and Maheswaranathan, Niru and Ganguli, Surya},
	year         = 2015,
	booktitle    = {International Conference on Machine Learning},
	pages        = {2256--2265},
	organization = {PMLR}
}

@article{dinh2014nice,
	title        = {Nice: Non-linear independent components estimation},
	author       = {Dinh, Laurent and Krueger, David and Bengio, Yoshua},
	year         = 2014,
	journal      = {ICLR'2015 Workshop, arXiv:1410.8516}
}

@article{bahdanau2014neural,
	title        = {Neural machine translation by jointly learning to align and translate},
	author       = {Bahdanau, Dzmitry and Cho, Kyunghyun and Bengio, Yoshua},
	year         = 2014,
	journal      = {ICLR'2015, arXiv:1409.0473}
}

@inproceedings{vaswani2017attention,
	title        = {Attention is all you need},
	author       = {Vaswani, Ashish and Shazeer, Noam and Parmar, Niki and Uszkoreit, Jakob and Jones, Llion and Gomez, Aidan N and Kaiser, {\L}ukasz and Polosukhin, Illia},
	year         = 2017,
	booktitle    = {Advances in neural information processing systems},
	pages        = {5998--6008}
}

@book{sutton2018reinforcement,
	title        = {Reinforcement learning: An introduction},
	author       = {Sutton, Richard S and Barto, Andrew G},
	year         = 2018,
	publisher    = {MIT press}
}

@inproceedings{toussaint2006probabilistic,
	title        = {Probabilistic inference for solving discrete and continuous state Markov Decision Processes},
	author       = {Toussaint, Marc and Storkey, Amos},
	year         = 2006,
	booktitle    = {Proceedings of the 23rd international conference on Machine learning},
	pages        = {945--952}
}

@article{auer2002nonstochastic,
	title        = {The nonstochastic multiarmed bandit problem},
	author       = {Auer, Peter and Cesa-Bianchi, Nicolo and Freund, Yoav and Schapire, Robert E},
	year         = 2002,
	journal      = {SIAM journal on computing},
	publisher    = {SIAM},
	volume       = 32,
	number       = 1,
	pages        = {48--77}
}

@inproceedings{ziebart2008maximum,
	title        = {Maximum entropy inverse reinforcement learning.},
	author       = {Ziebart, Brian D and Maas, Andrew L and Bagnell, J Andrew and Dey, Anind K and others},
	year         = 2008,
	booktitle    = {Aaai},
	volume       = 8,
	pages        = {1433--1438},
	organization = {Chicago, IL, USA}
}

@article{levine2018reinforcement,
	title        = {Reinforcement learning and control as probabilistic inference: Tutorial and review},
	author       = {Levine, Sergey},
	year         = 2018,
	journal      = {arXiv preprint arXiv:1805.00909}
}

@article{nachum2017bridging,
	title        = {Bridging the gap between value and policy based reinforcement learning},
	author       = {Nachum, Ofir and Norouzi, Mohammad and Xu, Kelvin and Schuurmans, Dale},
	year         = 2017,
	journal      = {arXiv preprint arXiv:1702.08892}
}

@inproceedings{haarnoja2017reinforcement,
	title        = {Reinforcement learning with deep energy-based policies},
	author       = {Haarnoja, Tuomas and Tang, Haoran and Abbeel, Pieter and Levine, Sergey},
	year         = 2017,
	booktitle    = {International Conference on Machine Learning},
	pages        = {1352--1361},
	organization = {PMLR}
}

@article{hastings1970monte,
	title        = {Monte Carlo sampling methods using Markov chains and their applications},
	author       = {Hastings, W Keith},
	year         = 1970,
	journal      = {Biometrika},
	publisher    = {Oxford University Press}
}

@article{metropolis1953equation,
	title        = {Equation of state calculations by fast computing machines},
	author       = {Metropolis, Nicholas and Rosenbluth, Arianna W and Rosenbluth, Marshall N and Teller, Augusta H and Teller, Edward},
	year         = 1953,
	journal      = {The journal of chemical physics},
	publisher    = {American Institute of Physics},
	volume       = 21,
	number       = 6,
	pages        = {1087--1092}
}

@inproceedings{dai2020learning,
	title        = {Learning Discrete Energy-based Models via Auxiliary-variable Local Exploration},
	author       = {Hanjun Dai and Rishabh Singh and Bo Dai and Charles Sutton and Dale Schuurmans},
	year         = 2020,
	booktitle    = {Neural Information Processing Systems (NeurIPS)}
}

@inproceedings{xie2021mars,
	title        = {{\{}MARS{\}}: Markov Molecular Sampling for Multi-objective Drug Discovery},
	author       = {Yutong Xie and Chence Shi and Hao Zhou and Yuwei Yang and Weinan Zhang and Yong Yu and Lei Li},
	year         = 2021,
	booktitle    = {International Conference on Learning Representations},
	url          = {https://openreview.net/forum?id=kHSu4ebxFXY}
}

@inproceedings{nash2019autoregressive,
	title        = {Autoregressive energy machines},
	author       = {Nash, Charlie and Durkan, Conor},
	year         = 2019,
	booktitle    = {International Conference on Machine Learning},
	pages        = {1735--1744},
	organization = {PMLR}
}

@article{seff2019discrete,
	title        = {Discrete object generation with reversible inductive construction},
	author       = {Seff, Ari and Zhou, Wenda and Damani, Farhan and Doyle, Abigail and Adams, Ryan P},
	year         = 2019,
	journal      = {arXiv preprint arXiv:1907.08268}
}

@article{brown2004graph,
	title        = {A graph-based genetic algorithm and its application to the multiobjective evolution of median molecules},
	author       = {Brown, Nathan and McKay, Ben and Gilardoni, Fran{\c{c}}ois and Gasteiger, Johann},
	year         = 2004,
	journal      = {Journal of chemical information and computer sciences},
	publisher    = {ACS Publications},
	volume       = 44,
	number       = 3,
	pages        = {1079--1087}
}

@article{jensen2019graph,
	title        = {A graph-based genetic algorithm and generative model/Monte Carlo tree search for the exploration of chemical space},
	author       = {Jensen, Jan H},
	year         = 2019,
	journal      = {Chemical science},
	publisher    = {Royal Society of Chemistry},
	volume       = 10,
	number       = 12,
	pages        = {3567--3572}
}

@inproceedings{swersky2020amortized,
	title        = {Amortized bayesian optimization over discrete spaces},
	author       = {Swersky, Kevin and Rubanova, Yulia and Dohan, David and Murphy, Kevin},
	year         = 2020,
	booktitle    = {Conference on Uncertainty in Artificial Intelligence},
	pages        = {769--778},
	organization = {PMLR}
}

@misc{salimans2017evolution,
	title        = {Evolution Strategies as a Scalable Alternative to Reinforcement Learning},
	author       = {Tim Salimans and Jonathan Ho and Xi Chen and Szymon Sidor and Ilya Sutskever},
	year         = 2017,
	eprint       = {1703.03864},
	archiveprefix = {arXiv},
	primaryclass = {stat.ML}
}

@article{mouret2012encouraging,
	title        = {{Encouraging Behavioral Diversity in Evolutionary Robotics: An Empirical Study}},
	author       = {Mouret, J.-B. and Doncieux, S.},
	year         = 2012,
	month        = {03},
	journal      = {Evolutionary Computation},
	volume       = 20,
	number       = 1,
	pages        = {91--133},
	doi          = {10.1162/EVCO\_a\_00048},
	issn         = {1063-6560},
	url          = {https://doi.org/10.1162/EVCO\_a\_00048},
	eprint       = {https://direct.mit.edu/evco/article-pdf/20/1/91/1494197/evco\_a\_00048.pdf}
}

@article{wen2020batch,
	title        = {Batch stationary distribution estimation},
	author       = {Wen, Junfeng and Dai, Bo and Li, Lihong and Schuurmans, Dale},
	year         = 2020,
	journal      = {arXiv preprint arXiv:2003.00722}
}

@misc{kumar2019maximum,
	title        = {Maximum Entropy Generators for Energy-Based Models},
	author       = {Rithesh Kumar and Sherjil Ozair and Anirudh Goyal and Aaron Courville and Yoshua Bengio},
	year         = 2019,
	eprint       = {1901.08508},
	archiveprefix = {arXiv},
	primaryclass = {cs.LG}
}

@article{nachum2019dualdice,
	title        = {Dualdice: Behavior-agnostic estimation of discounted stationary distribution corrections},
	author       = {Nachum, Ofir and Chow, Yinlam and Dai, Bo and Li, Lihong},
	year         = 2019,
	journal      = {arXiv preprint arXiv:1906.04733}
}

@article{kingma2013auto,
	title        = {Auto-encoding variational bayes},
	author       = {Kingma, Diederik P and Welling, Max},
	year         = 2013,
	journal      = {arXiv preprint arXiv:1312.6114}
}

@inproceedings{kusner2017grammar,
	title        = {Grammar variational autoencoder},
	author       = {Kusner, Matt J and Paige, Brooks and Hern{\'a}ndez-Lobato, Jos{\'e} Miguel},
	year         = 2017,
	booktitle    = {International Conference on Machine Learning},
	pages        = {1945--1954},
	organization = {PMLR}
}

@article{griffiths2017constrained,
	title        = {Constrained bayesian optimization for automatic chemical design},
	author       = {Griffiths, Ryan-Rhys and Hern{\'a}ndez-Lobato, Jos{\'e} Miguel},
	year         = 2017,
	journal      = {arXiv preprint arXiv:1709.05501}
}

@article{bai2019deqs,
	title        = {Deep Equilibrium Models},
	author       = {Shaojie Bai and J. Zico Kolter and Vladlen Koltun},
	year         = 2019,
	journal      = {CoRR},
	volume       = {abs/1909.01377},
	url          = {http://arxiv.org/abs/1909.01377},
	eprinttype   = {arXiv},
	eprint       = {1909.01377},
	bibsource    = {dblp computer science bibliography, https://dblp.org}
}

@misc{buesing2019approximate,
	title        = {Approximate Inference in Discrete Distributions with Monte Carlo Tree Search and Value Functions},
	author       = {Lars Buesing and Nicolas Heess and Theophane Weber},
	year         = 2019,
	eprint       = {1910.06862},
	archiveprefix = {arXiv},
	primaryclass = {cs.LG}
}

@misc{grathwohl2021oops,
	title        = {Oops I Took A Gradient: Scalable Sampling for Discrete Distributions},
	author       = {Will Grathwohl and Kevin Swersky and Milad Hashemi and David Duvenaud and Chris J. Maddison},
	year         = 2021,
	eprint       = {2102.04509},
	archiveprefix = {arXiv},
	primaryclass = {cs.LG}
}

@article{ernst2005tree,
	title        = {Tree-based batch mode reinforcement learning},
	author       = {Ernst, Damien and Geurts, Pierre and Wehenkel, Louis},
	year         = 2005,
	journal      = {Journal of Machine Learning Research},
	publisher    = {Microtome Publishing},
	volume       = 6,
	pages        = {503--556}
}

@inproceedings{riedmiller2005neural,
	title        = {Neural fitted Q iteration--first experiences with a data efficient neural reinforcement learning method},
	author       = {Riedmiller, Martin},
	year         = 2005,
	booktitle    = {European conference on machine learning},
	pages        = {317--328},
	organization = {Springer}
}

@incollection{lange2012batch,
	title        = {Batch reinforcement learning},
	author       = {Lange, Sascha and Gabel, Thomas and Riedmiller, Martin},
	year         = 2012,
	booktitle    = {Reinforcement learning},
	publisher    = {Springer},
	pages        = {45--73}
}

@article{malkin2022trajectory,
	title        = {Trajectory Balance: Improved Credit Assignment in GFlowNets},
	author       = {Malkin, Nikolay and Jain, Moksh and Bengio, Emmanuel and Sun, Chen and Bengio, Yoshua},
	year         = 2022,
	journal      = {arXiv preprint arXiv:2201.13259}
}

@article{malik2023batchgfn,
  title   = {BatchGFN: Generative Flow Networks for Batch Active Learning},
  author  = {Shreshth A. Malik and Salem Lahlou and Andrew Jesson and Moksh Jain and Nikolay Malkin and Tristan Deleu and Yoshua Bengio and Yarin Gal},
  year    = {2023},
  journal = {arXiv preprint arXiv: 2306.15058}
}

@article{zhang2022scheduling,
	title        = {Robust Scheduling with GFlowNets},
	author       = {Zhang, David W and Rainone, Coraddo and Peschl, Markus and Bondesan, Roberto},
	year         = 2023,
	journal={International Conference on Learning Representations (ICLR)},
}

@article{pan2023better,
  title   = {Better Training of GFlowNets with Local Credit and Incomplete Trajectories},
  author  = {Ling Pan and Nikolay Malkin and Dinghuai Zhang and Yoshua Bengio},
  year    = {2023},
  journal = {arXiv preprint arXiv: 2302.01687}
}

@article{zhang2022discretemodeling,
	title        = {Generative Flow Networks for Discrete Probabilistic Modeling},
	author       = {Zhang, Dinghuai and Malkin, Nikolay and Liu, Zhen and Volokhova, Alexandra and Courville, Aaron and Bengio, Yoshua},
	year         = 2022,
	journal={International Conference on Machine Learning (ICML)},
}

@article{lahlou2023continuous,
	title        = {A theory of continuous generative flow networks
},
	author    = {Salem Lahlou and T. Deleu and Pablo Lemos and Dinghuai Zhang and Alexandra Volokhova and Alex Hernández-García and L'ena N'ehale Ezzine and Y. Bengio and Nikolay Malkin},
	year         = 2023,
	journal={International Conference on Machine Learning (ICML)},
}

@article{jain2022biologicalseqdesign,
	title        = {Biological Sequence Design with GFlowNets},
	author       = {Jain, Moksh and Bengio, Emmanuel and Hernandez-Garcia, Alex and Rector-Brooks, Jarrid and Dossou, Bonaventure F P and Ekbote, Chanakya and Fu, Jie and Zhang, Tianyu and Kilgour, Michael and Zhang, Dinghuai and Simine, Lena and Das, Payel and Bengio, Yoshua},
	year         = 2022,
	journal={International Conference on Machine Learning (ICML)},
}

@article{jain2023multiobjective,
	title        = {Multi-Objective GFlowNets},
	author       = {Jain, Moksh and Raparthy, Sharath Chandra and Hernandez-Garcia, Alex and Rector-Brooks, Jarrid and Bengio, Yoshua and Miret, Santiago and Bengio, Emmanuel},
	year         = 2023,
	journal={arXiv preprint arXiv:2210.12765},
}
\clearpage
\ifarxiv
\appendix
\section{Direct Credit Assignment in GFlowNets}
\label{sec:direct-credit-assignment}

Similarly to temporal-difference methods, which are based on minimizing the mismatch with respect to the Bellman equations, the flow-matching and detailed-balance losses will take many updates (and sampling many trajectories) to propagate a reward mismatch on a terminating state into the transition probabilities inside the flow network. This would be particularly acute for longer trajectories and prompts the question of alternative more direct training objectives. This section attempts to answer the following question:  ``Given a training trajectory $\tau$, are there more direct ways of assigning credit to the earlier transitions in the trajectory?''. \cite{malkin2022trajectory} provide an alternative answer by introducing the Trajectory Balance loss (\cref{ex:tb-loss}).

We can view the process of sampling a trajectory with a GFlowNet as analogous to sampling a sequence of states in a stochastic recurrent neural network. What makes things complicated is that such a neural network (i) does not directly output a prediction to be matched by some target, and (ii) that the state may be discrete (or a combination of discrete and continuous components). 

Regarding (i), we recall that the flow in a transition $s{\rightarrow}s'$ (involved in the training objective) is an intractable sum over all the possible trajectories leading to $s$. However, we may be able to obtain a stochastic gradient that, in average over such trajectories, estimates the desired quantity. For this, we exploit the properties of flows to obtain a stochastic gradient estimator, derived through the next three propositions.\\

\noindent{\bf Notation for derivatives through the flow-matching constraint}: below we sometimes need to distinguish total derivatives (noted $\frac{d y}{d x}$) that take into account the indirect effects due to the flow matching constraint from other derivatives (noted $\frac{\partial y}{\partial x}$) and capturing only direct gradients. Either notation can be used when the constraint does not change the result. Computing indirect gradients through implicit dependencies is an active area of research and commonly utilizes the Implicit Function Theorem, e.g., for implicit layers in fixed-point iteration layers and Deep Equilibrium Models \citep{bai2019deqs}.

\begin{proposition}
\label{prop:dlogF-dlogF}
Consider the effect of a slight change in the log of the flow at $s<s'$ under the flow-matching constraint: it yields a change in the log of the flow at $s'$, following the conditional probability $P(s | s')$:
\begin{equation}
\label{eq:dlogF-dlogF}
    \frac{d \log F(s')}{d \log F(s)} =P(s | s')
\end{equation}
where $P$ is the distribution on events over trajectories.
\end{proposition}
\begin{proof}
We are going to consider a partition of the complete trajectories going through $s'$ into those that also go through $s$ (set $\mathcal{T}_{s'}^s$) and those that don't (set $\mathcal{T}_{s'}^{-s}$). By definition we can write:
\begin{align*}
    P(s') = P(\mathcal{T}_{s'}^s) + P(\mathcal{T}_{s'}^{-s})
\end{align*}
Because $\mathcal{T}_{s'}^s = \{\tau \in \mathcal{T}: \ s \in \tau, \ s' \in \tau\}$, we can write $P(\mathcal{T}_{s'}^s) = P(s) P(s' | s)$. Hence:
\begin{align*}
    F(s') = F(s) P(s' | s) + F(\mathcal{T}_{s'}^{-s}).
\end{align*}
Additionally, because the flow in $s$ does not influence trajectories in $\mathcal{T}_{s'}^{-s}$, then $\frac{d F(\mathcal{T}_{s'}^{-s})}{d F(s)} = 0$; which leads to:
\begin{align*}
\frac{d F(s')}{d F(s)}&= P(s'|s)  \\
\frac{d F(s')}{d \log F(s)} &= F(s) P(s'|s) \\
      &= F(s \cap s') \\
      &= P(s | s') F(s') \\
\frac{d \log F(s')}{d \log F(s)} &=P(s | s').
\end{align*}
\end{proof}

\begin{proposition}
Consider the effect of a slight change in the log of the flow on the edge $s{\rightarrow}s'$: we obtain a change in the log of the flow at $s'$ following the backward conditional probability $P_B(s | s')$:
\begin{equation}
\label{eq:dlogF-dlogFedge}
    \frac{d \log F(s')}{d \log F(s{\rightarrow}s')} =P_B(s | s')
\end{equation}
where $P$ is the distribution on events over trajectories.
\end{proposition}
\begin{proof}
We first use the chain rule and properties of the derivatives of the log, and then we use from the flow matching constraint that $F(s')=\sum_{s \in Par(s')} F(s{\rightarrow}s') \Rightarrow \frac{d F(s')}{d F(s{\rightarrow}s')}=1$:
\begin{align*}
    \frac{d \log F(s')}{d \log F(s{\rightarrow}s')} &= \frac{1}{F(s')} \frac{d F(s')}{d F(s{\rightarrow}s')} \frac{d F(s{\rightarrow}s')}{d \log F(s{\rightarrow}s')} \\
    &= \frac{F(s{\rightarrow}s')}{F(s')} \frac{d F(s')}{d F(s{\rightarrow}s')}  \\
    &= \frac{F(s{\rightarrow}s')}{F(s')} \\
    &= P_B(s | s').
\end{align*}
\end{proof}

Intuitively, \cref{eq:dlogF-dlogF} and \cref{eq:dlogF-dlogFedge} tell us how a perturbation in the flow at one place would result in a change elsewhere in order to maintain the flow match everywhere, using the flow-matching conditions to propagate infinitesimal changes in flow backwards and forwards. Hence, they only become true as we approach the limit of matched flows, and in practice (with an imperfectly trained GFlowNet) the corresponding expressions will be biased. However, we can exploit them to estimate long-range equilibrium gradients and obtain an estimator of credit assignment across a long trajectory in \cref{prop:grad-estimators} below.

In the GFlowNet setting, suppose we parametrize the edge flow estimator $F_\theta(s{\rightarrow}s')$ via parameters $\theta$. In order to understand the effect of a change in $\theta$ on our loss function $\mathcal{L}$, we must compute the \textit{total} derivative $\frac{dL}{d\theta}$ by summing the \textit{direct} and \textit{indirect} gradients. In our context, the direct gradient $\frac{\partial L}{\partial \theta}$ is due to the explicit change in loss from changing $\theta$ (not taking the flow-matching constraint into account) and the indirect gradient includes the induced changes in the flow due to the constraint. With this, we are in a position to formalize unbiased estimators for the total derivative $\frac{dL}{d\theta}$ in \cref{prop:grad-estimators}:
\begin{proposition}
\label{prop:grad-estimators}
Let $\cal L$ be a flow-matching loss computed along a trajectory $\tau=(s_0,s_1,\ldots,s_f)$ sampled according to the GFlowNet trajectory distribution $P(\tau)$. Let ${\cal L}=\sum_i L(s_i)$ decompose the total loss into per-state losses $L(s_i)$ along the trajectory. Let $\theta$ parametrize the edge flow estimator $F_\theta(s{\rightarrow}s')$. Then, in the limit of the flows becoming matched,
\begin{equation}
\label{eq:G1}
    G_1 \coloneqq \sum_{i} \frac{\partial L(s_i)}{\partial \theta} + \frac{\partial L(s_i)}{\partial \log F(s_i)} \sum_{t=1}^{i-1} \frac{\partial \log F_\theta(s_{t-1}{\rightarrow}s_t)}{\partial \theta}
\end{equation}
is an unbiased estimator of the total derivative $\frac{d{\cal L}}{d\theta}$, as is
\begin{align}
\label{eq:G2}
    G_2 \coloneqq \sum_{i} \frac{\partial L(s_i)}{\partial \theta} + \frac{\partial L(s_i)}{\partial \log F(s_i)} \sum_{t=1}^{i-1}  \sum_{s\in Par(s_t)} P_B(s|s_t)      \frac{\partial \log F_\theta(s{\rightarrow}s_t)}{\partial \theta}
\end{align}
as is the convex combination
\begin{align}
\label{eq:G3}
    G = \lambda G_1 + (1-\lambda) G_2
\end{align}
for any $0\leq\lambda\leq 1$. Intuitively, $\frac{\partial L(s_i)}{\partial \theta}$ indicates the gradient directly through the occurrences of $F_\theta$ in $L(s_i)$, and $F(s_i)$ indicates either the forward or backward flow sum present in $L(s_i)$ to obtain the flow through $s_i$.
\end{proposition}

\begin{proof}
We will use the partial derivative notation when not considering the indirect influence due to the matching flow constraint and the total derivative notation when considering it, to apply~\cref{prop:dlogF-dlogF}.
Consider the expected value under the GFlowNet's trajectory distribution of the $G$ in~\cref{eq:G1},
use the previous proposition (\cref{eq:dlogF-dlogF}) and the chain rule:
\begin{align*}
    E[G_1] &= E\left[\sum_{i} \frac{\partial L(s_i)}{\partial \theta} + \sum_{t=1}^{i-1} \frac{\partial L(s_i)}{\partial \log F(s_i)}  \frac{\partial \log F_\theta(s_{t-1}{\rightarrow}s_t)}{\partial \theta} \right] \\
    &= \sum_{s_i} P(s_i) \left(
    \frac{\partial L(s_i)}{\partial \theta} + \sum_{s{\rightarrow}s'<s_i} P(s{\rightarrow}s'|s_i)  \frac{\partial L(s_i)}{\partial \log F(s_i)}  \frac{\partial \log F_\theta(s{\rightarrow}s')}{\partial \theta} \right) \\
    &= \sum_{s_i} P(s_i) \Biggl( \frac{\partial L(s_i)}{\partial \theta} \\ 
    & \mspace{50mu} + \sum_{s{\rightarrow}s'<s_i} P_B(s|s')P(s'|s_i)\left(
    \frac{\partial L(s_i)}{\partial \log F(s_i)} 
    \frac{\partial \log F_\theta(s{\rightarrow}s')}{\partial \theta} \right) \Biggr) \\
    &= \sum_{s_i} P(s_i) \Biggl( \frac{\partial L(s_i)}{\partial \theta}\\
    & \mspace{50mu} + \sum_{s{\rightarrow}s'<s_i} 
    \frac{\partial L(s_i)}{\partial \log F(s_i)} 
    \frac{d \log F(s_i)}{d \log F(s')}
    \frac{d \log F(s')}{d \log F_\theta(s{\rightarrow}s')}
    \frac{\partial \log F_\theta(s{\rightarrow}s')}{\partial \theta} \Biggr)  \\
    &= \sum_{s_i} P(s_i) \Bigg( \underbrace{\vphantom{\sum_{s{\rightarrow}s' < s_{i}}}\frac{\partial L(s_i)}{\partial \theta}}_{\text{direct gradients}} + \underbrace{\sum_{s{\rightarrow}s'<s_i}
    \frac{d L(s_i)}{d \log F_\theta(s{\rightarrow}s')}
    \frac{\partial \log F_\theta(s{\rightarrow}s')}{\partial \theta}}_{\text{indirect gradients}} \Bigg) \\
    &= E\left[ \frac{d \cal L}{d \theta}\right] \\
    &= \frac{d E[{\cal L}]}{d \theta}
\end{align*}

The above demonstration shows that $G$ in~\cref{eq:G3} is asymptotically (as the flows become matched) unbiased when $\lambda=1$ because we recover the $G_1$ of~\cref{eq:G1}. The same proof technique can then be used for $G_2$ which uses transitions $s{\rightarrow}s_t$ sampled from $P_B(s|s_t)$ instead of the trajectory transitions $s_{t-1}{\rightarrow}s_t$ and we obtain that the estimator $G$ in~\cref{eq:G3} is asymptotically unbiased when $\lambda=0$:

\begin{align*}
    E[G_2] &= E\left[\sum_{i} \sum_{t=1}^{i-1} \frac{\partial L(s_i)}{\partial \theta} + \frac{\partial L(s_i)}{\partial \log F(s_i)} \sum_{s\in Par(s')} P_B(s|s') \frac{\partial \log F_\theta(s{\rightarrow}s_t)}{\partial \theta}\right] \\
    &= \sum_{s_i} \sum_{s'<s_i} P(s',s_i) \Biggl(
    \frac{\partial L(s_i)}{\partial \theta} + \\
    & \mspace{50mu} \sum_{s\in Par(s')} P_B(s|s') \frac{\partial L(s_i)}{\partial \log F(s_i)}  \frac{\partial \log F_\theta(s{\rightarrow}s')}{\partial \theta} \Biggr) \\
    &= \sum_{s_i} P(s_i) \Biggl( \frac{\partial L(s_i)}{\partial \theta} + \\
    & \mspace{50mu} \sum_{s'<s_i} P(s'|s_i) \sum_{s\in Par(s')} P_B(s|s') \frac{\partial L(s_i)}{\partial \log F(s_i)} \frac{\partial \log F_\theta(s{\rightarrow}s')}{\partial \theta} \Biggr) \\
    &= \sum_{s_i} P(s_i) \Biggl( \frac{\partial L(s_i)}{\partial \theta} + \\
    &  \mspace{50mu} \sum_{s{\rightarrow}s'<s_i} \left(
    \frac{\partial L(s_i)}{\partial \log F(s_i)} 
    \frac{d \log F(s_i)}{d \log F(s')}
    \frac{d \log F(s')}{d \log F_\theta(s{\rightarrow}s')}
    \frac{\partial \log F_\theta(s{\rightarrow}s')}{\partial \theta} \right) \Biggr) \\
    &= \sum_{s_i} P(s_i) \Biggl( \underbrace{\vphantom{\sum_{s{\rightarrow}s' < s_{i}}}\frac{\partial L(s_i)}{\partial \theta}}_{\text{direct gradients}} + \underbrace{\sum_{s{\rightarrow}s'<s_i} \left(
    \frac{d L(s_i)}{d \log F_\theta(s{\rightarrow}s')}
    \frac{\partial \log F_\theta(s{\rightarrow}s')}{\partial \theta} \right)}_{\text{indirect gradients}} \Biggr) \\
    &= \frac{d E[{\cal L}]}{d \theta}
\end{align*}
where the last identity follows the fourth line in the proof for $G_1$.
Finally a convex combination of two unbiased estimators is unbiased, so we obtain that $G$ in~\cref{eq:G3} is asymptotically unbiased for any $0 \leq \lambda \leq 1$.

\end{proof}

This surprising result says that something very close to policy gradient actually provides an asymptotically (i.e., when flows are matched) unbiased gradient on the parameters of the edge flow, in expectation\footnote{the connection becomes clearer when you imagine minus the loss $L$ to be the reward itself, and we see that we immediately get a training signal at earlier times in the sequence with $G$, similarly to policy gradient. There are also differences because the above proposition relies on staying close to the learning fixed point where the flows are matched.}. Note that it only works exactly in an online setting, i.e., when the trajectory is sampled according to the learner's current policy. Otherwise, the gradient estimator may be biased (it would be biased anyways in practice because the flows are never perfectly matched). However, if instead of sampling trajectories $\tau$ from the GFlowNet transition probabilities $P_F(s_{t+1}|s_t)$ we sample them from a training distribution $\tilde{P}$ with transition probabilities $\tilde{P}(s_{t+1}|s_t)$ we can calculate the importance weights (by the ratio $P(\tau)/\tilde{P}(\tau)$) and correct the estimator accordingly. Since the training distribution $\tilde{P}$ should be broader and have a full support, the importance ratio cannot explode but there could still be the usual numerical problems with the variance of such importance-weighted estimators.

We now consider the setting in which the sampling policy is only slightly different from $\hat{P}$ , which is typically the case because we want the sampling policy to be broader and more exploratory, and because we may be using delayed data, e.g., with a replay buffer. This slight difference may induce a bias but it might still be advantageous to use the above gradient estimator. Note how it does not come in conflict with the gradient of the flow matching loss (which is the first term in $G$). The expected advantage of using $G$ is that it may initially speed up training by directly providing updates to earlier transitions of a complete trajectory. However, analogous to the trade-off between temporal-difference methods and policy-gradient methods, this may come at the price of higher variance.

This estimator is unbiased when the flows are matched and when the trajectory is sampled according to the GFlowNet's distribution, but it also makes a lot of intuitive sense: if the estimated flow at $s_i$ is too small (in the eye of $L_i$) one can clearly push that flow up by increasing the probability of a transition on a path leading to $s_i$. Even if we consider a slightly different trajectory sampling distribution, so long as it leads to $s_i$ we would expect that increasing its probability would increase the probability of ending up in $s_i$ (see~\cref{prop:dlogF-dlogF}).

If the state has a continuous component, we could also increase the probability of ending up in $s_i$ by choosing more often a more probable path to $s_i$. This could be calculated by backpropagating through the state transitions (with some form of backpropagation through time). However, if the transitions are not fully known or are not differentiable, this approach may be more challenging, and is related to similar questions raised with credit assignment in reinforcement learning with continuous states.

Finally, keep in mind that the more direct credit assignment terms in $G$ have to be combined with the local terms $\frac{\partial L_i}{\partial \theta}$ which make sure that the flows becomes better matched, since flow-matching is a necessary condition for $G$ to be unbiased.

\section{Conditional GFlowNets}

\begin{definition}
\label{def:param-cond-gfn}
Consider a set of conditioning information $\mathcal{X}$, a family of DAGs $\mathcal{G}= \{G_{x} = (\gS_x, \sA_x), \ x \in \mathcal{X} \}$, a family of target reward functions $\mathcal{R} = \{R_x: \gS_x^f \rightarrow \R^+, \ x \in \mathcal{X}\}$, and a flow parametrization $(\gO_x, \Pi_x, \gH_x)$ of $(G_x, R_x)$ for every $x \in \mathcal{X}$. 
Define the \textbf{conditional configuration space} as the product
\begin{equation}
    \gO \defeq \prod_{x \in \mathcal{X}} \gO_x.
\end{equation}
An element $o \in \gO$ is a family $(o_x)_{x \in \mathcal{X}}$; we write $o_x$ and $o(x)$ interchangeably for the component at condition $x$.

The tuple $(\mathcal{X}, \mathcal{G}, \mathcal{R}, \{\gO_x\}_{x \in \mathcal{X}}, \{\Pi_x\}_{x \in \mathcal{X}}, \{\gH_x\}_{x \in \mathcal{X}})$ is called a \textbf{conditional GFlowNet}.

\end{definition}

As with the unconditional case, conditional GFlowNets provide a way to sample from different target reward functions $R_x$ {\em simultaneously}. Given a configuration $o \in \gO$ of the conditional GFlowNet and a condition $x \in \mathcal{X}$, the per-condition component $o_x \in \gO_x$ induces a distribution $\pi_x \defeq \Pi_x(o_x)$ over $\gT_x$, the set of complete trajectories in $G_x$, which implicitly defines a terminating state probability measure in $G_x$:
\begin{equation}
    \forall x \in \mathcal{X} \quad \forall s \in \gS_x^f \quad P_T(s \mid x) \defeq \sum_{\tau \in \gT_x: s\rightarrow s_f \mid x \in \tau} \pi_x(\tau),
\end{equation}
where the dependence on $o$ in $P_T$ is omitted for clarity. 

Conditional GFlowNets cast the problem of sampling from target reward functions to a search problem: searching for objects $o \in \gO$ such that $o_x \in \gH_x(\gF_{Markov}(G_x, R_x)) \subseteq \gO_x$. For such objects $o \in \gO$, the terminating state probability measures correspond to the distribution of interest, i.e.:
\begin{equation*}
    \forall x \in \mathcal{X} \quad \forall s \in \gS_x^f \quad P_T(s \mid x) \propto R_x(s).
\end{equation*}

Similar to the unconditional case, we need to design a \textbf{conditional loss function} $\gL$ on $\gO$ that equals zero on such objects $o$ and only on those objects.

\begin{definition}
\label{def:cond-flow-match-loss}
Let $(\mathcal{X}, \mathcal{G}, \mathcal{R}, \gO, \Pi, \gH)$ be a conditional GFlowNet. A \textbf{conditional flow-matching loss} is any function $\gL: \gO \rightarrow \R^+$ such that:
\begin{equation}
    \forall o \in \gO \quad \gL(o) = 0 \Leftrightarrow \forall x \in \mathcal{X} \ \exists F_x \in \gF_{Markov}(G_x, R_x) \ \ o_x = \gH_x(F_x).
\end{equation}
We say that $\gL$ is \textbf{condition-decomposable}, if there are functions $\gL_x:\gO_x \rightarrow \R^+$ such that:
\begin{equation*}
    \forall o \in \gO \quad \gL(o) = \sum_{x \in \mathcal{X}} \gL_x(o_x)
\end{equation*}
\end{definition}

\begin{remark}[Amortized parameterization]
\label{rem:amortized}
In practice, the configuration $o \in \gO$ is often parameterized by shared parameters $\theta$ (e.g., a neural network), so that $o_x = o_x(\theta)$ for all $x \in \mathcal{X}$. In this \emph{amortized} setting, the loss becomes $\gL(\theta) = \sum_{x \in \mathcal{X}} \gL_x(o_x(\theta))$. 
The theoretical guarantee ($\gL = 0$ if and only if all conditions are flow-matched) still holds, but optimization is coupled across conditions. With finite model capacity, the global optimum $\gL = 0$ may be unattainable, and the practical optimum represents a trade-off across conditions.
\end{remark}

We can obtain conditional flow-matching losses that are condition-decomposable starting from any family of flow-matching losses $\gL_x$ for $(G_x, R_x, \gO_x, \Pi_x, \gH_x)$ (\cref{def:flow-matching-loss}). In particular, we can choose these losses to be state-decomposable, edge-decomposable or trajectory-decomposable.

For instance, if each $\gL_x = \sum_{s \in \gS_x} L_x(. \ , s)$ is state-decomposable, then the simultaneous sampling problem is cast to the following minimization problem:
\begin{equation}
    \min_{o \in \gO} \mathbb{E}_{(x, s) \sim \pi_T} [L_x(o_x, s)],
\end{equation}
where $\pi_T$ is any \textbf{conditional full support distribution} on $\mathcal{X} \times \bigcup_{x \in \mathcal{X}} \gS_x$, i.e. a probability distribution that satisfies:
\begin{equation*}
    \forall x \in \mathcal{X} \ \ \forall s \in \bigcup_{x \in \mathcal{X}} \gS_x \quad \pi_T(x, s) > 0 \Leftrightarrow s \in \gS_x.
\end{equation*}
Such a conditional full support distribution can be obtained starting from any distribution $\pi_\mathcal{X}$ with full support on $\mathcal{X}$ and distributions $\pi_x$ with full support on $\gS_x$ for any $x \in \mathcal{X}$, as:
\begin{equation*}
    \pi_T(x, s) = \begin{cases}
    \pi_\mathcal{X}(x) \pi_x(s) \quad \text{if } s \in \gS_x
    \\ 0 \quad \text{otherwise}.
    \end{cases}
\end{equation*}

\begin{example}
Consider the set:
\begin{equation*}
    \gO = \{\hat{F}: \mathcal{X} \times \bigcup_{x \in \mathcal{X}} \sA_x^{-f} \rightarrow \R^+, \ \ \hat{F}(s \rightarrow s' \mid x ) = 0 \ \text{ if } s \rightarrow s' \notin \sA_x \}.
\end{equation*}
For each $x \in \mathcal{X}$ the function $\hat{F}_x \defeq \hat{F}(. \mid x)$ is an element of $\gO_{edge, x}$(\cref{ex:EDGEparam}), i.e. it is a function from $\sA_x^{-f}$ to $\R^+$. Meaning that the set $\gO$ can be seen as a function mapping each $x \in \mathcal{X}$ to $F_x \in \gO_{edge, x}$\footnote{The definition of $\gO$ as a set of functions is equivalent to the product-space definition $\gO = \prod_{x \in \mathcal{X}} \gO_{edge,x}$ via the natural isomorphism $\prod_{x \in \mathcal{X}} \left( \sA_x^{-f} \to \R^+ \right) \;\cong\; \left\{ \hat{F}: \mathcal{X} \times \textstyle\bigcup_{x} \sA_x^{-f} \to \R^+ \;\middle|\; \hat{F}(s \to s' \mid x) = 0 \text{ if } s \to s' \notin \sA_x \right\},$
given by $(\hat{F}_x)_{x \in \mathcal{X}} \mapsto \hat{F}$ where $\hat{F}(\cdot \mid x) = \hat{F}_x(\cdot)$. The function-space representation is convenient when $\hat{F}$ is parameterized by a neural network that takes both $x$ and the edge as input.}. Denoting by $\gH: x\in \mathcal{X} \rightarrow \gH_{edge, x}$ and $\Pi: x\in \mathcal{X} \rightarrow \Pi_{edge, x}$, we obtain a valid conditional flow parametrization $(\gO, \gH, \Pi)$.

Instead of learning each function $\hat{F}_x$ separately, this parametrization enables learning functions $\hat{F} \in \gO$ of both the condition $x$ and the non-terminating edge $s \rightarrow s'$, thus exploiting the generalization capabilities of machine learning algorithms not only on edges, but also on conditions.

Consider the functions $L_x$ defined for each $x \in \mathcal{X}$ as:
\begin{equation*}
    L_x(\hat{F}_x, s') = \begin{cases} \left( \log \left( \frac{\delta + \sum_{s \in Par(s')} \hat{F}(s \rightarrow s' \mid x)}{\delta + R(s' \mid x) + \sum_{s'' \in Child(s') \setminus \{s_f \mid x\}}\hat{F}(s' \rightarrow s'' \mid x)} \right) \right)^2 \quad \text{if } s' \neq s_f, \\
 0 \quad \text{otherwise}
 \end{cases}
\end{equation*}
 where $\delta \geq 0$ is a hyperparameter. The function $\gL$ mapping each $\hat{F} \in \gO$ to 
 \begin{equation*}
     \gL(\hat{F}) = \sum_{x \in \mathcal{X}} \sum_{s \in \gS_x} L_x(\hat{F}_x, s'),
 \end{equation*}
 is a conditional flow-matching loss, that is both condition-decomposable and state-decomposable.
\end{example}

\section{Policies in Deterministic and Stochastic Environments}
\label{sec:stochastic-environment}

Until now, we have focused on a deterministic environment where state changes can perfectly be calculated from a given action. This makes sense when the actions are cognitive actions, internal to an agent, e.g., sequentially constructing a candidate solution to a problem (an explanation, a plan, an inferred guess, etc). What about the scenario where the actions are external and affect the real world? The outcome are likely to be only imperfectly predictable. To address this scenario, we will now extend the GFlowNet framework to learn a policy $\pi$ for an agent in an environment that could be deterministic or stochastic.
\begin{definition}
A {\bf policy} $\pi:{\cal A}\times{\cal S}\mapsto\mathbb{R}$ is a probability distribution $\pi(a|s)$ over actions $a \in {\cal A}$ for each state $s$. To denote the fact that the action space may be restricted based on $s$, we write ${\cal A}(s)$ for the valid actions in state $s$.
\end{definition}
To denote the introduction of actions in the GFlowNet framework, we will decompose transitions in two steps: first an action $a_t$ is sampled according to a policy $\pi$ from state $s_t$, and then the environment transforms this (in a possibly stochastic way) into a new state $s_{t+1}$.
\begin{definition}
We generalize the notion of state as follows: {\bf even states} are of the form $s \in {\cal S}$ while {\bf odd states} are of the form $(s,a) \in {\cal S}\times{\cal A}$. The policy $\pi$ governs the transition from an even state to a compatible next odd state with $a \in {\cal A}(s)$, while the {\bf environment} $P(s_t{\rightarrow}s_{t+1}|s_t,a_t)$ governs the transition from an odd state to the next even state.
\end{definition}
As a result of the above definition, the even-to-even transition is summarized by
\begin{equation}
\label{eq:transition-decomposition}
    P_F(s_{t+1}|s_t) = \sum_{a_t} P(s_t{\rightarrow}s_{t+1}|s_t,a_t) \pi(a_t|s_t).
\end{equation}

Note that the detailed balance condition, which involves a backward transition $P_B$, will also be decomposed in two parts: (1) for inverting the even-to-odd transition,
\begin{equation}
    P_B(s_t|(s_t,a_t))=1
\end{equation}
by definition, and (2) for inverting the odd-to-even transition, we have to actually represent (and learn)
\begin{equation}
    P_B((s_t,a_t)|s_{t+1}).
\end{equation}
This conditional distribution incorporates the preference we may have over different paths leading to the same state while consistent with the environment $P(s_t{\rightarrow}s_{t+1}|s_t,a_t)$. The normalization constraint on $\hat{P}_B$ can guarantee flow-matching via detailed balance, as argued around~\cref{eq:DB-solution}.

\subsection{Known Deterministic Environments}

A deterministic environment is perfectly controllable: we can choose the action that leads to the most desired next state, among the valid actions from the previous state.
In the case where the environment is deterministic, we can directly apply the results of~\cref{sec:GFlowNets-Learning-a-Flow}, as follows. At each time step $t$ and from state $s_t$, the agent picks an allowed action $a_t \in {\cal A}(s_t)$ according to a policy $\pi(a_t|s_t)$. The set of allowed actions should coincide with those actions for which $\pi(a_t|s_t)>0$. Since the environment is deterministic and known, there is a deterministic function $T:{\cal S}\times{\cal A}\mapsto{\cal S}$ which gives us the next state $s_{t+1}=T(s_t,a_t)$. In that case, we can ignore the even/odd state distinction and identify the learnable policy $\pi$ with the learnable transition probability function $P_F(s_{t+1}|s_t)$ of GFlowNets, as follows.
\begin{proposition}
\label{prop:pi-to-P}
In a deterministic environment and a GFlowNet agent with policy $\pi(a_t|s_t)$ and state transitions given by $s_{t+1}=T(s_t,a_t)$, the transition probability $P_F(s_{t+1}|s_t)$ is given by
\begin{equation}
    P_F(s_{t+1}|s_t) = \sum_{a:T(s_t,a)=s_{t+1}} \pi(a|s_t).
\end{equation}
Hence if only one action $a_t$ can transition from $s_t$ to $s_{t+1}=T(s_t,a_t)$, then
\begin{equation}
    P_F(s_{t+1}|s_t) = \pi(a_t|s_t).
\end{equation}
\end{proposition}
\begin{proof}
The result is obtained by marginalizing over $a$:
\begin{align*}
 P_F(s_{t+1}|s_t) &= \sum_a P(s_t{\rightarrow}s_{t+1},a_t|s_t)\\
  &= \sum_a P(s_t{\rightarrow}s_{t+1}|s_t,a_t) \pi(a|s_t) \\
  &= \sum_{a:s_{t+1}=T(s_t,a)} \pi(a|s_t)
\end{align*}
with $P(s_t{\rightarrow}s_{t+1}|s_t,a_t)=1_{s_{t+1}=T(s_t,a)}$.

The case with a single possible action to obtain the transition is obtained because the sum contains only one term.
\end{proof}

\begin{proposition}
In a deterministic environment with $s_{t+1}=T(s_t,a_t)$, a backwards transition probability function can be derived from a backwards policy $\pi_B$,
\begin{equation}
    P_B(s_t|s_{t+1}) = \sum_{a:T(s_t,a)=s_{t+1}} \pi_B(a|s_{t+1})
\end{equation}
and in the case where a single action $a_t$ explains each transition $s_{t+1}=T(s_t,a_t)$,
\begin{equation}
    P_B(s_t|s_{t+1}) = \pi_B(a_t|s_{t+1})
\end{equation}
\end{proposition}
\begin{proof}
The proof goes along exactly the same lines as for~\cref{prop:pi-to-P}.
\end{proof}

\subsection{Unknown Deterministic Environments}

If the environment is deterministic but unknown, we have to learn the transition function $T$ and we should also learn its inverse $T^{-1}$ which recovers the previous state given the next state and the action:
\begin{equation}
    T^{-1}:{\cal S}\times{\cal A}\mapsto{\cal S} \quad {\rm s.t.}\quad
    T^{-1}(T(s,a),a)=s.
\end{equation}
Unfortunately, if the state and action spaces are discrete and in high dimension, learning $T$ and $T^{-1}$ in a way that generalizes to unseen transitions\footnote{Seen transitions can just be recorded in a table, but in a combinatorial state-space, they will form an exponentially tiny fraction of the ones to be encountered in the future.} may be difficult and might be more easily achievable via a continuous relaxation. The methods for stochastic environments could be used for this purpose.

\subsection{Stochastic Environments}

The setting of stochastic environments is less straightforward but more general. We will decompose the transition as per~\cref{eq:transition-decomposition} but not assume that $P(s_t{\rightarrow}s_{t+1}|s_t,a_t)$ is a dirac. The first thing to note is that we can still obtain a Markovian flow, but that we are not guaranteed to find a policy which matches the desired terminal reward function.
\begin{proposition}
In a stochastic environment with environment transitions $P(s_t{\rightarrow}s_{t+1}|s_t,a_t)$, any policy $\pi(a|s)$ can yield a Markovian flow and it may not be possible to perfectly achieve desired flows $\hat{F}(s{\rightarrow}s_f)=R(s)$.
\end{proposition}
\begin{proof}
We obtain a flow by satisfying the flow-matching or detailed balance equations for both even and odd steps, which can always be done for the following reason. From the even states, we can define an edge flow 
$$
\hat{F}(s_t{\rightarrow}(s_t,a_t))=\hat{F}(s_t) \pi(a_t|s_t)
$$ 
and the backwards transition is $\pi_B(s_t|(s_t,a_t))=1$. This leads to the intermediate state flow
$$
 \hat{F}((s_t,a_t)) = \hat{F}(s_t{\rightarrow}(s_t,a_t))
$$
since there is only one edge into $(s_t,a_t)$, the one starting at $s_t$ and taking action $a_t$. From the odd states, we have the edge flow
\begin{align*}
    \hat{F}((s_t,a_t){\rightarrow}s_{t+1}) =& \hat{F}((s_t,a_t)) P(s_t{\rightarrow}s_{t+1}|s_t,a_t) \\
                                           =& \hat{F}(s_t) \pi(a_t|s_t) P(s_t{\rightarrow}s_{t+1}|s_t,a_t) 
\end{align*}
with $P(s_t{\rightarrow}s_{t+1}|s_t,a_t)$ representing the environment,
and we obtain the even state flow with the usual formula (\cref{eq:flow-match})
$$
\hat{F}(s_{t+1}) = \sum_{(s_t,a_t)} \hat{F}((s_t,a_t){\rightarrow}s_{t+1}) = \sum_{(s_t,a_t)} \hat{F}(s_t) \pi(a_t|s_t) P(s_t{\rightarrow}s_{t+1}|s_t,a_t).
$$
If unknown, the environment transitions $P(s_t{\rightarrow}s_{t+1}|s_t,a_t)$ can be estimated in the usual supervised way by observing the triplets $(s_t,a_t,s_{t+1})$ and estimating transition probabilities that approximately maximize the empirical log-likelihood of these observations. However, whereas in a stationary environment the transitions $P(s_t{\rightarrow}s_{t+1}|s_t,a_t)$ do not depend on the policy, the backwards transitions $\hat{P}_B((s_t,a_t)|s_{t+1})$ depend on the forward environment transition and on the state flows, i.e., on the policy. 
With enough training time and capacity, the forward and backward transitions can be made compatible, but as usual with GFlowNets, in a realistic settings the flow matching equations will not be perfectly achieved.

If there is enough capacity and training time (training to completion), we thus obtain a flow. Then, defining the transition probabilities by the sequential sampling of transitions from the even and odd steps above, we obtain a Markovian flow (\cref{prop:markovian-equivalences}).

To show that the desired terminal flows are not necessarily achievable, it is sufficient to identify a counter-example. Consider a terminal reward $R(s)>0$ while the environment transitions into $s$ have zero probability. In that case, no matter how we choose our policy, we cannot put the desired flow into state $s$.
\end{proof}

Keep also in mind that in practice, even in a completely controllable environment, we will not be guaranteed to find a flow that matches the target terminal reward function simply because of finite capacity and finite training time for the GFlowNet.

Whereas with a deterministic environment for the GFlowNet, one can freely choose $P_B$ for non-terminal edges, it is not so for stochastic environments, as argued below. On even-to-odd transitions, $P_B(s|s_t,a_t)=1_{s=s_t}$ by construction. On odd-to-even transitions $(s_t,a_t){\rightarrow}s_{t+1}$ the problem is that the forward transition is not a free parameter (it corresponds to the environment's $P(s_t{\rightarrow}s_{t+1}|s_t,a_t)$).


\begin{counterexample}
With a GFlowNet with a fixed stochastic environment, it may not be possible for $P_B(s_t,a_t|s_{t+1})$ to be chosen freely while also matching the flows and the terminal rewards.

For a counterexample, consider the setting in~\cref{fig:diamondcounterexample}. 
Suppose $R(s''') \neq 0$ and the environment-provided transition $T(s''' | s'', a'') = 0$. Then, in order to match the terminal reward $R(s''')$,  we must require that $P_B(s', a' | s''') \neq 0$, which means that $P_B$ cannot  be chosen freely.

\end{counterexample}

\begin{figure}
    \centering
    \includegraphics[scale=0.15]{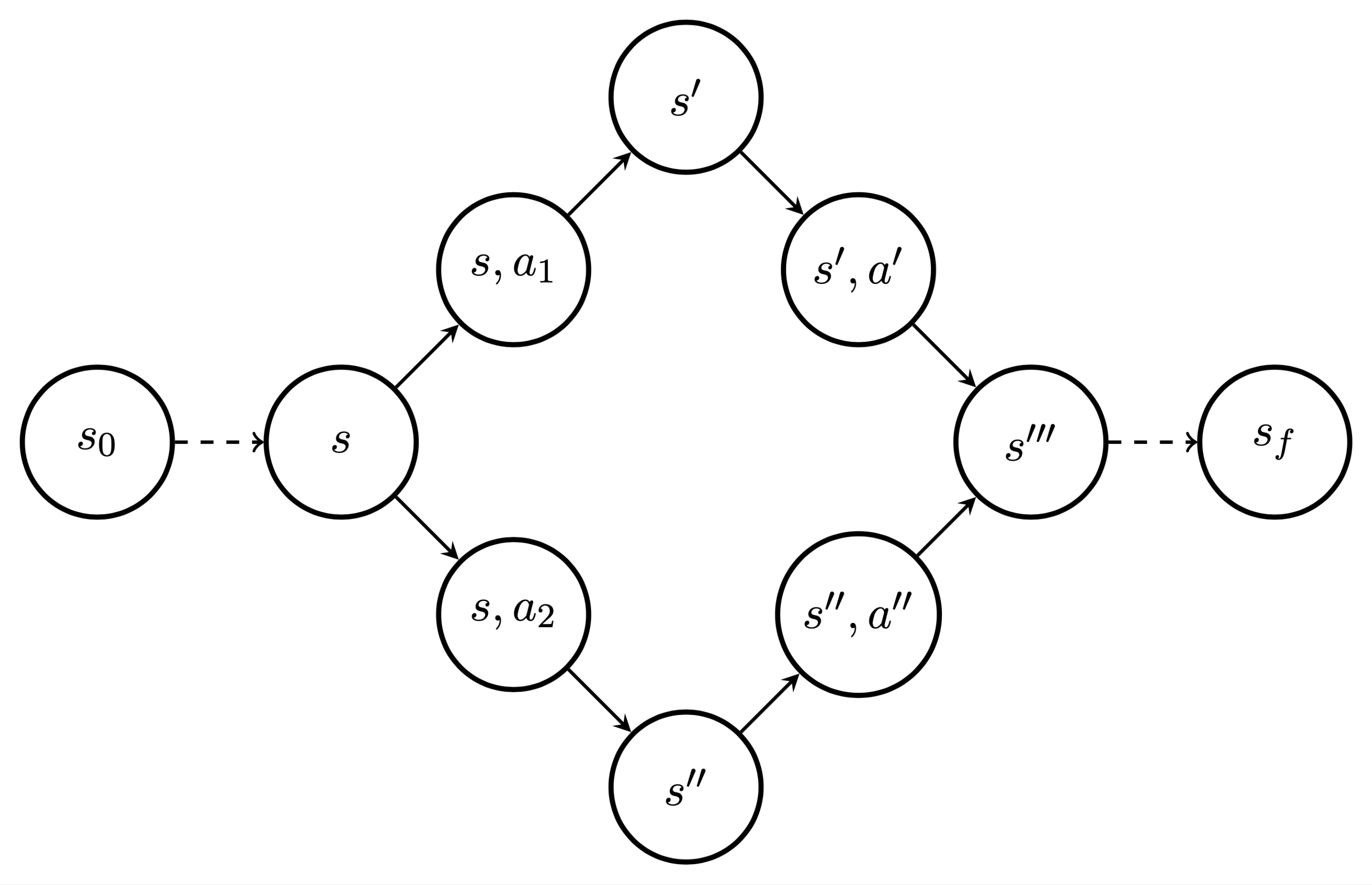}
    \caption{Consider a simple counter-example with only two paths from $s$ to $s'''$, with a given $R(s''')>0$. One path goes through $s'$ and the other through $s''$. From $s'$ the only feasible action $a'$ leads to $s'''$ and similarly from $s''$ with action $a''$. However, it may be that the environment probability $P(s'''| (s'',a''))=0$, constraining $P_B((s'',a'')|s''')=0$. Therefore it may not possible to choose the backward transitions freely while matching the flows and terminal rewards.}
    \label{fig:diamondcounterexample}
\end{figure}




\section{Expected Downstream Reward and Reward-Maximizing Policy}
\label{sec:expected-downstream-reward}

We have already introduced the probability distribution $P_{T}(s)$ and conditional probabilities $P_{T}(s'|s \leq s')$ over terminating states (the states visited just before exiting into $s_{f}$). More generally, one can consider any distribution $P_\pi(s)$ over terminating states arising from some arbitrary choice of GFlowNet policy $\pi$ and compute the expected reward under this distribution:
\begin{definition}
\label{def:V}
The {\bf expected reward} after visiting state $s$ of a flow with terminal reward function $R$, under some distribution over terminating states $P$, is
\begin{equation}
    \label{eq:V}
    V_{P_\pi}(s) \defeq E_{P_\pi(S)}[R(S)|S \geq s] = \sum_{s' \geq s}R(s') P_\pi(s'|s \leq s').
\end{equation}
\end{definition}
\begin{proposition}
\label{prop:V=R2-over-R}
When the probability distribution over terminating states is $P_{T}$ given by the flow (see~\cref{def:flow-P_T}), the expected reward under $P_{T}$ is
\begin{equation}
    \label{eq:V=R2-over-R}
    V_{P_{T}}(s) = \frac{\sum_{s'\geq s}R(s')^{2}}{\sum_{s' \geq s}R(s')}.
\end{equation}
\end{proposition}
\begin{proof}
We apply the definition of conditional $P_{T}$ (\cref{cor:state-conditional-P-T}) to~\cref{eq:V} and obtain the result.
\end{proof}

While we have a simple expression of the expected reward under $P_{T}$, the expected reward is defined more broadly for the distribution $P_\pi$ arising from any policy $\pi$. In particular for any policy $\pi(a|s)$, we can also define the expected reward $V_{P_{\pi}}$ under the distribution $P_\pi$ over terminating states induced by $\pi$. Expected rewards play a role similar to the state and state-action value functions in reinforcement learning, and as a consequence they also satisfy an equivalent of the policy improvement theorem when intermediate rewards are $0$ and the discount factor $\gamma = 1$:
\begin{proposition}
\label{prop:policy-improvement}
Let $P_\pi$ be a distribution over terminating states arising from a policy $\pi$, and $\bar{\pi}$ a greedy policy under the expected reward $V_{P_\pi}$, i.e.,
\begin{align}
    &\bar{\pi}(a|s) = 0\quad{\rm unless} \nonumber \\
    &V_{P_\pi}((s, a)) \geq V_{P_\pi}((s, a')) \quad \forall a'.
\end{align}
Then for all $s$
\begin{equation}
    V_{P_{\bar{\pi}}}(s) \geq V_{P_\pi}(s).
\end{equation}
That is, the expected reward under the probability induced by $\bar{\pi}$ is no worse than the one induced by $\pi$.
\end{proposition}

\begin{proof}
Let us denote by $\bar{\pi}(s)$ the action deterministically chosen by greedy policy $\bar{\pi}$ from $s$, and $s_n$ the stochastically sampled terminating state. Then:

\begin{align}
    V_{P_\pi}(s) &\leq V_{P_\pi}((s, \bar{\pi}(s))) \nonumber \\
                &= E_{\bar{\pi}}[V_{P_{\pi}}(s_{t+1}) | s_t = s]  \nonumber \\
                &\leq E_{\bar{\pi}}[V_{P_\pi}(s_{t+1}, \bar{\pi}(s_{t+1})) | s_t = s] \nonumber  \\
                &= E_{\bar{\pi}}[\mathbb{E_{\bar{\pi}}}[V_{P_{\pi}}(s_{t+2}) | s_t = s | s_t = s] \nonumber  \\
                &= E_{\bar{\pi}}[V_{P_{\pi}}(s_{t+2}) | s_t = s] \nonumber \\
                &\dots \nonumber \\
                &\leq E_{\bar{\pi}}[V_{P_{\pi}}(s_n) | s_t = s] \nonumber \\
                &= E_{\bar{\pi}}[R(s_n) | s_t = s] \nonumber \\
                &= V_{P_{\bar{\pi}}}(s) \nonumber
\end{align}

where we have used the fact that, for all $s'$, $V_{P_\pi}(s') \leq V_{P_\pi}(s', \bar{\pi}(s'))$ since $\bar{\pi}$ is a greedy policy.

\end{proof}
An immediate consequence is the following:
\begin{corollary}
\label{cor:reward-maximizing-policy}
There exists a policy $\pi^{*}(a|s)$ that maximizes the expected reward for all states $s$, namely the greedy policy of~\cref{prop:policy-improvement} associated with the GFlowNet's policy $\pi$ yielding terminal distribution $P_T(s)=R(s)/Z$.
\end{corollary}
How do can we estimate the expected reward under $P_{T}$? We just need to train another flow (or set of heads for a GFlowNet, see \cref{sec:multi-flows}) with $R^2$ as the reward function.
\begin{proposition}
Consider two flows $F$ and $F'$, one matching terminal reward function $R$ and the other matching terminal reward function $R^2$. Then the expected reward under $P_{T}$ (the distribution over terminating states defined by the flow $F$) is
\begin{equation}
    V_{P_{T}}(s) = \frac{F'(s|s)}{F(s|s)}.
\end{equation}
\end{proposition}
\begin{proof}
We start from~\cref{eq:V=R2-over-R} of the above corollary and notice that the numerator is the self-flow for $F'$ while the denominator is the self-flow for $F$ (see~\cref{eq:free-energy}).
\end{proof}

\subsection{Preference for High-Reward Early Trajectory}

We have seen in~\cref{sec:free-P_B} that by imposing a particular preference on $P_B$, one can make the GFlowNet sampling mechanism prefer to construct states in some orders more than others, e.g., one could prefer to start with states with larger expected reward (over their potential continuations), using~\cref{def:V}. It suffices to define $P_B(s_t,a_t|s_{t+1})$ so that it puts more probability mass on state-action pairs $(s_t,a_t)$ with larger $V((s_t,a_t))$.

\section{Intermediate Rewards and Trajectory Returns}
\label{sec:intermediate-rewards-and-returns}

Up to now and in the GFlowNet paper~\citep{bengio2021flow}, we have considered terminal rewards as events happening only once per trajectory, at its end.  Consider instead an agent experiencing a complete trajectory $\tau$ and declare its {\bf return} to be the sum of some intermediate environment rewards associated with all the transitions into the sink node from each of the visited states.
\begin{definition}
\label{def:return}
The {\bf trajectory return} $\rho(\tau)$ associated with a partial trajectory \mbox{$\tau=(s_i,s_{i+1}\ldots,s_n,s_f)$} is defined as
\begin{align}
    \rho(s_i,s_{i+1},\ldots,s_n,s_f) &= \sum_{t=i}^{n} R(s_t) = \sum_{t=i}^n F(s_t{\rightarrow}s_f) \\
    \rho(s_f) &= 0
\end{align}
and the {\bf expected future return $\bar{\rho}(s_t)$} associated with a state $s_t$ is defined as
\begin{align}
    \bar{\rho}(s_t) &= E[\rho(s_t,s_{t+1},s_{t+2},\ldots,s_n,s_f)|s_t] \nonumber \\
    &= \sum_{s_{t+1},\ldots,s_n} P(s_{t+1},\ldots,s_n,s_f|s_t) \rho(s_t,s_{t+1},\ldots,s_f)
\end{align}
where the expectation is defined under the flow's probability measure over trajectories (conditioned on the trajectory going through $s_t$).
\end{definition}

\begin{proposition}
The expected future return $\bar{\rho}(s)$ achievable from trajectories starting at $s$ satisfies the following recursion:
\begin{equation}
  \bar{\rho}(s) = R(s) + \sum_{s'\in Child(s)} P_F(s'|s) \bar{\rho}(s').
\end{equation}
\end{proposition}
\begin{proof}
From the definition of return (\cref{def:return}), we obtain the following:
\begin{align*}
 \bar{\rho}(s_t) &= \sum_{s_{t+1},s_{t+2},\ldots}
 P(s_{t+1},\ldots,s_f|s_t)(R(s_t) + \rho(s_{t+1},s_{t+2},\ldots,s_n)) \\
 &= R(s_t) + \sum_{s_{t+1}} P_F(s_{t+1}|s_t) \sum_{s_{t+2}, \ldots, s_n} P(s_{t+2},\ldots,s_n|s_{t+1}) \rho(s_{t+1},\ldots,s_n) \\
  &= R(s_t) + \sum_{s_{t+1}} P_F(s_{t+1}|s_t) E[\rho(s_{t+1},s_{t+2},\ldots)|s_{t+1}] \\
  &= R(s_t) + \sum_{s_{t+1}} P_F(s_{t+1}|s_t) \bar{\rho}(s_{t+1}).
\end{align*}
\end{proof}
One reason why the above recursion is interesting is that it corresponds to the Bellman equation~\citep{sutton2018reinforcement} for the value function (which is the expected downstream return) in the case of no discounting (with an episodic setting). It is interesting to compare it with one of the equations we obtain for the state flow (\cref{eq:flow-match}):
\begin{align*}
F(s_t) &= \sum_{s_{t+1}} F(s_{t+1}) P_B(s_t|s_{t+1})  \\
      &= R(s_t) + \sum_{s_{t+1} \neq s_f} P_B(s_t|s_{t+1}) F(s_{t+1})
\end{align*}
The two recursions are different: one uses the forward transition to propagate values backward, while the other uses the backward transitions to propagate flows.

\begin{definition}
\label{def:return-augmented}
Let us denote $r(s)$ the possibly stochastic {\bf environment reward}, provided when an agent visits state $s$ (and generally distinct from the GFlowNet terminal reward at $s$), and consider the environment reward accumulated in the partial trajectory $\tau=(s_0,s_1,\ldots,s_n)$ leading to state $s_n=s$. Let us call the GFlowNet state {\bf return-augmented} if the state $s$ includes the {\bf accumulated reward} $\nu(s)$, i.e., there exists a function $\nu(s)=\sum_{t=0}^n r(s_t)$. Let us call a GFlowNet with a return-augmented state a {\bf return-augmented GFlowNet}.
\end{definition}
We want to keep track of accumulated intermediate rewards in the GFlowNet state to compute the terminal reward from the state, and thus train GFlowNets to sample in proportion to the accumulated reward:
\begin{proposition}
\label{prop:proportional-to-return}
Suppose $G$ is a return-augmented GFlowNet with target terminal reward function $R(s)$ equal to the accumulated environment reward $\nu(s)$. Furthermore, suppose $G$ is trained to completion. Then sampling from $G$ produces accumulated reward $\nu$ with probability proportional to $\nu$.
\end{proposition}
\begin{proof}
Since $G$ is a GFlowNet trained to completion with terminal reward function $R(s) = \nu(s)$, we know that sampling from $G$ samples terminal accumulated rewards $\nu(s)$ with probability proportional to $\nu$.
\end{proof}

Note that such a GFlowNet can only be trained offline, i.e., using trajectories which have terminated and for which we have observed the return.

Note also that if we did not augment the state with the accumulated reward, then the GFlowNet terminal reward would not be a function of the state. Having a return-augmented state also makes it possible to handle stochastic environment rewards in the GFlowNet framework.

\section{Multi-Flows, Distributional GFlowNets, Unsupervised GFlowNets and Pareto GFlowNets}
\label{sec:multi-flows}

Consider an environment with stochastic rewards. As with~\cref{def:return-augmented}, we could augment the state to include the random accumulated rewards, thus (by~\cref{prop:proportional-to-return}) making the GFlowNet sample trajectories with returns $\rho$ occurring with probability proportional to $\rho$. However, similarly to Distributional RL~\citep{bellemare2017distributional}, it could be interesting to generalize GFlowNets to capture not just the expected value of achievable terminal rewards but also other statistics of its distribution. More generally, we can think of this like a family of GFlowNets, each of which models in its flow a particular future environmental outcome of interest. With the particle analogy of GFlowNets, it would be as if the particles had a colour or label (just like the frequency of each photon in an group of photons travelling together in a beam of light) and that we separately account for the flows associated with all the possible label types.

If the number of outcomes (the number of possible labels) is small, this could be implemented with different output heads of the GFlowNet (e.g., one output for the flow associated with each label). When a trajectory associated with a particular label outcome is observed, the corresponding heads get gradients. A more powerful and general implementation puts the outcome event as an {\em input} of the GFlowNet, thus amounting to training a {\em conditional GFlowNet} (see~\cref{sec:conditional-gflownets}), and formalized below.

\begin{definition}
Let us define the {\bf outcome} $y=f(s)$ as a known function $f(s)$ of the state $s$. An outcome can be whether an environment reward takes a particular value, or it can be a vector of important features of $s$ which are sufficient to determine many possible environment reward functions (in particular, $f$ can be the identity function). Let us call the conditional GFlowNet taking $y$ as conditioning input, with conditional flows $F(A|y)$ for events $A$ over the trajectories consistent with reward function 
\begin{equation}
\label{eq:outcome-conditioned-reward}
    R_y(s)=1_{f(s)=y}
\end{equation}
an {\bf outcome-conditioned GFlowNet} with outcome function $f$.
\end{definition}

We will limit ourselves here to a discrete set of outcomes for simplicity but expect that this approach can be generalized to a continuous set. Note how, if the outcome-conditioned GFlowNet is trained to completion, it makes it possible to only sample terminating states $s$ yielding the chosen outcome $y$. In principle, this allows sampling objects guaranteed to have a high reward (under the given reward function). In practice, a GFlowNet will never be perfectly trained to completion, and we should think of such an outcome-conditioned GFlowNet similarly to a goal-conditioned policy in RL~\citep{ghosh2018learning} or the reward-conditioned 
upside-down RL~\citep{schmidhuber2019reinforcement}. An interesting question for future work is to extend these outome-conditioned GFlowNets to the case of stochastic rewards or stochastic environments.

\begin{definition}
A {\bf distributional GFlowNet} is an outcome-conditioned GFlowNet taking as conditioning input the value of the environment return associated with complete trajectories. This can be achieved by making the GFlowNet return-augmented, so that the return can be read from the terminating state of the trajectory.
\end{definition}

Training an outcome-conditioned GFlowNet can only be done offline because the conditioning input (e.g., the final return) may only be known after the trajectory has been sampled. A reasonable contrastive training procedure could thus proceed as follows:
\begin{enumerate}
    \item Sample a trajectory $\tau^+$ according to an unconditional training policy $\pi_T$. 
    \item Obtain the outcome $y^+=f(s^+)$ from the terminating state $s^+$ (occurring just before the sink state $s_f$ in $\tau^+$).
    \item Update the conditional GFlowNet with $\tau^+$ and target terminating reward $R(s^+|y^+)=1_{f(s^+)=y^+}=1$.
    \item Sample a trajectory $\tau^-$ according to the conditional GFlowNet policy with condition $y^+$.
    \item Obtain the actual outcome $y^- = f(s^-)$ for the terminating state $s^-$ (occurring just before the sink state $s_f$) in $\tau^-$. If the GFlowNet was perfectly trained, we should have $y^+=y^-$ but otherwise, especially if the number of possible outcomes is large, this becomes unlikely.
    \item Update the conditional GFlowNet using a flow-matching loss with trajectory $\tau^-$ and target terminating reward $R(s^-|y)=1_{y^-=y^+}$ (likely to be 0 if there are many possible values for $y$).
\end{enumerate}

An interesting question for future work is to consider a smoother reward function instead of the sharp but sparse reward $R(s|y)=1_{f(s)=y}$ as conditional reward, in order to make training easier.

\subsection{Defining a reward function a posteriori}

The reward function may not be known a priori or it may be known only up to some unknown constants (e.g., defining a Pareto front) or we may wish to generalize GFlowNets so they can be trained or pre-trained in a more unsupervised way, with the specific generative task only specified afterwards. The following important proposition allows us to do just that, and convert an outcome-conditioned GFlowNet into one that samples according to a given reward function, without having to retrain the network.

\begin{proposition}
Consider an outcome-conditioned GFlowNet trained to completion with respect to the possible outcomes $y=f(s)$ over terminating states $s$ and a terminal reward function $R(s)=r(f(s))$ given a posteriori (possibly after training the GFlowNet) as a function $r(y)$ of the outcome $y=f(s)$. Then a GFlowNet with flow $F_{r\circ f}(A)$ over events $A$ which matches target terminal reward function $R={r\circ f}$ can be obtained from the flow $F(A|y)$ of the outcome-conditioned GFlowNet via
\begin{equation}
\label{eq:outcome-sum}
    F_{r\circ f}(A) = \sum_y r(y) F(A|y).
\end{equation}
For example, the GFlowNet policy $\pi_{r\circ f}(a|s)$ for terminal reward $R=r\circ f$ can be obtained from
\begin{equation}
    \pi_{r\circ f}(a|s) = \frac{\sum_y r(y)F((s,a)|y)}{\sum_y r(y) F(s|y)}
\end{equation}
or
\begin{equation}
    \pi_{r\circ f}(a|s) = \frac{\sum_y r(y)F(s|y)\pi(a|s,y)}{\sum_y r(y) F(s|y)}
\end{equation}
where $F(s|y)$, $F((s,a)|y)$ and $\pi(a|s,y)$ are the outcome-conditioned state flow, state-action flow and action policy respectively.
\end{proposition}
\begin{proof}
We first clarify what $F_{r\circ f}(A)$ means:
\begin{equation}
\label{eq:Frf}
 F_{r\circ f}(A) = \sum_s r(f(s)) P(A|s{\rightarrow}s_f)
   = \sum_s r(f(s)) P_B(A|s{\rightarrow}s_f)
\end{equation}
where $P(A|s{\rightarrow}s_f)$ is the probability of event $A$ (e.g., a particular state or transition) among the trajectories that end in terminal transition $s{\rightarrow}s_f$, and it can be determined entirely by $P_B$ (considering all the backward paths starting at $s$ and going back to the particular state or transition $A$). This means that $P(A|s{\rightarrow}s_f)$ {\em does not depend of the choice of reward function}, so~\cref{eq:Frf} is valid for any $r$.

Clearly, we see that~\cref{eq:outcome-sum} works in the extreme case where $r(y')=1_{y=y'}$ is an indicator function at some specific $y$ value since 
the sum in~\cref{eq:outcome-sum} reduces to $F(A|y)$
which corresponds to reward function $R_y$ as per~\cref{eq:outcome-conditioned-reward}. This yields
\begin{equation}
\label{eq:indicator-F}
    F_{1_y}(A) = F(A|y) = \sum_s 1_{f(s)=y} P(A|s{\rightarrow}s_f)
\end{equation}
where $1_y$ denotes the function that, given $s$, returns $1_{y=s}$.

To complete the proof let us start with the right-hand side of~\cref{eq:outcome-sum} and insert the above definition of  $F(A|y)$ (\cref{eq:indicator-F}), then swap the sums and use the indicator function to cancel the sum over $y$, and finally apply the definition of $F_{r\circ f}(A)$ in~\cref{eq:Frf}:
\begin{align}
    \sum_y r(y) F(A|y) &= \sum_y  r(y) \sum_{s} 1_{f(s)=y} P(A|s{\rightarrow}s_f) \nonumber \\
    &= \sum_s  \sum_y r(y) 1_{f(s)=y} P(A|s{\rightarrow}s_f) \nonumber \\
    &= \sum_s r(f(s)) P(A|s{\rightarrow}s_f) \nonumber \\
    &= F_{r\circ f}(A)
\end{align}
recovering the left-hand side of~\cref{eq:outcome-sum} as desired.
\end{proof}

This makes it possible to predict probabilities and perform sampling actions for all the possible outcomes $y$ arising in different states (in the extreme where we know nothing about possible reward functions and the outcome is $y=s$), and then convert that GFlowNet on the fly to one specialized to a given terminal reward function $R=r\circ f$. However, we note that it requires more computation for each action at run-time: we have to perform these sums (possibly via Monte-Carlo integration) over the outcome space, and there may be a computational time versus accuracy trade-off in the resulting decisions (based on how many Monte-Carlo samples are used to approximate the above sums).

\subsection{Pareto GFlowNets}

A related application of these ideas concerns Pareto optimization, where we are not sure about the correct reward function up to a few coefficients forming a convex combination of underlying objectives.
\begin{definition}
The {\bf Pareto additive terminal reward functions} can be written as
\begin{equation}
    R_\omega(s) = \sum_i \omega_i f_i(s)
\end{equation}
where $\omega \in \{\omega \in W \subset \mathbb{R}^d: \omega_i\geq 0, \sum_i \omega_i=1\}$ are convex weights and the outcomes of interest are $d$ objectives $y_i=f_i(s)$, and $W$ is a discrete
set of convex weights.
\end{definition}
\begin{definition}
The {\bf Pareto multiplicative terminal reward functions} can be written as
\begin{equation}
    R_\omega(s) = e^{ - \sum_i \omega_i e_i(s)}
\end{equation}
where $\omega \in \{\omega \in W \subset \mathbb{R}^d: \omega_i\geq 0, \sum_i \omega_i=1\}$ are convex weights and the outcomes of interest are $d$ objectives $y_i=f_i(s)$, and $W$ is a discrete
set of convex weights.
\end{definition}
In these cases, we can train a conditional GFlowNet with $\omega$ as conditioning input and $R_\omega(s)=R(s|\omega)$ as conditional terminal reward function. At run-time, we can scan the set of $\omega$'s in order to obtain different policies or predicted probabilities or free energies. The above can easily be generalized to a non-convex and non-linear combination of the objectives, so long as the combined objective is parametrized by $\omega$.  Note the similarity between this idea (and more generally outcome-conditioned GFlowNets) and the earlier work by~\citet{dosovitskiy2019you}. The same idea of conditioning by a form of specification of the loss can be applied to GFlowNets to obtain a family of GFlowNets, one for each variant of the loss function.

A useful application of a Pareto GFlowNet with such reward functions is to draw samples from the Pareto frontier. Once the Pareto GFlowNet is trained, we can draw samples from the Pareto frontier by first sampling the convex weights $\omega$ and then sampling trajectories. This can be useful in multi-objective optimization or sampling, where we want to draw a diversity of solutions corresponding to different trade-off points of the various objectives.

We could also train an outcome-driven GFlowNet by providing the vector of objective values $y$ as input and exploit prior knowledge about the objectives.
For example, we may believe that different objectives can be modeled independently of each other and that $F(s|y)$ (or similarly for $F((s,a)|y)$) can be written as a basis expansion $F(s|y)=\sum_{i=1}^d \sum_{j=1}^N \phi_j(y_i) F_{i,j}(s)$ with $N$ bases $\phi_j$ used to represent the objective $y_i$. 
This could be advantageous from a generalization point of view if learning about the different objectives should be disentangled from one another (e.g., one is stationary and the other is not).

\fi
\end{document}